\documentclass[letterpaper]{article} 
\usepackage{aaai24}  
\usepackage{times}  
\usepackage{helvet}  
\usepackage{courier}  
\usepackage[hyphens]{url}  
\usepackage{graphicx} 
\usepackage{natbib}  
\usepackage{caption} 
\usepackage{algorithm}
\usepackage{algorithmic}
\usepackage[utf8]{inputenc} 
\usepackage[T1]{fontenc}    
\usepackage{url}            
\usepackage{booktabs}       
\usepackage{amsfonts}       
\usepackage{nicefrac}       
\usepackage{microtype}      
\usepackage{xcolor}         
 \usepackage{amssymb}
 
\usepackage[most]{tcolorbox}
\usepackage{newfloat}
\usepackage{listings}
\DeclareCaptionStyle{ruled}{labelfont=normalfont,labelsep=colon,strut=off} 
\floatstyle{ruled}
\newfloat{listing}{tb}{lst}{}
\floatname{listing}{Listing}
\title{\texttt{DePRL}: Achieving Linear Convergence Speedup in Personalized Decentralized Learning with Shared Representations}
\author{
    Guojun Xiong\textsuperscript{\rm 1}, Gang Yan\textsuperscript{\rm 2}, Shiqiang Wang\textsuperscript{\rm 3}, Jian Li\textsuperscript{\rm 1}
}
\affiliations{
    \textsuperscript{\rm 1}Stony Brook University, \textsuperscript{\rm 2}Binghamton University,
    \textsuperscript{\rm 3}IBM T. J. Watson Research Center\\

    \textsuperscript{\rm 1}\{ guojun.xiong, jian.li.3\}@stonybrook.edu, \textsuperscript{\rm 2}gyan2@binghamton.edu, 
    \textsuperscript{\rm 3}wangshiq@us.ibm.com
}

\usepackage{bibentry}

\usepackage{microtype}
\usepackage{graphicx}
\usepackage{booktabs} 
\usepackage{subcaption}
\usepackage{caption}

\usepackage{array}
 \usepackage{amssymb}

\usepackage{amsmath}
\usepackage{amssymb}
\usepackage{mathtools}
\usepackage{amsthm}
\usepackage{algorithm}
\usepackage{algorithmic}
\usepackage{amsfonts,comment}
\usepackage[capitalize,noabbrev]{cleveref}
\usepackage{multirow}
\usepackage{svg}
\newtheorem{theorem}{Theorem}

\newtheorem{lemma}{Lemma}
\newtheorem{corollary}{Corollary}

\newtheorem{assumption}{Assumption}
\newtheorem{remark}{Remark}

\usepackage[textsize=tiny]{todonotes}

  \def\cC{{\mathcal{C}}} \def\cD{{\mathcal{D}}}
\def\cE{{\mathcal{E}}}  \def\cG{{\mathcal{G}}} 
   
 \def\cN{{\mathcal{N}}}

\def\ba{{\mathbf{a}}} \def\bb{{\mathbf{b}}}

  \def\bw{{\mathbf{w}}} \def\bx{{\mathbf{x}}}

\def\bP{{\mathbf{P}}}    
   \def\bX{{\mathbf{X}}}

\newcommand{\DePRL}{\texttt{DePRL}\space}

\begin{document}

\maketitle

\begin{abstract}
Decentralized learning has emerged as an alternative method to the popular parameter-server framework which suffers from high communication burden, single-point failure and scalability issues due to the need of a central server.  However, most existing works focus on a single shared model for all workers regardless of the data heterogeneity problem, rendering the resulting model performing poorly on individual workers.  In this work, we propose a novel personalized decentralized learning algorithm named \texttt{DePRL} via shared representations.  Our algorithm  relies on ideas from representation learning theory to learn a low-dimensional global representation collaboratively among all workers in a fully decentralized manner, and a user-specific low-dimensional local head leading to a personalized solution for each worker.  We show that \texttt{DePRL} achieves, for the first time, a provable \textit{linear speedup for convergence} with general non-linear representations (i.e., the convergence rate is improved linearly with respect to the number of workers). Experimental results support our theoretical findings showing the superiority of our method in data heterogeneous environments.

\end{abstract}

\section{Introduction}\label{intro}
Fueled by the rise of machine learning applications in 
Internet of Things, federated learning (FL) \cite{mcmahan2017communication, imteaj2022federated} has become an emerging paradigm that allows a large number of workers to produce a global model without sharing local data.  The task of coordinating between workers is fulfilled by a central server that aggregates models received from workers at each round and broadcasts updated models to them.  However, this parameter-server (PS) based scheme has a major drawback for the need of a central server \cite{kairouz2019advances}.  In practice, the communication occurs between the server and workers leads to a quite large communication burden for the server \cite{lian2017can}, and the server could face system failure or attacks, which may leak users' privacy or jeopardize the training process.

With this regard, \textit{consensus-based decentralized learning} has recently emerged as a promising method, where each worker maintains a local copy of the model and embraces peer-to-peer communication for faster convergence \cite{lian2017can,lian2018asynchronous}.  In decentralized learning, workers follow a communication graph to reach a so-called consensus model.  However, like conventional PS  framework, one of the most important challenges in decentralized learning is the issue of \textit{data heterogeneity}, where the data distribution among workers may vary to a large extent.  As a result, if all workers learn a \textit{single shared model} with parameter $\bw$, the resulting model could perform poorly on many of individual workers.  To this end, \textit{personalized decentralized learning} \cite{vanhaesebrouck2017decentralized,dai2022dispfl}
is important for achieving personalized models for each worker $i$ with parameter $\bw_i$ instead of using a single shared model.

In this paper, we take a further step towards personalized decentralized learning.  In particular, we take advantage of common representation among workers. This is inspired by observations in centralized learning, which suggest that heterogeneous data distributed across tasks (e.g., image classification) may share a common (low-dimensional) representation despite having different labels \cite{bengio2013representation,lecun2015deep}. To our best knowledge, \citet{collins2021exploiting} is the first to leverage this insight to design personalized PS based scheme, while we generalize it to decentralized setting.  Specifically, we consider the setting in which 
all workers' model parameters share a common map, coupled with a personalized map that fits their local data.  Formally, the parameter for worker $i$'s model can be represented as $\bw_i= \pmb{\theta}_i \circ \pmb{\phi}$, where $\pmb{\phi}: \mathbb{R}^d \rightarrow \mathbb{R}^z$ is a shared \textbf{{global representation}}\footnote{For abuse of notion, we use $\pmb{\phi}$ to denote both the global representation model and its associated parameter, and $\{\pmb{\theta}_i\}_{i=1}^N$ to denote both the local heads and its associated parameter. For simplicity, we call $\pmb{\phi}$ the \textit{global representation} and $\pmb{\theta}_i$ the \textit{local head} of worker $i$ in the rest of the paper. The ``$\circ$'' symbol denotes the composition relation between the parameters $\pmb{\theta}_i$ and $\pmb{\phi}$ as in \citet{collins2021exploiting}. } which maps $d$-dimensional data points to a lower space of size $z$, and $\pmb{\theta}_i : \mathbb{R}^z \rightarrow \mathcal{Y}$ is the worker specific \textbf{{local head}} which maps from the lower dimensional subspace to the space of labels. Typically $z \ll d$ and thus given any fixed representation $\pmb{\phi}$, the worker specific heads $\pmb{\theta}_i $ are easy to optimize locally.  Though \citet{collins2021exploiting} provided a rigorous analysis with linear global representation, 
the following important questions remain open: 

\textit{Does there exist a personalized, fully decentralized algorithm that can solve the optimization problem  $\min_{\pmb{\phi} \in \Phi}  \frac{1}{N} \sum_{i=1}^N  \min_{\pmb{\theta}_i \in \Theta}F_i(\pmb{\theta}_i \circ \pmb{\phi}),$ where $F_i(\cdot)$ is the loss function associated with worker $i$?  Can we provide a convergence analysis for such a personalized, decentralized algorithm under general non-linear representations?}

In this paper, we provide affirmative answers to these questions.  We propose \textit{a fully decentralized algorithm} named \DePRL with alternating updates between global representation and local head parameters to solve the above optimization.  At each round, each worker performs one or more steps of stochastic gradient descent to update its local head and global representation from its side.  Then each worker \textit{only} shares its updated global representation with a subset of workers (neighbors) in the communication graph and computes a weighted average (i.e., consensus component) of global representations received from its neighbors.  The updated local head and global representation after consensus serve as the initialization for the next round update.  All workers in \DePRL collaborate to learn a common global representation, while locally each worker learns its unique head.

Compared to conventional decentralized learning with a single shared model \cite{lian2017can, lian2018asynchronous,assran2019stochastic}, the updates of parameters of local head and global representation in \DePRL are strongly coupled due to their intrinsic dependence and iterative update nature. This makes existing convergence analysis for  decentralized learning with a single shared model not directly applicable to ours, and necessitates different proof techniques. One fundamental reason is that, instead of learning only a single shared model, there are multiple local heads that need to be handled in \DePRL and the updates of local heads are also strongly coupled with global representation.  We summarize our contributions:

$\bullet$ \textbf{\DePRL Algorithm.} 
We propose for the first time \textit{a fully decentralized} algorithm named \DePRL which leverages ideas from representation learning theory to learn a global representation collaboratively among all workers, and a user-specific local head leading to a personalized solution for each worker.

$\bullet$ \textbf{Convergence Rate.} To incorporate the impact of two coupled parameters, we first introduce a new notion of $\epsilon$-approximation solution. Using this notion, to our best knowledge, we provide the first convergence analysis of personalized decentralized learning with \textit{shared non-linear representations}.  We show that the convergence rate of \DePRL is $\mathcal{O}(\frac{1}{\sqrt{NK}})$, where $N$ is the number of workers, and $K$ is the number of communication rounds which is sufficiently large.  This indicates that \DePRL achieves a \textbf{linear speedup} for convergence with respect to the number of workers. This is the first linear speedup result for personalized decentralized learning with shared representations, and is highly desirable since it implies that one can efficiently leverage the massive parallelism in large-scale decentralized systems.  In addition, interestingly, our results guarantee that all workers reach a consensus on the shared global representation, while learn a personalized local head.  This reveals new insights on the relationship between personalized decentralized model with shared representations and its generalization to unseen workers that have not participated in the training process, as we numerically verify in experimental results.

$\bullet$ \textbf{Evaluation.} To examine the performance of \DePRL and verify our theoretical results, we conduct experiments on different datasets with representative DNN models and compare with a set of baselines.  Our results show the superior performance of \DePRL in data heterogeneous environments.

\section{System Model and Problem Formulation}\label{sec:model}

\textbf{Notation.}
Denote the number of workers and communication rounds as $N$ and $K$, respectively.  We use calligraphy letter $\mathcal{A}$ to denote a finite set with cardinality $|\mathcal{A}|$, and $[N]$ to denote the set of integers $\{1,\cdots, N\}$.  We use boldface to denote matrices and vectors, and $\|\cdot\|$ to denote the $l_2$-norm. 

\textbf{Consensus-based Decentralized Learning.} Supervised learning aims to learn a model with optimal parameter  that maps an input to an output by using examples from a training data set $\mathcal{D}$ with each example being a pair of input $\bx_m$  and the associated output $y_m$. Due to increases in available data and the complexity of statistical model, an efficient decentralized
algorithm is to offload the computation overhead to $N$  workers, which jointly determine the optimal parameters through a decentralized coordination. This gives rise to the minimization of the sum of functions local to each worker 
 \vspace{-0.1cm}\begin{align}\label{eq:decen_objective}
\min_{\bw} f(\bw):=\frac{1}{N}\sum\limits_{i=1}^N F_i(\bw),
\vspace{-1em}
\end{align} 
where $F_i(\bw)=\frac{1}{|\cD_i|}\sum_{(\bx_m,y_m)\in\cD_i}\ell(\bw, \bx_m,y_m)$, $\cD_i$ is worker $i$'s local dataset,
with $\ell(\bw, \bx_m, y_m)$ being model error on example $(\bx_m, y_m)$ using model  parameter $\bw\in\mathbb{R}^{d\times 1}$.
The decentralized system can be modeled as \textit{a communication graph} $\cG=(\cN,\cE)$ with $\cN=[N]$ being the set of workers and an edge $(i, j)\in\cE$ indicates that workers $i$ and $j$ can communicate with each other.  
We assume the graph is strongly connected \cite{nedic2009distributed,nedic2018network}, i.e., there exists at least one path between any two arbitrary workers. Denote neighbors of worker $i$ as $\cN_i=\{j|(j,i)\in\cE\}\cup\{i\}$.  All workers perform local updates synchronously and broadcast updated models to their neighbors. Each worker then computes a weight average (i.e., consensus component) of the received models from its neighbors, which serves as the initialization for next round.

\textbf{Personalization via Common Representation.} 
Conventional decentralized learning aims at learning a {\em single} shared model parameter $\bw$ that performs well on average across all workers \cite{lian2017can,lian2018asynchronous}.  However, this approach may yield a solution that performs poorly in heterogeneous settings where data distributions
vary across workers.  Indeed, in the presence of data heterogeneity, the error functions $F_i$ will have different minimizers.   This necessities the search for more personalized solutions $\{\bw_i\}_{i=1}^N$ that can be learned in a decentralized manner using workers' local data.

To address this challenge, \citet{collins2021exploiting}  leveraged representation learning theory into the PS setting. Formally, the parameter for worker $i$'s model can be represented as $\bw_i= \pmb{\theta}_i \circ \pmb{\phi}$, where $\pmb{\phi}: \mathbb{R}^d \rightarrow \mathbb{R}^z$ is a shared \textbf{\textit{global representation}} which maps $d$-dimensional data points to a lower space of size $z$, and $\pmb{\theta}_i : \mathbb{R}^z \rightarrow \mathcal{Y}$ is the worker specific \textbf{\textit{local head}} which maps from the lower dimensional subspace to the space of labels.  
See an illustrative example of \citet{collins2021exploiting} in supplementary materials. We generalize this common structure studied in \citet{collins2021exploiting} to the decentralized setting, with which, (\ref{eq:decen_objective})  can be reformulated as 
\begin{align}\label{obj}
    \min_{\pmb{\phi} \in \Phi}  \min_{\pmb{\theta}_i \in \Theta} f(\pmb{\phi},\{\pmb{\theta}_i\}_{i=1}^N):=\frac{1}{N} \sum_{i=1}^N  F_i(\pmb{\theta}_i \circ \pmb{\phi}),\allowdisplaybreaks
\end{align}
where $\Phi$ is the class of feasible representations and $\Theta$ is the class of feasible heads.  
In our proposed decentralized learning scheme, workers collaborate to learn global representation $\pmb{\phi}$ using all workers' data, while using their local information to learn personalized local heads $\{\pmb{\theta}_i\}_{i=1}^N$. In other words, worker $i$ maintains \textit{a local estimate} of global representation $\pmb{\phi}_i(k)$ at each round $k$ and broadcasts it to its neighbors in $\mathcal{N}_i$, while the local head $\pmb{\theta}_i(k)$ is only updated locally.

\section{\DePRL Algorithm} \label{sec:alg}

In personalized decentralized learning, workers aim to learn the global representation $\pmb{\phi}$ collaboratively, while each worker $i$ aims to learn a unique local head $\pmb{\theta}_i$ locally.  To achieve this, we propose a stochastic gradient descent (SGD)-based algorithm named \DePRL that solves~(\ref{obj}) \textbf{\textit{in a fully decentralized manner}}.  Specially, \DePRL alternates between three steps among all workers at each communication round: (a) local head update; (b) local representation update; and (c) consensus-based global representation update.

\textbf{Local Head Update.} At round $k$, worker $i$ makes $\tau$ local stochastic gradient-based updates to solve for its optimal local head $\pmb{\theta}_i(k)$ given the current global representation $\pmb{\phi}_i(k)$ on its local side. In other words, {for $s=0,\ldots, \tau-1$,} worker $i$ updates its local head as
\begin{align}\label{eq:local_model_update}
    \pmb{\theta}_i(k, s+1) = \pmb{\theta}_i(k, s)-\alpha g_{\pmb{\theta}}(\pmb{\phi}_i(k), \pmb{\theta}_i(k,s)),
\end{align}
where $\alpha$ is the learning rate for local head and  $g_{\pmb{\theta}}( \pmb{\phi}_i(k), \pmb{\theta}_i(k,s))$ is a stochastic gradient  of local head $\pmb{\theta}_i(k,s)$ given the global representation  $\pmb{\phi}_i(k)$ on its side: 
\begin{align}\label{eq:local-subgradient1}\nonumber
&g_{\pmb{\theta}}(\pmb{\phi}_i(k), \pmb{\theta}_i(k,s))\\
&\!\!\!\!:=\!\frac{1}{|\cC_i(k,s)|} \!\sum\limits_{(\bx_m, y_m)\in \cC_i(k,s)} \!\!\!\!\!\!\!\!\!\!\!\nabla_{\pmb{\theta}} F_i(\pmb{\phi}_i(k), \pmb{\theta}_i(k,s), \bx_m, y_m),
\end{align}
where $\mathcal{C}_i(k,s)$ is a random subset of $\mathcal{D}_i$.  We allow each worker to perform $\tau$ steps local updates to find the optimal local head based on its local data.  For ease of presentation, we denote $\pmb{\theta}_i(k+1):=\pmb{\theta}_i(k+1,0)=\pmb{\theta}_i(k,\tau-1)$.

\begin{figure*}[t]
\centering
   \includegraphics[width=1\textwidth]{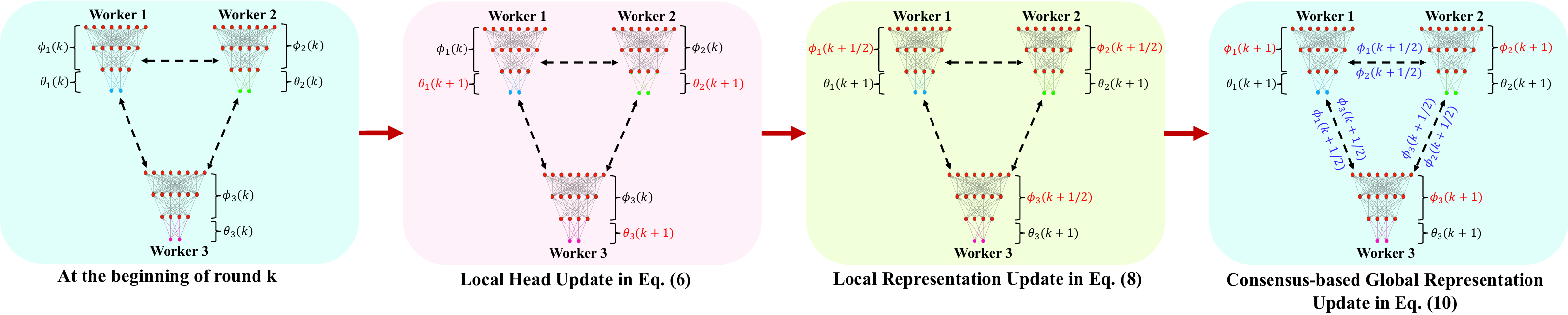}
\vspace{-0.2in}
\caption{An illustrative example of \DePRL for 3 workers with the communication graph being a ring (indicated by black dashed lines).  (a) At the beginning of each round $k$, each worker $i=1,2,3$ has the local head $\pmb{\theta}_i(k)$ and the global representation $\pmb{\phi}(k)$ on its side, which we denote as $\pmb{\phi}_i(k).$ (b) \textit{Local Head Update:} With $(\pmb{\theta}_i(k),\pmb{\phi}_i(k))$, each worker $i$ performs $\tau$ steps SGD to obtain $\pmb{\theta}_i(k+1)$.  Note that $\pmb{\phi}_i(k)$ remains unchanged at this step and the updated $\pmb{\theta}_i(k+1)$ depends on both $\pmb{\theta}_i(k)$ and $\pmb{\phi}_i(k)$. (c) \textit{Local Representation Update:} Each worker $i$ then updates the global representation on its side by executing one-step SGD to obtain $\pmb{\phi}_i(k+1/2),$ which depends on both $\pmb{\theta}_i(k+1)$ and $\pmb{\phi}_i(k)$. (d) \textit{Consensus-based Global Representation Update:} Each worker $i$ shares $\pmb{\phi}_i(k+1/2)$ with its neighbors and then executes a consensus step to produce the next global representation model $\pmb{\phi}_i(k+1)$.  We highlight the updated parameters in each step in red, and the shared parameters (only the global representation) between workers in blue.}
        \label{fig:example}
\vspace{-0.1in}
\end{figure*}

\textbf{Local Representation Update.} Once the updated local heads $\pmb{\theta}_i(k+1)$ are obtained, each worker $i$ executes one-step local update on their representation parameters, i.e., 
 \begin{align}\label{eq:global_model_update}
     \pmb{\phi}_i(k+1/2)=\pmb{\phi}_i(k)-\beta g_{\pmb{\phi}}(\pmb{\phi}_i(k), \pmb{\theta}_i(k+1)),
 \end{align} 
where $\beta$ is the learning rate for global representation and $g_{\pmb{\phi}}(\pmb{\phi}_i(k), \pmb{\theta}_i(k+1))$ is the stochastic gradient  of global representation $\pmb{\phi}_i(k)$ given the updated local head $\pmb{\theta}_i(k+1)$:  
\begin{align}\label{eq:local-subgradient2}\nonumber
&g_{\pmb{\phi}}(\pmb{\phi}_i(k), \pmb{\theta}_i(k+1))\\
&\!\!\!\!:=\!\frac{1}{|\cC_i(k)|} \sum\limits_{(\bx_m, y_m)\in \cC_i(k)} \!\!\!\!\!\!\!\!\!\!\!\nabla_{\pmb{\phi}} F_i(\pmb{\phi}_i(k), \pmb{\theta}_i(k+1), \bx_m, y_m). 
\end{align}

\textbf{Consensus-based Global Representation Update.} Each worker $i$ broadcasts its local representation update $\pmb{\phi}_i(k+1/2)$ to its neighbors $j\in\cN_i$, and computes a weighted average (i.e., consensus component) of local representation updates $\pmb{\phi}_j(k+1/2)$ received from its neighbors $j$ to produce the next representation model $\pmb{\phi}_i(k+1)$ on its side: 
 \begin{align}\label{eq:global_model_update2}
     \pmb{\phi}_i(k+1) =\sum\nolimits_{j\in \cN_i} \pmb{\phi}_j(k+1/2)P_{i,j},
 \end{align} 
where $\bP=(P_{i,j})$ is a $N\times N$ non-negative  \textit{consensus matrix}.

\DePRL alternates between \eqref{eq:local_model_update}, \eqref{eq:global_model_update} and \eqref{eq:global_model_update2} at each round, and the entire procedure is summarized in Algorithm \ref{alg:general}. An example of \DePRL with $3$ workers is illustrated in Figure \ref{fig:example}.

\begin{remark}\label{remark-alg}
We update parameters using SGD in~(\ref{eq:local-subgradient1}) and~(\ref{eq:local-subgradient2}); however, \DePRL can be easily incorporated with other methods such as gradient descent with momentum.  Further, we perform one-step update on global representation in~(\ref{eq:global_model_update}) given that we perform $\tau$-step updates on local head in~(\ref{eq:local_model_update}). However, \DePRL can be easily generalized to multi-step representation update in~(\ref{eq:global_model_update}).  Finally, we note that our representation model is inspired by \citet{collins2021exploiting}, which considered the PS setting. The theoretical analysis in \citet{collins2021exploiting} focused on showing that the learned representation converges to a ground-truth under the assumption that the global representation must be linear.  In contrast, we consider a fully decentralized framework and our convergence analysis for \DePRL is under the general non-linear representations. In addition, we numerically evaluate the generalization performance of \DePRL under non-linear representations.
\end{remark}

\begin{algorithm}[tb]
\caption{\DePRL}
\begin{algorithmic}[1] \label{alg:general}
\STATE \textbf{Parameters:} Learning rates $\alpha, \beta$; update step number for local head  $\tau$; number of communication rounds $K$.
\STATE Initialize $\pmb{\phi}(0), \pmb{\theta}_1(0,0), \dots, \pmb{\theta}_N(0,0).$
 \FOR{$k=0 ,1, \dots, K-1$}
  \FOR{$i=1,\cdots, N$ }

  \FOR{$s=0,\ldots, \tau-1$} 
\STATE $\pmb{\theta}_i(k, s+1) \leftarrow  \pmb{\theta}_i(k,s)-\alpha g_{\pmb{\theta}}(\pmb{\phi}_i(k), \pmb{\theta}_i(k,s))$; 
\ENDFOR
\STATE $\pmb{\phi}_i(k+1/2)\leftarrow\pmb{\phi}_i(k)-\beta g_{\pmb{\phi}}(\pmb{\phi}_i(k), \pmb{\theta}_i(k+1))$; 
\STATE $\pmb{\phi}_i(k+1) \leftarrow\sum_{j\in \cN_i} \pmb{\phi}_j(k+1/2)P_{i,j}$; 
 \STATE Worker $i$ initializes $\pmb{\theta}_i(k+1,0)\leftarrow \pmb{\theta}_i(k, \tau-1)$.
  \ENDFOR
 \ENDFOR
 \end{algorithmic}
\end{algorithm}

\section{Convergence Analysis}\label{sec:convergence}
In this section, we provide a rigorous analysis of the convergence of \DePRL in the decentralized setting with the general non-linear representation model.

\subsection{$\epsilon$-Approximation Solution}\label{sec:epsilon-approximation}
We first introduce the notion of $\epsilon$-approximation solution.  We denote $\bar{\pmb{\phi}}(k):=
    \frac{1}{N}\sum_{i=1}^N \pmb{\phi}_i(k)$ as the consensus global representation, and  the partial gradients of the global loss function with respect to 
(w.r.t.) $\pmb{\theta}$ and $\pmb{\phi}$ as $\nabla_{\pmb{\theta}}f(\bar{\pmb{\phi}}(k),\{\pmb{\theta}_i(k)\}_{i=1}^N)$ and $\nabla_{\pmb{\phi}}f(\bar{\pmb{\phi}}(k),\{\pmb{\theta}_i(k+1)\}_{i=1}^N)$, respectively, satisfying     
\begin{align}\label{eq:partial}
&\nabla_{\pmb{\theta}}f(\bar{\pmb{\phi}}(k),\{\pmb{\theta}_i(k)\}_{i=1}^N):=\frac{1}{N}\sum_{i=1}^N\nabla_{\pmb{\theta}}F_i(\bar{\pmb{\phi}}(k),\!\pmb{\theta}_i(k)),\nonumber\allowdisplaybreaks\\
&\nabla_{\pmb{\phi}}f(\bar{\pmb{\phi}}(k),\{\pmb{\theta}_i(k+1)\}_{i=1}^N)\!:=\!\frac{1}{N}\!\sum_{i=1}^N \nabla_{\pmb{\phi}}F_i(\bar{\pmb{\phi}}(k),\pmb{\theta}_i(k+1)), 
\end{align}
where we remark that in \DePRL each worker $i$ first updates its local head $\pmb{\theta}_i(k)$ to $\pmb{\theta}_i(k+1)$, and then updates global representation $\bar{\pmb{\phi}}(k)$ provided $\pmb{\theta}_i(k+1)$, see Algorithm~\ref{alg:general}.

Then, we say that $\{ \{\pmb{\phi}_i(k)\}_{i=1}^N, \{\pmb{\theta}_i(k)\}_{i=1}^N, \forall k\}$ is an $\epsilon$-approximation solution to~(\ref{obj}) if it satisfies 
\begin{align}\label{eq:epsilon}
    \frac{1}{K}\sum_{k=1}^K \mathbb{E}[M(k)]\leq \epsilon,
\end{align}
where 
\begin{align}\label{eq:lyapunov-function3}\nonumber
M(k)&:=\underset{\text{partial gradient w.r.t.}~\pmb{\phi}~\text{of global loss function}}{\underbrace{\left\|\nabla_{\pmb{\phi}}f(\bar{\pmb{\phi}}(k),\{\pmb{\theta}_i(k+1)\}_{i=1}^N)\right\|^2}}\nonumber\allowdisplaybreaks\\
&\qquad\qquad+\frac{\alpha\tau}{\beta}\underset{\text{partial gradient w.r.t.}~\pmb{\theta}~\text{of global loss function}}{\underbrace{\left\|\nabla_{\pmb{\theta}}f(\bar{\pmb{\phi}}(k),\{\pmb{\theta}_i(k)\}_{i=1}^N)\right\|^2}}\nonumber\allowdisplaybreaks\\
&\qquad\qquad+\underset{\text{consensus error of global representation}~\pmb{\phi}}{\underbrace{\frac{1}{N}\sum_{i=1}^N\|\pmb{\phi}_i(k)-\bar{\pmb{\phi}}(k)\|^2}}. 
\end{align}
The first two terms in \eqref{eq:lyapunov-function3} characterizes the performance of \DePRL and the third term measures the average error of global representation from the perspective of each worker $i$'s local representation update $\pmb{\phi}_i(k)$. Since \DePRL iteratively updates the local head $\{\pmb{\theta}_i(k)\}_{i=1}^N$ and representation $\{\pmb{\phi}_i(k)\}_{i=1}^N$ using different learning rates, i.e., $\tau$-step updates with a rate $\alpha$ and one-step update with a rate $\beta$, as shown in~(\ref{eq:local_model_update}) and~(\ref{eq:global_model_update}), we consider weighted partial gradients w.r.t. the global loss function in the first two terms.  This is inspired by finite-time analysis of two-timescale stochastic approximation \cite{borkar2009stochastic}. Finally, \eqref{eq:lyapunov-function3} does not explicitly include the local head error due to two reasons. First, we consider general non-convex loss functions, and the local optimum $\{\pmb{\theta}_i^*\}_{i=1}^N$ is often unknown. 
More importantly, the impact of local head $\{\pmb{\theta}_i\}_{i=1}^N$ is evaluated by partial gradients of local loss functions $F_i, \forall i$, which is implicitly incorporated in the first two terms in \eqref{eq:lyapunov-function3}, with definitions given in~(\ref{eq:partial}).

\begin{remark}
The $\epsilon$-approximation solution defined in~(\ref{eq:epsilon}) and \eqref{eq:lyapunov-function3} incorporates the impact of two coupled parameters, while conventional decentralized learning frameworks such as \citet{lian2017can,assran2019stochastic,xiong2023straggler} only considered a single shared model.  This makes existing convergence analysis not directly applicable to ours and necessitates different proof techniques.  One fundamental reason is that, instead of learning only a single shared model, there are multiple local heads strongly coupled with the global representation that need to be handled in our setting.  Compared to the PS framework, only the gap between the learned global representation $\pmb{\phi}$ and the global optimum $\pmb{\phi}^*$ under a linear representation  model is considered in \citet{collins2021exploiting}.  Finally, another line of work on decentralized bilevel optimization \cite{liu2022interact,qiu2022diamond} involves two coupled parameters under the assumption that inner parameters are strongly convex in outer parameters, and hence differ from our model and definition in \eqref{eq:lyapunov-function3}. 
\end{remark}

\subsection{Assumptions}\label{sec:assumptions}

\begin{assumption}[Doubly Stochastic Consensus Matrix]\label{assumption-weight} 
The consensus matrix $\bP=(P_{i,j})$
is doubly stochastic, i.e., $\sum_{j=1}^N P_{i,j}=\sum_{i=1}^N P_{i,j}=1, \forall i\in[N], j\in[N].$
\end{assumption}

\begin{assumption}[$L$-Lipschitz Continuous Gradient]\label{assumption-lipschitz}
There exists a constant $L>0$, such that 
$\|\nabla_{\pmb{\phi}} F_i(\pmb{\phi},\pmb{\theta})-\nabla_{\pmb{\phi}} F_i(\pmb{\phi}^\prime,\pmb{\theta}^\prime)\|\leq L(\|\pmb{\phi}-\pmb{\phi}^\prime\|+\|\pmb{\theta}-\pmb{\theta}^\prime\|)$ and
$\|\nabla_{\pmb{\theta}} F_i(\pmb{\phi},\pmb{\theta})-\nabla_{\pmb{\theta}} F_i(\pmb{\phi}^\prime,\pmb{\theta}^\prime)\|\leq L(\|\pmb{\phi}-\pmb{\phi}^\prime\|+\|\pmb{\theta}-\pmb{\theta}^\prime\|)$, $~\forall i\in[N], \forall \pmb{\phi}, \pmb{\phi}^\prime\in{\Phi}, \pmb{\theta}, \pmb{\theta}^\prime\in{\Theta}.$
\end{assumption}

\begin{assumption}[Unbiased Local Gradient Estimator]\label{assumption-gradient}
The local gradient estimators are unbiased, i.e., $\forall \pmb{\phi}_i, \pmb{\phi}^\prime_i\in{\Phi},$ $\forall\pmb{\theta}_i, \pmb{\theta}^\prime_i\in{\Theta},$ $\forall i\in[N],$ 
    $\mathbb{E}[g_{\pmb{\phi}}(\pmb{\phi}_i,\pmb{\theta}_i)]=\nabla_{\pmb{\phi}} F_i(\pmb{\phi}_i, \pmb{\theta}_i),~
    \mathbb{E}[g_{\pmb{\theta}}(\pmb{\phi}_i,\pmb{\theta}_i)]=\nabla_{\pmb{\theta}} F_i(\pmb{\phi}_i, \pmb{\theta}_i),$
with the expectation being taken over the local data samples.
\end{assumption}

\begin{assumption}[Bounded Variance]\label{assumption-variance}
{There exists a constant $\sigma>0$ such that the variance of each local gradient estimator is bounded}, i.e.,  $\forall \pmb{\phi}_i, \pmb{\phi}^\prime_i\in{\Phi},$ $\forall\pmb{\theta}_i, \pmb{\theta}^\prime_i\in{\Theta},$ $\forall i\in[N],$ 
    $\mathbb{E}[\|g_{\pmb{\phi}}(\pmb{\phi}_i,\pmb{\theta}_i)-\nabla_{\pmb{\phi}} F_i(\pmb{\phi}_i,\pmb{\theta}_i)\|^2]\leq \sigma^2, ~
    \mathbb{E}[\|g_{\pmb{\theta}}(\pmb{\phi}_i,\pmb{\theta}_i)-\nabla_{\pmb{\theta}} F_i(\pmb{\phi}_i,\pmb{\theta}_i)\|^2]\leq \sigma^2.$
\end{assumption}

\begin{assumption}[Bounded Global Variability]\label{assumption:global-var}
 There exists a constant  $\varsigma>0$ such that the global variability of the local partial gradients on $\pmb{\phi}$ of the loss function $\forall \pmb{\theta}_i\in\Theta$ is bounded, i.e., 
   $\frac{1}{N}\sum_{i=1}^N \mathbb{E}[\|\nabla_{\pmb{\phi}} F_i({\pmb{\phi}},\pmb{\theta}_i)-\nabla_{\pmb{\phi}} f({\pmb{\phi}},\{\pmb{\theta}_i\}_{i=1}^N)\|^2]\leq\varsigma^2.$

\end{assumption}

Assumptions~\ref{assumption-weight}-\ref{assumption:global-var} are standard \citep{kairouz2019advances,tang2020communication,yang2021achieving}, except the difference caused by two coupled parameters in our representation learning model. We use a universal bound $\varsigma$ to quantify the global variability of local partial gradients on global representation $\pmb{\phi}$ in Assumption \ref{assumption:global-var} due to the non-i.i.d. data among workers, which is similar to the heterogeneity assumption in conventional decentralized frameworks with a single global parameter, where $\|\nabla_{\bw} F_i(\bw)-\nabla_{\bw} f(\bw)\|^2\leq \varsigma^2, \forall i\in[N]$ with $\varsigma=0$ meaning i.i.d. data across workers \cite{lian2017can}.  Finally, it is worth noting that we do \textit{not} require a bounded gradient assumption, which is often used in distributed optimization analysis \citep{nedic2009distributed}.

\subsection{Convergence Analysis for \DePRL}\label{sec:convergence-results}
\begin{theorem}\label{thm:loss_convergence}
Under Assumptions \ref{assumption-weight}-\ref{assumption:global-var}, we choose learning rates satisfying $\alpha\leq \frac{1+36\tau^2}{\tau L}$ and $\beta\leq \min\left(1/L,N/2,\frac{1-q}{3\sqrt{2}CLN}\right)$, where $C:=\frac{2(1+p^{-N})}{1-p^{N}}$, $q:=(1-p^{N})^{1/N}$ and $p=\arg\min P_{i,j}, \forall i,j, P_{i,j}>0.$ 
 Denote the optimal parameters of global representation and local heads as $\pmb{\phi}^*$ and $\{\pmb{\theta}_i^*\}_{i=1}^N$, respectively.  The sequence of parameters  $\{ \{\pmb{\phi}_i(k)\}_{i=1}^N,  \{\pmb{\theta}_i(k)\}_{i=1}^N,\forall k\}$ generated by \DePRL satisfy
\begin{align}\label{eq:thm_loss}
&\frac{1}{K}\!\sum\limits_{k=0}^{K-1}\mathbb{E}[M(k)]
\leq\frac{4f(\bar{\pmb{\phi}}(0),\!\{\pmb{\theta}_i(0)\}_{i=1}^N)\!-\!4 f({\pmb{\phi}^*},\!\{\pmb{\theta}_i^*\}_{i=1}^N)}{K\beta}\nonumber\displaybreak[1]\\
&+\frac{2\beta L}{N}\sigma^2+\frac{12\alpha^3L^2\tau}{\beta}(\tau-1)(6\tau+1)\sigma^2+\frac{2\alpha^2\tau L}{\beta}\sigma^2\nonumber\displaybreak[1]\\
&+\frac{2\beta}{3 N}\!\left(\!1\!+\!\frac{1}{L^2}\!\right)\!\sigma^2\!+\!\frac{2\beta}{N}\!\left(\!1\!+\!\frac{1}{L^2}\!\right)\!\varsigma^2.
\end{align}
\end{theorem}
There are  two terms on right hand side of \eqref{eq:thm_loss}: (i) a vanishing term $ \frac{4f(\bar{\pmb{\phi}}(0),\{\pmb{\theta}_i(0)\}_{i=1}^N)-4f({\pmb{\phi}}^*,\{\pmb{\theta}_i^*\}_{i=1}^N)}{K\beta}$ that goes to zero as $K$ increases; and (ii) a constant noise term $\frac{2\beta L}{N}\sigma^2+\frac{12\alpha^3L^2\tau}{\beta}(\tau-1)(6\tau+1)\sigma^2
+\frac{2\alpha^2\tau L}{\beta}\sigma^2+\frac{2\beta}{3 N}\left(1+\frac{1}{L^2}\right)\sigma^2+\frac{2\beta}{ N}\left(1+\frac{1}{L^2}\right)\varsigma^2$ that depends on the problem instance and is independent of $K$.  
The  decay rate of the vanishing term  matches that of conventional decentralized learning frameworks with a single shared model \cite{lian2017can, lian2018asynchronous, assran2019stochastic}.  The constant noise term mainly comes from the variance of stochastic partial gradients on consensus global representation $\bar{\pmb{\phi}}(k)$ and local head $\{\pmb{\theta}_i(k)\}_{i=1}^N$, as well as the global variability due to the model heterogeneity when bounding the consensus error of global representation, i.e., $\|\bar{\pmb{\phi}}(k)-\pmb{\phi}_i(k)\|^2, \forall i\in[N]$ at each round $k$.  In particular, the noise term $\frac{12\alpha^3L^2\tau}{\beta}(\tau-1)(6\tau+1)\sigma^2+\frac{2\alpha^2\tau L}{\beta}\sigma^2$ is caused by $\tau$-step local head update, where the first part measures the accumulated error between each intermediate update  $\pmb{\theta}_i(k,s), \forall s\in\{0, 1, \ldots, \tau-1\}$  and $\pmb{\theta}_i(k)$, i.e., $\mathbb{E}\|\pmb{\theta}_i(k,s)-\pmb{\theta}_i(k)\|^2,$ and goes to zero when $\tau=1$ due to the fact that $\pmb{\theta}_i(k,0)=\pmb{\theta}_i(k)$. To lower its impact, an inverse relationship between the local head learning rate $\alpha$ and the updated steps $\tau$ in each round is desired, i.e., $\alpha=\mathcal{O}(\frac{1}{\tau})$, such that the error can be offset by a small $\alpha$. This is consistent with observations in the PS setting with non-IID datasets across workers \cite{yang2021achieving}.

\begin{corollary}\label{cor:1}
Let $\alpha=\frac{1}{\tau\sqrt{K}}$ and $\beta=\sqrt{{N}/{K}}$. The convergence rate of \DePRL is 
$\mathcal{O}\left(\frac{1}{\sqrt{NK}}+\frac{1}{K\sqrt{N}}+\frac{1}{\tau\sqrt{NK}}\right),$ 
when the total number of communication rounds $K$ satisfies 
    $K\geq 
    \max\left(\frac{18C^2L^2N^3}{(1-q)^2},\frac{(2L^2+2)^2}{NL^4},NL^2\right).$
\end{corollary}

Since $\frac{1}{K\sqrt{N}}$ and $\frac{1}{\tau\sqrt{NK}}$ are dominated by $
\frac{1}{\sqrt{NK}}$, \DePRL with two coupled parameters achieves  a \textbf{linear speedup} for convergence, i.e., we can proportionally decrease $K$ as $N$ increases while keeping the same convergence rate. This is the first linear speedup result for personalized decentralized learning with shared representations, and is highly desirable since it implies that one can efficiently leverage the massive parallelism in large-scale decentralized systems. An interesting point is that our result also indicates that the number
of local updates $\tau$ does not hurt the convergence with a proper learning rate choice for $\alpha$ as observed in PS setting \cite{yang2021achieving}. {Need to mention that local SGD steps usually slow down the convergence around $\mathcal{O}(\frac{\tau}{K})$ even for strongly convex objectives as shown in \citet{li2019convergence}.}

\begin{table*}[t]
\centering
\caption{Average test accuracy  with different communication graphs and data heterogeneities.} 
\scalebox{0.8}{
\begin{tabular}{|c|c|ccc|ccc|ccc|}
\hline
\multirow{2}{*}{\begin{tabular}[c]{@{}c@{}}Dataset\\ (Model)\end{tabular}}           & \multirow{2}{*}{$\pi$} & \multicolumn{3}{c|}{Ring}                                         & \multicolumn{3}{c|}{Random}                                       & \multicolumn{3}{c|}{FC}                                           \\ \cline{3-11} 
&                           & \multicolumn{1}{c|}{D-PSGD} & \multicolumn{1}{c|}{DisPFL} & DePRL & \multicolumn{1}{c|}{D-PSGD} & \multicolumn{1}{c|}{DisPFL} & DePRL & \multicolumn{1}{c|}{D-PSGD} & \multicolumn{1}{c|}{DisPFL} & DePRL \\ \hline\hline

\multirow{3}{*}{\begin{tabular}[c]{@{}c@{}}CIFAR-100\\ (ResNet-18)\end{tabular}}     
&0.1 &\multicolumn{1}{c|}{25.18\scriptsize{$\pm$0.4}} &\multicolumn{1}{c|}{46.09\scriptsize{$\pm$0.2}} &\textbf{60.72}\scriptsize{$\pm$0.2} &\multicolumn{1}{c|}{30.27\scriptsize{$\pm$0.5}} &\multicolumn{1}{c|}{47.77\scriptsize{$\pm$0.3}} &\textbf{61.51}\scriptsize{$\pm$0.5} &\multicolumn{1}{c|}{33.04\scriptsize{$\pm$0.7}} &\multicolumn{1}{c|}{47.96\scriptsize{$\pm$0.3}} &\textbf{62.40}\scriptsize{$\pm$0.7}  \\

&0.3 &\multicolumn{1}{c|}{26.51\scriptsize{$\pm$0.4}} &\multicolumn{1}{c|}{37.92\scriptsize{$\pm$0.3}} &\textbf{49.82}\scriptsize{$\pm$0.4} &\multicolumn{1}{c|}{31.71\scriptsize{$\pm$0.3}} &\multicolumn{1}{c|}{39.91\scriptsize{$\pm$0.5}} &\textbf{50.49}\scriptsize{$\pm$0.7} &\multicolumn{1}{c|}{34.93\scriptsize{$\pm$0.7}} &\multicolumn{1}{c|}{40.37\scriptsize{$\pm$0.5}} &\textbf{51.51}\scriptsize{$\pm$0.5}  \\

&0.5 &\multicolumn{1}{c|}{26.90\scriptsize{$\pm$0.3}} &\multicolumn{1}{c|}{35.33\scriptsize{$\pm$0.4}} &\textbf{45.89}\scriptsize{$\pm$0.4} &\multicolumn{1}{c|}{32.05\scriptsize{$\pm$0.4}} &\multicolumn{1}{c|}{37.39\scriptsize{$\pm$0.3}} &\textbf{46.68}\scriptsize{$\pm$0.4} &\multicolumn{1}{c|}{35.22\scriptsize{$\pm$0.6}} &\multicolumn{1}{c|}{37.86\scriptsize{$\pm$0.3}} &\textbf{47.63}\scriptsize{$\pm$0.3}   \\ \hline\hline

\multirow{3}{*}{\begin{tabular}[c]{@{}c@{}}CIFAR-10\\ (VGG-11)\end{tabular}}         
&0.1 &\multicolumn{1}{c|}{53.91\scriptsize{$\pm$0.2}} &\multicolumn{1}{c|}{86.38\scriptsize{$\pm$0.3}} &\textbf{89.57}\scriptsize{$\pm$0.2} &\multicolumn{1}{c|}{57.69\scriptsize{$\pm$0.2}} &\multicolumn{1}{c|}{89.01\scriptsize{$\pm$0.2}} &\textbf{91.03}\scriptsize{$\pm$0.1} &\multicolumn{1}{c|}{58.90\scriptsize{$\pm$0.1}} &\multicolumn{1}{c|}{89.19\scriptsize{$\pm$0.3}} &\textbf{91.33}\scriptsize{$\pm$0.2}  \\

&0.3 &\multicolumn{1}{c|}{59.17\scriptsize{$\pm$0.2}} &\multicolumn{1}{c|}{73.48\scriptsize{$\pm$0.3}} &\textbf{76.41}\scriptsize{$\pm$0.3} &\multicolumn{1}{c|}{64.12\scriptsize{$\pm$0.4}} &\multicolumn{1}{c|}{77.36\scriptsize{$\pm$0.4}} &\textbf{79.60}\scriptsize{$\pm$0.2} &\multicolumn{1}{c|}{65.82\scriptsize{$\pm$0.3}} &\multicolumn{1}{c|}{78.52\scriptsize{$\pm$0.5}} &\textbf{79.84}\scriptsize{$\pm$0.4}  \\

&0.5 &\multicolumn{1}{c|}{60.45\scriptsize{$\pm$0.4}} &\multicolumn{1}{c|}{68.83\scriptsize{$\pm$0.2}} &\textbf{72.51}\scriptsize{$\pm$0.2} &\multicolumn{1}{c|}{65.48\scriptsize{$\pm$0.2}} &\multicolumn{1}{c|}{73.30\scriptsize{$\pm$0.4}} &\textbf{74.80}\scriptsize{$\pm$0.2} &\multicolumn{1}{c|}{67.30\scriptsize{$\pm$0.3}} &\multicolumn{1}{c|}{74.50\scriptsize{$\pm$0.3}} &\textbf{75.04}\scriptsize{$\pm$0.2}   \\ \hline\hline

\multirow{3}{*}{\begin{tabular}[c]{@{}c@{}}Fashion\\ MNIST\\ (AlexNet)\end{tabular}} 
&0.1 &\multicolumn{1}{c|}{77.45\scriptsize{$\pm$0.2}} &\multicolumn{1}{c|}{95.74\scriptsize{$\pm$0.2}} &\textbf{96.66}\scriptsize{$\pm$0.2} &\multicolumn{1}{c|}{84.74\scriptsize{$\pm$0.2}} &\multicolumn{1}{c|}{96.59\scriptsize{$\pm$0.3}} &\textbf{97.16}\scriptsize{$\pm$0.2} &\multicolumn{1}{c|}{87.24\scriptsize{$\pm$0.3}} &\multicolumn{1}{c|}{96.76\scriptsize{$\pm$0.3}} &\textbf{97.36}\scriptsize{$\pm$0.2}  \\

&0.3 &\multicolumn{1}{c|}{81.95\scriptsize{$\pm$0.5}} &\multicolumn{1}{c|}{91.52\scriptsize{$\pm$0.3}} &\textbf{92.81}\scriptsize{$\pm$0.2} &\multicolumn{1}{c|}{87.76\scriptsize{$\pm$0.3}} &\multicolumn{1}{c|}{93.47\scriptsize{$\pm$0.3}} &\textbf{94.81}\scriptsize{$\pm$0.2} &\multicolumn{1}{c|}{89.80\scriptsize{$\pm$0.5}} &\multicolumn{1}{c|}{93.50\scriptsize{$\pm$0.2}} &\textbf{95.03}\scriptsize{$\pm$0.2}  \\

&0.5 &\multicolumn{1}{c|}{84.63\scriptsize{$\pm$0.2}} &\multicolumn{1}{c|}{89.49\scriptsize{$\pm$0.2}} &\textbf{91.36}\scriptsize{$\pm$0.3} &\multicolumn{1}{c|}{88.93\scriptsize{$\pm$0.5}} &\multicolumn{1}{c|}{91.99\scriptsize{$\pm$0.3}} &\textbf{93.55}\scriptsize{$\pm$0.2} &\multicolumn{1}{c|}{90.38\scriptsize{$\pm$0.4}} &\multicolumn{1}{c|}{92.13\scriptsize{$\pm$0.2}} &\textbf{93.87}\scriptsize{$\pm$0.3}   \\ \hline\hline

\multirow{3}{*}{\begin{tabular}[c]{@{}c@{}}HARBox\\ (DNN)\end{tabular}}              
&0.1 &\multicolumn{1}{c|}{54.90\scriptsize{$\pm$0.7}} &\multicolumn{1}{c|}{90.96\scriptsize{$\pm$0.1}} &\textbf{92.07}\scriptsize{$\pm$0.1} &\multicolumn{1}{c|}{57.59\scriptsize{$\pm$0.6}} &\multicolumn{1}{c|}{91.36\scriptsize{$\pm$0.3}} &\textbf{92.49}\scriptsize{$\pm$0.1} &\multicolumn{1}{c|}{58.23\scriptsize{$\pm$0.3}} &\multicolumn{1}{c|}{91.47\scriptsize{$\pm$0.1}} &\textbf{93.46}\scriptsize{$\pm$0.1}  \\

&0.3 &\multicolumn{1}{c|}{55.41\scriptsize{$\pm$0.7}} &\multicolumn{1}{c|}{80.02\scriptsize{$\pm$0.2}} &\textbf{80.85}\scriptsize{$\pm$0.1} &\multicolumn{1}{c|}{57.97\scriptsize{$\pm$0.7}} &\multicolumn{1}{c|}{80.35\scriptsize{$\pm$0.2}} &\textbf{81.30}\scriptsize{$\pm$0.2} &\multicolumn{1}{c|}{58.93\scriptsize{$\pm$0.7}} &\multicolumn{1}{c|}{82.14\scriptsize{$\pm$0.2}} &\textbf{83.55}\scriptsize{$\pm$0.2} \\

&0.5 &\multicolumn{1}{c|}{56.66\scriptsize{$\pm$0.7}} &\multicolumn{1}{c|}{74.47\scriptsize{$\pm$0.1}} &\textbf{75.84}\scriptsize{$\pm$0.1} &\multicolumn{1}{c|}{58.59\scriptsize{$\pm$0.7}} &\multicolumn{1}{c|}{74.72\scriptsize{$\pm$0.3}} &\textbf{76.22}\scriptsize{$\pm$0.2} &\multicolumn{1}{c|}{59.17\scriptsize{$\pm$0.7}} &\multicolumn{1}{c|}{77.61\scriptsize{$\pm$0.3}} &\textbf{78.74}\scriptsize{$\pm$0.2}  \\ \hline
\end{tabular}
}
\label{tbl:final-accuracy}
\end{table*}

\subsection{Intuitions and Proof Sketch}\label{sec:convergence-proof}
We now highlight the key ideas and challenges behind the convergence proof of \DePRL with two coupled parameters.  Given the definition of $\epsilon$-approximation solution defined in~(\ref{eq:epsilon}) and \eqref{eq:lyapunov-function3}, we characterize the descending property of the global loss function as 
\begin{align} \label{eq:step1}
   &\hspace{-0.3cm}\mathbb{E}[f(\bar{\pmb{\phi}}(k+1),\{\pmb{\theta}_i(k+1)\}_{i=1}^N)]- \mathbb{E}[f(\bar{\pmb{\phi}}(k),\{\pmb{\theta}_i(k)\}_{i=1}^N)]\nonumber\allowdisplaybreaks\\
   &{\leq} \underset{C_1}{\underbrace{\frac{1}{N}\sum_{i=1}^N\mathbb{E}\Big\langle \nabla_{\pmb{\phi}}F_i(\bar{\pmb{\phi}}(k),\pmb{\theta}_i(k+1)), \bar{\pmb{\phi}}(k+1)\!-\!\bar{\pmb{\phi}}(k)\Big\rangle}}\nonumber\allowdisplaybreaks\\
   &+\underset{C_2}{\underbrace{\frac{1}{N}\sum_{i=1}^N\frac{L}{2}\mathbb{E}[\|\bar{\pmb{\phi}}(k+1)\!-\!\bar{\pmb{\phi}}(k)\|^2]}}\nonumber\allowdisplaybreaks\\
   &+ \underset{C_3}{\underbrace{\frac{1}{N}\sum_{i=1}^N\!\mathbb{E}\Big\langle \nabla_{\pmb{\theta}}F_i(\bar{\pmb{\phi}}(k),\pmb{\theta}_i(k)), \ \pmb{\theta}_i(k+1) \!-\! \pmb{\theta}_i(k) \Big\rangle}}\nonumber\allowdisplaybreaks\\
   &+\underset{C_4}{\underbrace{\frac{1}{N}\sum_{i=1}^N\frac{L}{2}\mathbb{E}[\|{\pmb{\theta}_i}(k+1)\!-\!{\pmb{\theta}_i}(k)\|^2]}}, 
\end{align}
by following the Lipschitz assumption. 
Bounding $C_1, C_2, C_3$ and $C_4$ leads to all key components in $M(k)$ defined in \eqref{eq:lyapunov-function3}, including the partial gradients on $\bar{\pmb{\phi}}(k)$ and $\{\pmb{\theta}_i(k)\}_{i=1}^N$ of the global loss function, the error of gradient estimation, and the average consensus error of the global representation.

As aforementioned, instead of learning a single shared model as in conventional decentralized learning frameworks \cite{lian2017can, lian2018asynchronous, assran2019stochastic}, \DePRL needs to handle multiple local heads that are strongly coupled with the global representation, which necessitates different proof techniques.  Below, we highlight several key differences: 1) \textbf{\textit{Coupled model parameters.}} The updates of local heads $\{\pmb{\theta}_i\}_{i=1}^N$ and global representations $\{\pmb{\phi}_i\}_{i=1}^N$ are strongly coupled, which makes bounding $C_1$ and $C_3$ challenging. In particular, the update of  consensus global representation $\bar{\pmb{\phi}}(k)$ depends on local heads $\{\pmb{\theta}_i(k+1)\}_{i=1}^N$ in $C_1$, and the update of local heads $\{\pmb{\theta}_i(k)\}_{i=1}^N$ depends on the  consensus global representation $\bar{\pmb{\phi}}(k)$ in $C_3$. Since the loss function is evaluated on  consensus global representation $\bar{\pmb{\phi}}(k)$, we show that bounding $C_1$ and $C_3$ can be reduced to bound the consensus error of global representation $\|\pmb{\phi}_i(k)-\bar{\pmb{\phi}}(k)\|^2$. {Specifically, the bound on consensus error allows us to control the terms involving the local partial gradients and local updates in the drift of the global loss function as shown in \eqref{eq:step1}, which also serves as a bridge to track the update of $\bar{\pmb{\phi}}(k+1)-\bar{\pmb{\phi}}(k+1)$ in $C_2$ and $\pmb{\theta}_i(k+1)-\pmb{\theta}_i(k)$ in $C_4$.} 2) \textbf{\textit{Consensus error.}} Based on \eqref{eq:global_model_update}, \eqref{eq:local-subgradient2} and \eqref{eq:global_model_update2}, the average consensus error $\mathbb{E}\|\pmb{\phi}_i(k)-\bar{\pmb{\phi}}(k)\|^2$ depends on both the consensus matrix $\bP$, and the local partial gradient on $\pmb{\phi}$, i.e., $g_{\pmb{\phi}}(\pmb{\phi}_i(k), \pmb{\theta}_i(k+1))$, which is correlated with local heads $\{\pmb{\theta}_i(k)\}_{i=1}^N$.  We address these impacts by leveraging Assumption~\ref{assumption:global-var}. 3) \textbf{\textit{Two learning rates.}} As discussed earlier, we leverage a weight term in \eqref{eq:lyapunov-function3} to capture the different learning rates for $\{\pmb{\theta}_i(k)\}_{i=1}^N$ and $\{\pmb{\phi}_i(k)\}_{i=1}^N$.  This weight benefits for characterizing the desired learning rate for convergence.

  \begin{figure*}[t]
  \centering
    \begin{minipage}{.48\textwidth}
  \centering
  \includegraphics[width=1\columnwidth]{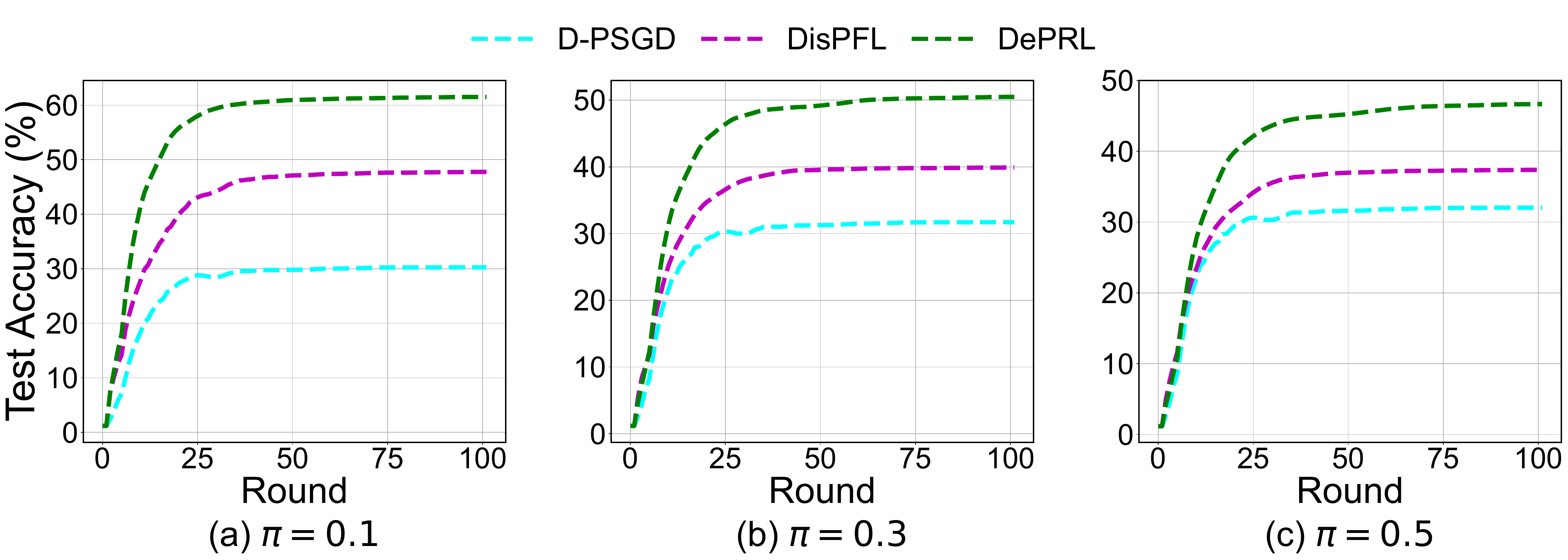}
 \subcaption{Test accuracy vs. training rounds.}
 \label{fig:CIFAR100-ResNet-random-round}
  \end{minipage}\hfill
     \begin{minipage}{.48\textwidth}
  \centering
  \includegraphics[width=1\columnwidth]{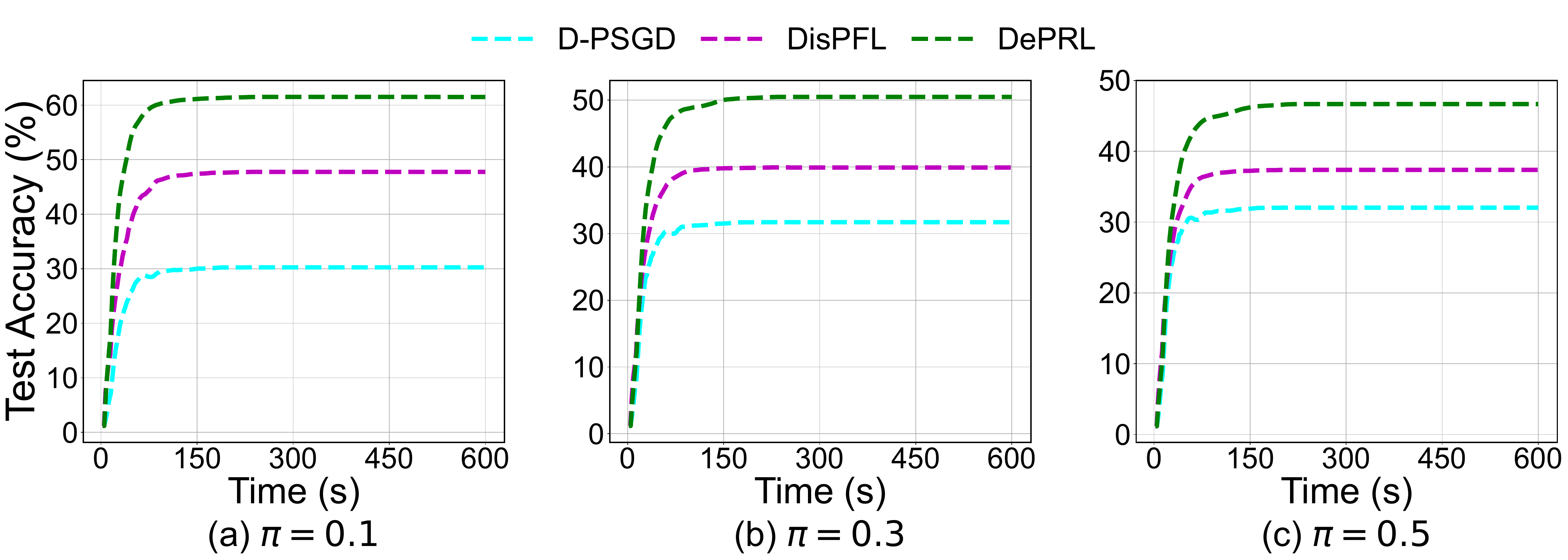}
 \subcaption{Test accuracy vs. training time.}
 \label{fig:CIFAR100-ResNet-random-time}
  \end{minipage}
  \caption{Learning curves of different baselines using ResNet-18 on non-IID partitioned CIFAR-100 with different heterogeneities when the communication graph is ``Random''.}
	\label{fig:CIFAR100-ResNet-random}
  \end{figure*}

\section{Experiments}\label{sec:sim}

We experimentally evaluate the performance of \DePRL.  Further details about experiments, hyperparameters, and additional results are provided in supplementary materials.  


\textbf{Datasets and Models.} We use (i) three image classification datasets: CIFAR-100, CIFAR-10 \cite{krizhevsky2009learning} and Fashion-MNIST \cite{xiao2017fashion}; and (ii) a human activity recognition dataset: HARBox \cite{ouyang2021clusterfl}.  We simulate non-IID scenario by considering a heterogeneous data partition for which the number of data points and class proportions are unbalanced  \citet{wang2020federated,wang2020tackling}.  In particular, we simulate a heterogeneous partition into $N$ workers by sampling $\boldsymbol p_i\sim\text{Dir}_N(\pi)$, {where $\pi$ is the parameter of Dirichlet distribution.}  
 We use ResNet-18 \cite{he2016deep} for CIFAR-100, VGG-11 \cite{simonyan2015very} for CIFAR-10, AlexNet \cite{krizhevsky2012imagenet} for Fashion-MNIST, and a fully connected DNN \cite{li2021hermes,li2022pyramidfl}.   As in \citet{collins2021exploiting}, we treat the head as the weights and biases as the final fully-connected layer in each of the models.

\textbf{Baselines.} We compare \DePRL with a diverse of baselines including both conventional decentralized learning algorithms and the popular PS based algorithms.  For decentralized setting, we take the commonly used D-PSGD \cite{lian2017can} and DisPFL \cite{dai2022dispfl}, a personalized method with a single shared model. PS based baselines include FedAvg \cite{mcmahan2017communication}, FedRep \cite{collins2021exploiting}, Ditto \cite{li2021ditto} and FedRoD \cite{hong2022fedrod}. We implement all algorithms in PyTorch \citep{paszke2017automatic} on Python~3 with three NVIDIA RTX A6000 GPUs.

\textbf{Communication Graph.} Based on our model and theoretical analysis, we randomly generate a connected communication graph (``Random'' for short) for decentralized settings.  We also experiment on two representative communication graph including ``Ring'' and ``fully connected (FC)''.  Further, since communications occur between the central server and workers in PS based setting, for a fair comparison\footnote{ Our ``fair comparison'' is defined from the perspective of total communications. Specifically, each worker sends its update only to two neighbors in ``Ring''. As a result, the total communications are the same as that for PS based schemes (i.e., a star graph).}, we only compare decentralized baselines with PS based baselines under ``Ring''.   Due to space constraints, we relegate the comparisons with PS based methods to supplementary materials. 

\textbf{Configurations.} All results are averaged over four random seeds. The final accuracy is calculated through the average of each worker's local test accuracy.   The total worker number is $128$, and the epoch number for local head update is $2$.  An ablation study is conducted in supplementary materials.


\textbf{Testing Accuracy.}  We show the final test accuracy for all considered algorithms under various settings in Table~\ref{tbl:final-accuracy}, and report the learning curve 
in Figure~\ref{fig:CIFAR100-ResNet-random}.   We observe that our \DePRL outperforms all baselines over all four datasets and three non-IID partitions. First, the state-of-the-art D-PSGD performs worse in non-IID settings due to the fact that it targets on learning a single model without encouraging personalization.  Second, though DisPFL is incorporated with personalization, and significantly improves the performance of D-PSGD, \DePRL always outperforms DisPFL. In particular, \DePRL achieves a remarkable performance improvement on non-IID partitioned CIFAR-100.  Compared to CIFAR-10, the data heterogeneity across workers are further increased due to the larger number of classes, and hence calls for personalization of local models.  This observation makes our representation learning augmented personalized model in \DePRL even pronounced compared to learning a single full-dimensional model in these baseline methods.  Finally, the superior performance of \DePRL over D-PSGD and DisPFL is consistent and robust over all communication graphs.

\begin{figure}
    \centering
 \includegraphics[width=1\columnwidth]{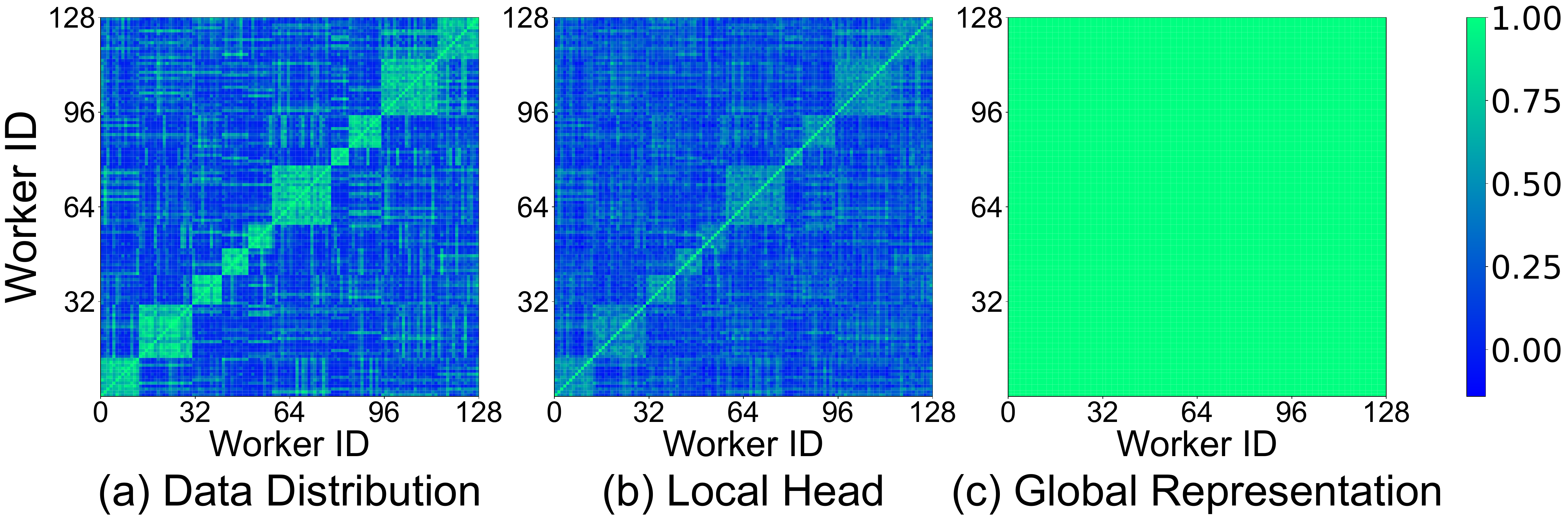}
\caption{Similarities between workers on (a) data distributions; (b) local heads; and (c) global representation.}
 \label{fig:similarity}
\end{figure}

\textbf{Learned Local Head and Global Representation.} To further advocate the benefits of \DePRL for producing personalized models via leveraging representation learning theory, we report the distance between learned local heads, global representation and task similarities. We measure the similarities by cos-similarity between data distributions, learned local heads, and learned global representation across workers.  As shown in Figure~\ref{fig:similarity} on non-IID partitioned CIFAR-100, \DePRL is able to accommodate the heterogeneities among workers not only with the learned local heads in alignment with local data distribution, but also with the same global representation.  This further validates our theoretical analysis.

\begin{figure}
    \centering
 \includegraphics[width=0.8\columnwidth]{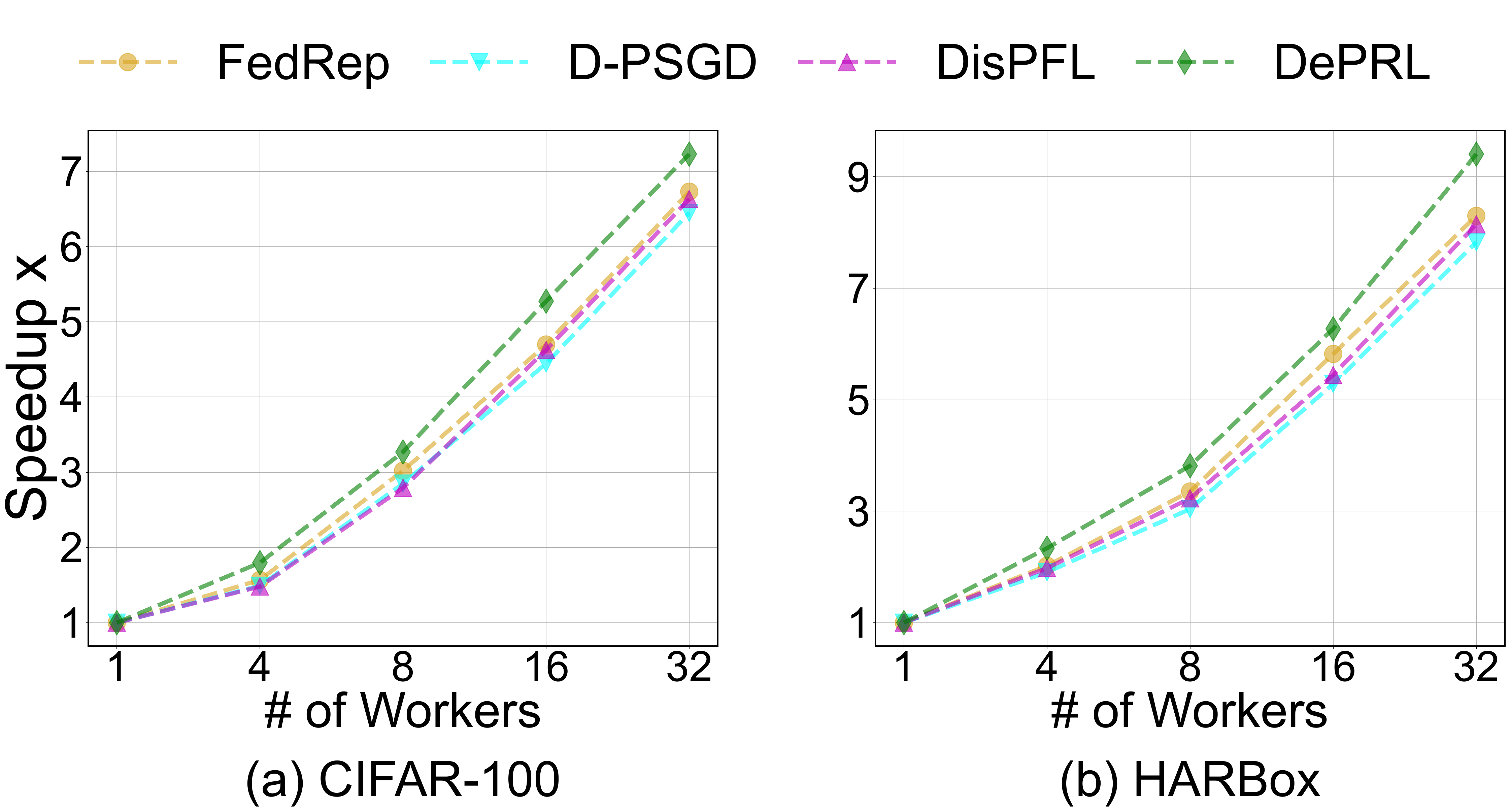}
\caption{Speedup with different number of workers.}
 \label{fig:linearspeedup}
\end{figure}

\textbf{Speedup.} We now validate our main theoretical result that \DePRL achieves a linear speedup of convergence. This often  means that the convergence time (measures how quickly the gradient norm converges to zero) should be linear in the number of workers. Specifically, we measure the convergence time of different algorithms, and compute their speedup with respect to that of a centralized setting as in D-PSGD \cite{lian2017can}. From Figure~\ref{fig:linearspeedup}, we observe that the speedup (convergence
time) is almost linearly increasing (decreasing) as the number of workers increases, which validates the linear speedup property of \DePRL.

\textbf{Generalization to New Workers.} We evaluate the effectiveness of global representation learned by \DePRL when generalizes it to new workers.  Specifically, when training all considered models using different datasets with $\alpha=0.3$, all initial 128 workers collaboratively learn a corresponding global representation $\pmb{\phi}$.  Then we encounter 64 new workers and partition the datasets across these new workers, which lead to significantly different datasets across workers compared to all initial workers.  Each new worker $i$ leverages the learned global representation by initial workers, and perform multiple local steps to learn its local head $\pmb{\theta}_i$. We then evaluate the test accuracy as above across these 64 new workers.  As shown in Table~\ref{tbl:generalization}, \DePRL significantly outperforms all baselines in the generalization performance across all communication graphs and varying levels of heterogeneity.

\begin{table}[t]
\centering
\caption{Generalization performance in terms of test accuracy.}
\scalebox{0.44}{
\begin{tabular}{|c|c|ccc|ccc|ccc|}
\hline
\multirow{2}{*}{\begin{tabular}[c]{@{}c@{}}Dataset\\ (Model)\end{tabular}}           & \multirow{2}{*}{$\pi$} & \multicolumn{3}{c|}{Ring}                                         & \multicolumn{3}{c|}{Random}                                       & \multicolumn{3}{c|}{FC}                                           \\ \cline{3-11} 
&                           & \multicolumn{1}{c|}{D-PSGD} & \multicolumn{1}{c|}{DisPFL} & DePRL & \multicolumn{1}{c|}{D-PSGD} & \multicolumn{1}{c|}{DisPFL} & DePRL & \multicolumn{1}{c|}{D-PSGD} & \multicolumn{1}{c|}{DisPFL} & DePRL \\ \hline\hline

\multirow{3}{*}{\begin{tabular}[c]{@{}c@{}}CIFAR-100\\ (ResNet-18)\end{tabular}}     
&0.1 &\multicolumn{1}{c|}{38.64\scriptsize{$\pm$0.1}} &\multicolumn{1}{c|}{34.67\scriptsize{$\pm$0.2}} &\textbf{53.58}\scriptsize{$\pm$0.3} &\multicolumn{1}{c|}{49.81\scriptsize{$\pm$0.2}} &\multicolumn{1}{c|}{42.93\scriptsize{$\pm$0.2}} &\textbf{53.78}\scriptsize{$\pm$0.2} &\multicolumn{1}{c|}{51.32\scriptsize{$\pm$0.2}} &\multicolumn{1}{c|}{47.50\scriptsize{$\pm$0.1}} &\textbf{54.16}\scriptsize{$\pm$0.2}   \\

&0.3 &\multicolumn{1}{c|}{27.75\scriptsize{$\pm$0.2}} &\multicolumn{1}{c|}{24.89\scriptsize{$\pm$0.1}} &\textbf{44.62}\scriptsize{$\pm$0.2} &\multicolumn{1}{c|}{42.27\scriptsize{$\pm$0.1}} &\multicolumn{1}{c|}{34.05\scriptsize{$\pm$0.1}} &\textbf{44.79}\scriptsize{$\pm$0.1} &\multicolumn{1}{c|}{43.42\scriptsize{$\pm$0.2}} &\multicolumn{1}{c|}{39.20\scriptsize{$\pm$0.1}} &\textbf{45.48}\scriptsize{$\pm$0.2}   \\

&0.5 &\multicolumn{1}{c|}{26.15\scriptsize{$\pm$0.2}} &\multicolumn{1}{c|}{22.93\scriptsize{$\pm$0.1}} &\textbf{42.27}\scriptsize{$\pm$0.2} &\multicolumn{1}{c|}{39.25\scriptsize{$\pm$0.1}} &\multicolumn{1}{c|}{31.11\scriptsize{$\pm$0.2}} &\textbf{42.65}\scriptsize{$\pm$0.2} &\multicolumn{1}{c|}{41.20\scriptsize{$\pm$0.1}} &\multicolumn{1}{c|}{37.26\scriptsize{$\pm$0.1}} &\textbf{42.68}\scriptsize{$\pm$0.1}    \\ \hline\hline

\multirow{3}{*}{\begin{tabular}[c]{@{}c@{}}CIFAR-10\\ (VGG-11)\end{tabular}}         
&0.1 &\multicolumn{1}{c|}{73.81\scriptsize{$\pm$0.2}} &\multicolumn{1}{c|}{70.85\scriptsize{$\pm$0.3}} &\textbf{75.93}\scriptsize{$\pm$0.2} &\multicolumn{1}{c|}{74.78\scriptsize{$\pm$0.3}} &\multicolumn{1}{c|}{73.62\scriptsize{$\pm$0.5}} &\textbf{80.22}\scriptsize{$\pm$0.2} &\multicolumn{1}{c|}{75.39\scriptsize{$\pm$0.3}} &\multicolumn{1}{c|}{73.71\scriptsize{$\pm$0.2}} &\textbf{80.48}\scriptsize{$\pm$0.2}  \\

&0.3 &\multicolumn{1}{c|}{57.34\scriptsize{$\pm$0.3}} &\multicolumn{1}{c|}{51.89\scriptsize{$\pm$0.4}} &\textbf{59.01}\scriptsize{$\pm$0.4} &\multicolumn{1}{c|}{58.86\scriptsize{$\pm$0.2}} &\multicolumn{1}{c|}{57.08\scriptsize{$\pm$0.3}} &\textbf{65.60}\scriptsize{$\pm$0.3} &\multicolumn{1}{c|}{59.68\scriptsize{$\pm$0.2}} &\multicolumn{1}{c|}{57.18\scriptsize{$\pm$0.4}} &\textbf{65.67}\scriptsize{$\pm$0.3}  \\

&0.5 &\multicolumn{1}{c|}{50.92\scriptsize{$\pm$0.2}} &\multicolumn{1}{c|}{47.36\scriptsize{$\pm$0.2}} &\textbf{52.72}\scriptsize{$\pm$0.2} &\multicolumn{1}{c|}{52.51\scriptsize{$\pm$0.3}} &\multicolumn{1}{c|}{49.75\scriptsize{$\pm$0.3}} &\textbf{63.04}\scriptsize{$\pm$0.2} &\multicolumn{1}{c|}{53.67\scriptsize{$\pm$0.2}} &\multicolumn{1}{c|}{49.81\scriptsize{$\pm$0.2}} &\textbf{63.29}\scriptsize{$\pm$0.3}    \\ \hline\hline

\multirow{3}{*}{\begin{tabular}[c]{@{}c@{}}Fashion\\ MNIST\\ (AlexNet)\end{tabular}} 
&0.1 &\multicolumn{1}{c|}{84.76\scriptsize{$\pm$0.3}} &\multicolumn{1}{c|}{83.72\scriptsize{$\pm$0.3}} &\textbf{87.34}\scriptsize{$\pm$0.2} &\multicolumn{1}{c|}{85.37\scriptsize{$\pm$0.4}} &\multicolumn{1}{c|}{85.44\scriptsize{$\pm$0.3}} &\textbf{88.48}\scriptsize{$\pm$0.2} &\multicolumn{1}{c|}{86.48\scriptsize{$\pm$0.3}} &\multicolumn{1}{c|}{85.65\scriptsize{$\pm$0.3}} &\textbf{88.86}\scriptsize{$\pm$0.2}   \\

&0.3 &\multicolumn{1}{c|}{74.84\scriptsize{$\pm$0.3}} &\multicolumn{1}{c|}{73.07\scriptsize{$\pm$0.2}} &\textbf{78.29}\scriptsize{$\pm$0.3} &\multicolumn{1}{c|}{78.54\scriptsize{$\pm$0.3}} &\multicolumn{1}{c|}{76.92\scriptsize{$\pm$0.2}} &\textbf{80.74}\scriptsize{$\pm$0.2} &\multicolumn{1}{c|}{79.16\scriptsize{$\pm$0.2}} &\multicolumn{1}{c|}{77.69\scriptsize{$\pm$0.2}} &\textbf{80.84}\scriptsize{$\pm$0.2}   \\

&0.5 &\multicolumn{1}{c|}{67.54\scriptsize{$\pm$0.3}} &\multicolumn{1}{c|}{65.64\scriptsize{$\pm$0.4}} &\textbf{71.14}\scriptsize{$\pm$0.3} &\multicolumn{1}{c|}{71.97\scriptsize{$\pm$0.2}} &\multicolumn{1}{c|}{70.31\scriptsize{$\pm$0.2}} &\textbf{77.37}\scriptsize{$\pm$0.3} &\multicolumn{1}{c|}{72.57\scriptsize{$\pm$0.4}} &\multicolumn{1}{c|}{70.61\scriptsize{$\pm$0.2}} &\textbf{77.52}\scriptsize{$\pm$0.2}    \\ \hline\hline

\multirow{3}{*}{\begin{tabular}[c]{@{}c@{}}HARBox\\ (DNN)\end{tabular}}              
&0.1 &\multicolumn{1}{c|}{51.07\scriptsize{$\pm$0.7}} &\multicolumn{1}{c|}{50.58\scriptsize{$\pm$0.6}} &\textbf{55.97}\scriptsize{$\pm$0.3} &\multicolumn{1}{c|}{51.23\scriptsize{$\pm$0.7}} &\multicolumn{1}{c|}{51.12\scriptsize{$\pm$0.7}} &\textbf{56.39}\scriptsize{$\pm$0.5} &\multicolumn{1}{c|}{51.84\scriptsize{$\pm$0.7}} &\multicolumn{1}{c|}{51.21\scriptsize{$\pm$0.6}} &\textbf{57.70}\scriptsize{$\pm$0.3}   \\

&0.3 &\multicolumn{1}{c|}{49.50\scriptsize{$\pm$0.3}} &\multicolumn{1}{c|}{48.97\scriptsize{$\pm$0.3}} &\textbf{55.86}\scriptsize{$\pm$0.3} &\multicolumn{1}{c|}{49.81\scriptsize{$\pm$0.3}} &\multicolumn{1}{c|}{49.49\scriptsize{$\pm$0.4}} &\textbf{56.11}\scriptsize{$\pm$0.3} &\multicolumn{1}{c|}{51.53\scriptsize{$\pm$0.5}} &\multicolumn{1}{c|}{49.55\scriptsize{$\pm$0.5}} &\textbf{56.39}\scriptsize{$\pm$0.3}   \\

&0.5 &\multicolumn{1}{c|}{48.24\scriptsize{$\pm$0.3}} &\multicolumn{1}{c|}{48.32\scriptsize{$\pm$0.3}} &\textbf{52.42}\scriptsize{$\pm$0.3} &\multicolumn{1}{c|}{48.28\scriptsize{$\pm$0.3}} &\multicolumn{1}{c|}{48.47\scriptsize{$\pm$0.2}} &\textbf{52.46}\scriptsize{$\pm$0.3} &\multicolumn{1}{c|}{48.82\scriptsize{$\pm$0.2}} &\multicolumn{1}{c|}{48.49\scriptsize{$\pm$0.2}} &\textbf{52.59}\scriptsize{$\pm$0.3}    \\ \hline
\end{tabular}
}
\label{tbl:generalization}
\end{table}

\section{Conclusions}\label{sec:con}

In this paper, we proposed a novel fully decentralized algorithm \DePRL by leveraging the representation learning theory to tackle the data heterogeneity for personalized decentralized learning. \DePRL learned a user-specific, and hence personalized set of local parameters for each worker, along with a global common representation parameter.  We proved that \DePRL achieves a linear speedup convergence. 
Extensive experiments also verified the superior performance of our proposed algorithm.  We note that our paper is the first to explore representation learning theory in personalized decentralized learning with a rigorous convergence analysis under general non-linear representations. This opens up several interesting directions for future research, including but not limited to adopting compression or quantization techniques to further reduce communication costs, and developing privacy-preserving decentralized algorithms over networks.    

\section*{Acknowledgements} 

This work was supported in part by the National Science Foundation (NSF) grants 2148309 and 2315614, and was supported in part by funds from OUSD R\&E, NIST, and industry partners as specified in the Resilient \& Intelligent NextG Systems (RINGS) program. This work was also supported in part by the U.S. Army Research Office (ARO) grant W911NF-23-1-0072, and the U.S. Department of Energy (DOE) grant DE-EE0009341. Any opinions, findings, and conclusions or recommendations expressed in this material are those of the authors and do not necessarily reflect the views of the funding agencies.

\bibliography{refs,refs_app}
\bibliographystyle{aaai24}

\newpage
\appendix
 \onecolumn
 \section{Related Work}\label{sec:related}

\textbf{Personalized Federated learning} targets at producing personalized models for each worker, and has attracted lots of attentions recently, e.g.,  multi-task learning \cite{smith2017federated}, meta learning \cite{chen2018federated,fallah2020personalized}, using methods including but not limited to local fine-tuning \cite{fallah2020personalized, cheng2021fine}, adding
regularization term \cite{li2021ditto,t2020personalized}, layer personalization \cite{arivazhagan2019federated}, and model compression \cite{li2021fedmask,huang2022achieving,bergou2022personalized}.  However, each worker's personalized model is still in full dimension.
More recently, representation learning method for FL has gained attentions \cite{arivazhagan2019federated,liang2020think,collins2021exploiting}.  However, the local heads and global body were jointly updated by local procedures in \citet{arivazhagan2019federated} without theoretical justification; and 
\citet{liang2020think} learned many local representations and a single
global head as opposed to a single global representation and many local heads in our setting.  The closest to ours is \citet{collins2021exploiting} in the parameter-server based setting with their theoretical analysis restricted to a linear representation model. In contrast, we consider a decentralized setting and our convergence analysis holds for general non-linear representations.

\textbf{Decentralized Learning} targets the same consensus model through peer-to-peer communication \cite{nedic2020distributed}, starting from the seminal work in control community \cite{tsitsiklis1986distributed} and more recently in machine learning community \cite{lian2017can,lian2018asynchronous,lalitha2018fully,tang2018communication,tang2018d,assran2019stochastic,hsieh2020non,neglia2020decentralized, koloskova2020unified, lin2021quasi,chen2021communication,kovalev2021lower,loizou2021revisiting,kong2021consensus,zhu2022topology,vogels2022beyond,xiong2023straggler}. To tackle  heterogeneous scenarios, a batch of personalized decentralized learning methods have been proposed \cite{vanhaesebrouck2017decentralized,chayti2021linear,dai2022dispfl}. However, in all these methods, each worker's subproblem is still full-dimensional and there is no notion of learning local parameters with reduced dimensions.  The most recent DisPFL \cite{dai2022dispfl} achieves superior performance with no convergence analysis, while \DePRL not only outperforms DisPFL but also with a strong convergence rate guarantee. 

\section{Additional Discussions on Shared Representations for Decentralized and PS Frameworks}

An illustrative example for such conventional decentralized frameworks with 3 workers is presented in Figure~\ref{fig:conventional-DL}.  As discussed in the system model, these conventional decentralized learning frameworks aim at learning a single shared model, in contrast to our model considered in this paper with an example illustrated in Figure~\ref{fig:example}. 

\begin{figure}[t]
\centering
   \includegraphics[width=1\textwidth]{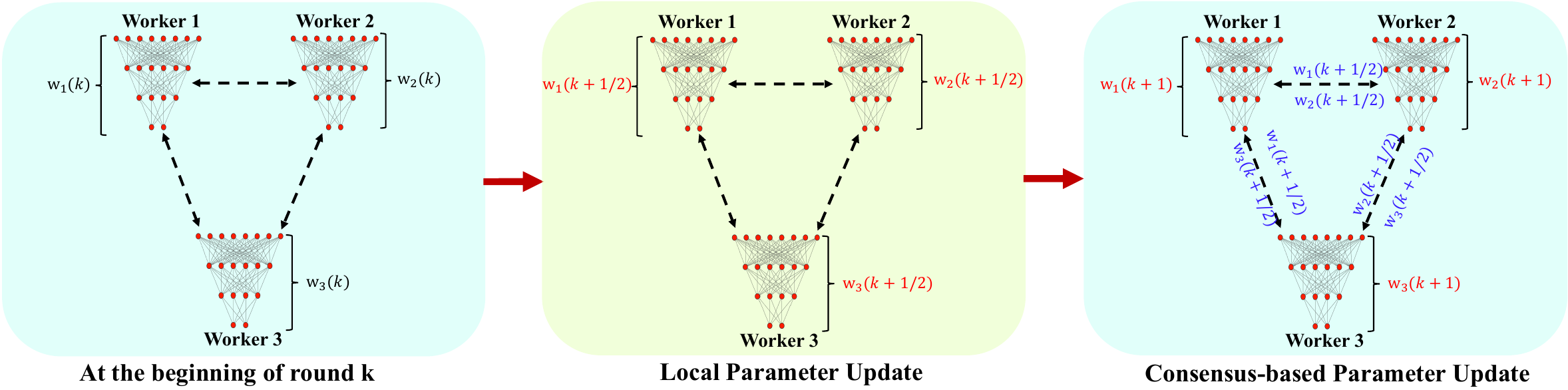}
\vspace{-0.1in}
\caption{An illustrative example of conventional decentralized learning framework, e.g., \cite{lian2017can} with a single shared model with 3 workers.
 }
        \label{fig:conventional-DL}
\end{figure}

\begin{figure}[t]
\centering
   \includegraphics[width=1\textwidth]{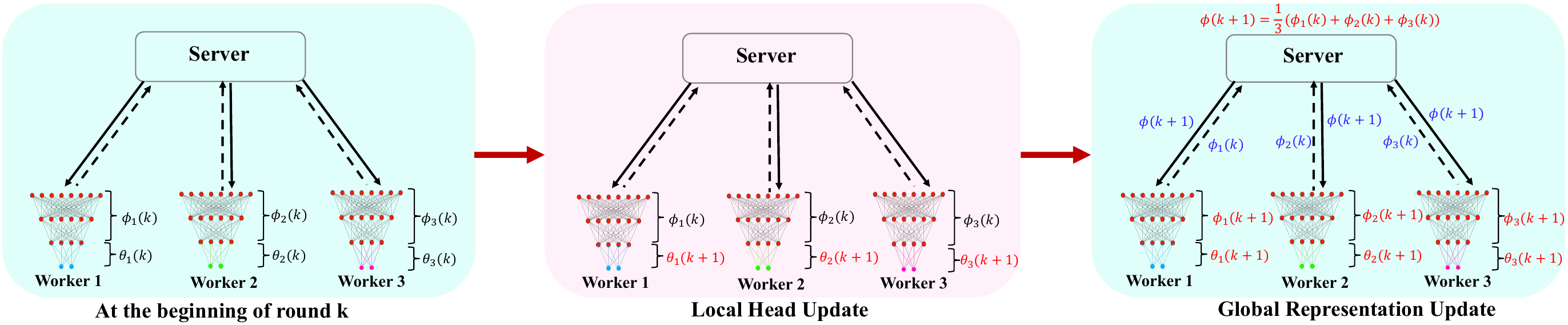}
\vspace{-0.1in}
\caption{An illustrative example of PS-based framework with shared representations in \cite{collins2021exploiting} with 3 workers.
 }
        \label{fig:FL-PRL}
\end{figure}

Furthermore, as discussed in Remark~\ref{remark-alg}, our model is inspired by \citet{collins2021exploiting} that is perhaps the first to exploit PS-based framework with a shared presentation.  An illustrative example of their algorithm with 3 workers is presented in Figure~\ref{fig:FL-PRL}.

\section{Proof of Theorem \ref{thm:loss_convergence}}
Under Assumption \ref{assumption-weight}, we present the following lemma from \citet{nedic2009distributed} for completeness. 
\begin{lemma}[Lemma 4 in \citet{nedic2009distributed}]
Assume that $\bP$ is doubly stochastic.  The difference between $1/N$ and any element of   $\bP^{k-s}$ can be bounded by
\begin{align}
\left|\frac{1}{N}-\bP^{k-s}(i,j)\right|\leq 2 \frac{(1+p^{-N})}{1-p^{N}}(1-p^{N})^{(k-s)/N},
\end{align}
where  $p$ is the smallest positive value of all consensus matrices, i.e., $p=\arg\min P_{i,j}$ with $P_{i,j}>0, \forall i,j.$
\label{lemma_bound_Phi}
\end{lemma}

{
For Lipschitz continuous gradient functions with two variables in Assumption~\ref{assumption-lipschitz}, we have: 
\begin{lemma}\label{lemma:2D-Lipschitz}
Under Assumption \ref{assumption-lipschitz}, we have 
\begin{align*}
    F_i(\pmb{\phi}, \pmb{\theta})-F_i(\pmb{\phi}^\prime, \pmb{\theta}^\prime)
   &\leq \nabla_{\pmb{\phi}}F_i(\pmb{\phi}^\prime,\pmb{\theta})^\intercal (\pmb{\phi}-\pmb{\phi}^\prime)+\frac{L}{2}\|\pmb{\phi}-\pmb{\phi}^\prime\|^2+\nabla_{\pmb{\theta}}F_i(\pmb{\phi}^\prime,\pmb{\theta}^\prime)^\intercal (\pmb{\theta}-\pmb{\theta}^\prime)+\frac{L}{2}\|\pmb{\theta}-\pmb{\theta}^\prime\|^2. 
\end{align*}
\end{lemma}

\begin{proof}
We rewrite $F_i(\pmb{\phi}, \pmb{\theta})-F_i(\pmb{\phi}^\prime, \pmb{\theta}^\prime)$ as follows:
\begin{align*}
   F_i(\pmb{\phi}, \pmb{\theta})-F_i(\pmb{\phi}^\prime, \pmb{\theta}^\prime)
    &= F_i(\pmb{\phi}, \pmb{\theta})- F_i(\pmb{\phi}^\prime, \pmb{\theta})+F_i(\pmb{\phi}^\prime, \pmb{\theta})-F_i(\pmb{\phi}^\prime,\pmb{\theta}^\prime)\\
    &\leq \nabla_{\pmb{\phi}}F_i(\pmb{\phi}^\prime,\pmb{\theta})^\intercal (\pmb{\phi}-\pmb{\phi}^\prime)+\frac{L}{2}\|\pmb{\phi}-\pmb{\phi}^\prime\|^2+\nabla_{\pmb{\theta}}F_i(\pmb{\phi}^\prime,\pmb{\theta}^\prime)^\intercal (\pmb{\theta}-\pmb{\theta}^\prime)+\frac{L}{2}\|\pmb{\theta}-\pmb{\theta}^\prime\|^2,
\end{align*}
where the inequality is due to the fact \cite{bottou2018optimization} that  $\|\nabla_{\pmb{\phi}} F_j(\pmb{\phi},\pmb{\theta})-\nabla_{\pmb{\phi}} F_j(\pmb{\phi}^\prime,\pmb{\theta})\|\leq L\|\pmb{\phi}-\pmb{\phi}^\prime\|$ and
$\|\nabla_{\pmb{\theta}} F_j(\pmb{\phi},\pmb{\theta})-\nabla_{\pmb{\theta}} F_j(\pmb{\phi},\pmb{\theta}^\prime)\|\leq L\|\pmb{\theta}-\pmb{\theta}^\prime\|$.
\end{proof}
}

To prove Theorem \ref{thm:loss_convergence}, we start with evaluating the drift of the global loss function, i.e., $f(\bar{\pmb{\phi}}(k+1),\{\pmb{\theta}_i(k+1)\}_{i=1}^N)-f(\bar{\pmb{\phi}}(k),\{\pmb{\theta}_i(k)\}_{i=1}^N)$. 
\begin{lemma}
The drift of the global loss function satisfies
\begin{align}\label{eq:global_iter2}
  & \mathbb{E}[f(\bar{\pmb{\phi}}(k+1),\{\pmb{\theta}_i(k+1)\}_{i=1}^N)]- \mathbb{E}[f(\bar{\pmb{\phi}}(k),\{\pmb{\theta}_i(k)\}_{i=1}^N)]\nonumber\displaybreak[0]\\
   {\leq} &\underset{C_1}{\underbrace{\frac{1}{N}\sum_{i=1}^N\mathbb{E}\Big\langle \nabla_{\pmb{\phi}}F_i(\bar{\pmb{\phi}}(k),\pmb{\theta}_i(k+1)), \bar{\pmb{\phi}}(k+1)-\bar{\pmb{\phi}}(k)\Big\rangle}}+\underset{C_2}{\underbrace{\frac{1}{N}\sum_{i=1}^N\frac{L}{2}\mathbb{E}[\|\bar{\pmb{\phi}}(k+1)-\bar{\pmb{\phi}}(k)\|^2]}}\nonumber\displaybreak[1]\\
   +& \underset{C_3}{\underbrace{\frac{1}{N}\sum_{i=1}^N\mathbb{E}\Big\langle \nabla_{\pmb{\theta}}F_i(\bar{\pmb{\phi}}(k),\pmb{\theta}_i(k)), \ \pmb{\theta}_i(k+1) -\ \pmb{\theta}_i(k) \Big\rangle}}+\underset{C_4}{\underbrace{\frac{1}{N}\sum_{i=1}^N\frac{L}{2}\mathbb{E}[\|{\pmb{\theta}_i}(k+1)-{\pmb{\theta}_i}(k)\|^2]}}.  
\end{align}
\end{lemma}
\begin{proof}
According to the definition of the global loss function $f$, we have 
\begin{align}\label{eq:global_iteration}\nonumber
    &\mathbb{E}[f(\bar{\pmb{\phi}}(k+1),\{\pmb{\theta}_i(k+1)\}_{i=1}^N)]\\
    =&\frac{1}{N}\sum_{i=1}^N \mathbb{E}[F_i(\bar{\pmb{\phi}}(k+1),\pmb{\theta}_i(k+1))]\nonumber \allowdisplaybreaks\\
    \leq& \frac{1}{N}\sum_{i=1}^N \mathbb{E}[\nabla_{\pmb{\phi}}F_i(\bar{\pmb{\phi}}(k),\pmb{\theta}_i(k+1))^\intercal (\bar{\pmb{\phi}}(k+1)-\bar{\pmb{\phi}}(k))]+\frac{1}{N}\sum_{i=1}^N\frac{L}{2}\mathbb{E}[\|\bar{\pmb{\phi}}(k+1)-\bar{\pmb{\phi}}(k)\|^2]\nonumber\allowdisplaybreaks\\
    &+\frac{1}{N}\sum_{i=1}^N \mathbb{E}[\nabla_{\pmb{\theta}}F_i(\bar{\pmb{\phi}}(k),\pmb{\theta}_i(k))^\intercal ({\pmb{\theta}}_i(k+1)-{\pmb{\theta}_i}(k))]+\frac{1}{N}\sum_{i=1}^N\frac{L}{2}\mathbb{E}[\|{\pmb{\theta}}_i(k+1)-{\pmb{\theta}}_i(k)\|^2]\nonumber\allowdisplaybreaks\\
    &+\frac{1}{N}\sum_{i=1}^N \mathbb{E}[F_i(\bar{\pmb{\phi}}(k),\pmb{\theta}_i(k))],
\end{align}
where the inequality follows from Lemma \ref{lemma:2D-Lipschitz}. The desired result holds by moving $\frac{1}{N}\sum_{i=1}^N \mathbb{E}[F_i(\bar{\pmb{\phi}}(k),\pmb{\theta}_i(k))]$ into the left-hand side of the inequality and replacing it by $\mathbb{E}[f(\bar{\pmb{\phi}}(k),\{\pmb{\theta}_i(k)\}_{i=1}^N)]$. 
\end{proof}

In the following, we bound  $C_1, C_2, C_3,$ and $C_4$, respectively. 
\begin{lemma}\label{lem:C1}
$C_1$ can be bounded as follows:
\begin{align}
    C_1&\leq {\frac{\beta L^2}{2N}\sum_{i=1}^N\mathbb{E}\left[\left\|\bar{\pmb{\phi}}(k)-\pmb{\phi}_i(k)\right\|^2\right]}-\frac{\beta}{2}\left\|\nabla_{\pmb{\phi}}f(\bar{\pmb{\phi}}(k),\{\pmb{\theta}_i(k+1)\}_{i=1}^N)\right\|^2\nonumber\allowdisplaybreaks\\
    &\qquad\qquad-\frac{\beta}{2N^2}\mathbb{E}\left[\left\|\sum_{i=1}^N g_{\pmb{\phi}}(\pmb{\phi}_i(k),\pmb{\theta}_i(k+1))\right\|^2\right].
\end{align}
\end{lemma}
\begin{proof}
$C_1$ can be bounded as follows
\begin{align}\nonumber
C_1&=\mathbb{E}\Bigg\langle \frac{1}{N}\sum_{i=1}^N\nabla_{\pmb{\phi}}F_i(\bar{\pmb{\phi}}(k),\pmb{\theta}_i(k+1)), \bar{\pmb{\phi}}(k+1)-\bar{\pmb{\phi}}(k)\Bigg\rangle\nonumber\allowdisplaybreaks\\
    &\overset{(a_1)}{=}-\beta \mathbb{E}\Big\langle \nabla_{\pmb{\phi}}f(\bar{\pmb{\phi}}(k),\{\pmb{\theta}_i(k+1)\}_{i=1}^N), \frac{1}{N}\sum_{i=1}^N g_{\pmb{\phi}}(\pmb{\phi}_i(k),\pmb{\theta}_i(k+1))\Big\rangle\nonumber\allowdisplaybreaks\\
    &\overset{(a_2)}{=}\frac{-\beta}{2}\Bigg(\left\|\nabla_{\pmb{\phi}}f(\bar{\pmb{\phi}}(k),\{\pmb{\theta}_i(k+1)\}_{i=1}^N)\right\|^2+\frac{1}{N^2}\mathbb{E}\left[\left\|\sum_{i=1}^N g_{\pmb{\phi}}(\pmb{\phi}_i(k),\pmb{\theta}_i(k+1))\right\|^2\right]\nonumber\allowdisplaybreaks\\
    &\quad\quad\quad-\frac{1}{N^2}\mathbb{E}\left[\left\|\sum_{i=1}^N(\nabla_{\pmb{\phi}}F_i(\bar{\pmb{\phi}}(k),\pmb{\theta}_i(k+1))-\nabla_{\pmb{\phi}}F_i({\pmb{\phi}_i}(k),\pmb{\theta}_i(k+1)))\right\|^2\right]\Bigg)\nonumber\allowdisplaybreaks \\
    &{=}\frac{\beta}{2N^2}\mathbb{E}\left[\left\|\sum_{i=1}^N(\nabla_{\pmb{\phi}}F_i(\bar{\pmb{\phi}}(k),\pmb{\theta}_i(k+1))-\nabla_{\pmb{\phi}}F_i({\pmb{\phi}_i}(k),\pmb{\theta}_i(k+1)))\right\|^2\right]\nonumber\allowdisplaybreaks\\
    &\quad\quad\quad-\frac{\beta}{2}\left\|\nabla_{\pmb{\phi}}f(\bar{\pmb{\phi}}(k),\{\pmb{\theta}_i(k+1)\}_{i=1}^N)\right\|^2-\frac{\beta}{2N^2}\mathbb{E}\left[\left\|\sum_{i=1}^N g_{\pmb{\phi}}(\pmb{\phi}_i(k),\pmb{\theta}_i(k+1))\right\|^2\right]\nonumber\allowdisplaybreaks\\
    &\overset{(a_3)}{\leq}\underset{D_1}{\underbrace{\frac{\beta}{2N}\sum_{i=1}^N\mathbb{E}\left[\left\|(\nabla_{\pmb{\phi}}F_i(\bar{\pmb{\phi}}(k),\pmb{\theta}_i(k+1))-\nabla_{\pmb{\phi}}F_i({\pmb{\phi}_i}(k),\pmb{\theta}_i(k+1)))\right\|^2\right]}}\nonumber\allowdisplaybreaks\\
    &\quad\quad\quad-\frac{\beta}{2}\left\|\nabla_{\pmb{\phi}}f(\bar{\pmb{\phi}}(k),\{\pmb{\theta}_i(k+1)\}_{i=1}^N)\right\|^2-\frac{\beta}{2N^2}\mathbb{E}\left[\left\|\sum_{i=1}^N g_{\pmb{\phi}}(\pmb{\phi}_i(k),\pmb{\theta}_i(k+1))\right\|^2\right]\nonumber\allowdisplaybreaks\\
    &\overset{(a_4)}{\leq}{\frac{\beta L^2}{2N}\sum_{i=1}^N\mathbb{E}\left[\left\|\bar{\pmb{\phi}}(k)-\pmb{\phi}_i(k)\right\|^2\right]}-\frac{\beta}{2}\left\|\nabla_{\pmb{\phi}}f(\bar{\pmb{\phi}}(k),\{\pmb{\theta}_i(k+1)\}_{i=1}^N)\right\|^2\nonumber\allowdisplaybreaks\\
    &\quad\quad\quad-\frac{\beta}{2N^2}\mathbb{E}\left[\left\|\sum_{i=1}^N g_{\pmb{\phi}}(\pmb{\phi}_i(k),\pmb{\theta}_i(k+1))\right\|^2\right],
\end{align}
where $(a_1)$ holds due to $\bar{\pmb{\phi}}(k+1)-\bar{\pmb{\phi}}(k)= \frac{-\beta}{N}\sum_{i=1}^N g_{\pmb{\phi}}(\pmb{\phi}_i(k),\pmb{\theta}_i(k+1))$; $(a_2)$ holds according to the equation $\langle \ba,\bb\rangle = \frac{1}{2}[\|\ba\|^2+\|\bb\|^2-\|\ba-\bb\|^2]$ and the fact that $\mathbb{E}\left[g_{\pmb{\phi}}(\pmb{\phi}_i(k),\pmb{\theta}_i(k+1))\right]=\nabla_{\pmb{\phi}}F_i(\pmb{\phi}_i(k),\pmb{\theta}_i(k+1)), \forall i$ from Assumption \ref{assumption-gradient};  $(a_3)$ follows the Cauchy-Schwartz inequality $\|\sum_{i=1}^N \ba_i\|^2\leq N\sum_{i=1}^N\|\ba_i\|^2$; and $(a_4)$ directly comes from Assumption \ref{assumption-lipschitz} i.e., $D_1\leq{\frac{\beta L^2}{2N}\sum_{i=1}^N\mathbb{E}\left\|\bar{\pmb{\phi}}(k)-\pmb{\phi}_i(k)\right\|^2}$.
\end{proof}

\begin{lemma}\label{lem:C2}
 $C_2$ is bounded as follows
 \begin{align}
     C_2\leq \frac{\beta^2L}{2N}\sigma^2+\frac{\beta^2L}{2N^2} \left\|\mathbb{E}\sum_{i=1}^N g_{\pmb{\phi}}(\pmb{\phi}_i(k), \pmb{\theta}_i(k+1))\right\|^2.
 \end{align}
\end{lemma}
\begin{proof}
The proof follows from the fact that $\bar{\pmb{\phi}}(k+1)-\bar{\pmb{\phi}}(k)= \frac{-\beta}{N}\sum_{i=1}^N g_{\pmb{\phi}}(\pmb{\phi}_i(k),\pmb{\theta}_i(k+1))$, i.e.,
\begin{align}
   C_2&=\frac{L}{2}\mathbb{E}[\|\bar{\pmb{\phi}}(k+1)-\bar{\pmb{\phi}}(k)\|^2]\nonumber\displaybreak[0]\\
   &=\frac{\beta^2L}{2N^2}\mathbb{E}\left[\left\|\sum_{i=1}^N g_{\pmb{\phi}}(\pmb{\phi}_i(k), \pmb{\theta}_i(k+1))\right\|^2\right]\nonumber\\
   &\overset{(b_1)}{=}\frac{\beta^2L}{2N^2}\mathbb{E}\left[\left\|\sum_{i=1}^N g_{\pmb{\phi}}(\pmb{\phi}_i(k), \pmb{\theta}_i(k+1))-\nabla_{\pmb{\phi}}F_i(\pmb{\phi}_i(k),\pmb{\theta}_i(k+1))\right\|^2\right]+\frac{\beta^2L}{2N^2} \left\|\mathbb{E}\sum_{i=1}^N g_{\pmb{\phi}}(\pmb{\phi}_i(k), \pmb{\theta}_i(k+1))\right\|^2 \nonumber\\
   &\overset{(b_2)}{\leq}\frac{\beta^2L}{2N}\sigma^2+\frac{\beta^2L}{2N^2} \left\|\mathbb{E}\sum_{i=1}^N g_{\pmb{\phi}}(\pmb{\phi}_i(k), \pmb{\theta}_i(k+1))\right\|^2, 
\end{align}
where $(b_1)$ is due to $\mathbb{E}[\|\bX\|^2]=\mathrm{Var}[\bX]+\|\mathbb{E}[\bX]\|^2$; and $(b_2)$ follows from the bounded variance in Assumption \ref{assumption-variance}.
\end{proof}

\begin{lemma}\label{lem:C3}
 $C_3$ can be bounded as
\begin{align}
     C_3{\leq} &\frac{-\alpha\tau}{2} \mathbb{E}\left[\left\|\nabla_{\pmb{\theta}}f(\bar{\pmb{\phi}}(k),\{\pmb{\theta}_i(k))\}_{i=1}^N\right\|^2\right]-\frac{\alpha}{2N\tau}\sum_{i=1}^N\mathbb{E}\left[\left\| \sum_{s=0}^{\tau-1} g_{\pmb{\theta}}(\pmb{\phi}_i(k),\pmb{\theta}_i(k,s))\right\|^2\right]\nonumber\allowdisplaybreaks\\
    &\qquad\qquad+\frac{\alpha\tau L^2}{N}\sum_{i=1}^N\mathbb{E}\left[\left\|\pmb{\phi}_i(k)- \bar{\pmb{\phi}}(k)\right\|^2\right]+\alpha^3L^2(18\tau^3-15\tau^2-3\tau)\sigma^2\nonumber\\
    &\qquad\qquad+\frac{\alpha^3L^2(18\tau^3-18\tau^2)}{N}\sum_{i=1}^N\mathbb{E}\left[\left\|g_{\pmb{\theta}}(\pmb{\phi}_i(k), \pmb{\theta}_i(k))\right\|^2\right].
\end{align}
\end{lemma}

\begin{proof}
We first rewrite $C_3$ as
\begin{align*}
    C_3&=\frac{1}{N}\sum_{i=1}^N\mathbb{E}\Big\langle \nabla_{\pmb{\theta}}F_i(\bar{\pmb{\phi}}(k),\pmb{\theta}_i(k)),  -\alpha \sum_{s=0}^{\tau-1} g_{\pmb{\theta}}(\pmb{\phi}_i(k),\pmb{\theta}_i(k,s))\Big\rangle\allowdisplaybreaks\\
    &=\frac{1}{N}\sum_{i=1}^N\mathbb{E}\Big\langle \nabla_{\pmb{\theta}}F_i(\bar{\pmb{\phi}}(k),\pmb{\theta}_i(k)),  -\alpha \sum_{s=0}^{\tau-1} g_{\pmb{\theta}}(\pmb{\phi}_i(k),\pmb{\theta}_i(k,s))+\alpha\tau \nabla_{\pmb{\theta}}F_i(\bar{\pmb{\phi}}(k),\pmb{\theta}_i(k))-\alpha\tau \nabla_{\pmb{\theta}}F_i(\bar{\pmb{\phi}}(k),\pmb{\theta}_i(k))\Big\rangle.
\end{align*}
Then, we have
\begin{align}\label{eq:C3}
    C_3&=\frac{-\alpha\tau}{N}\sum_{i=1}^N \mathbb{E}\left[\left\|\nabla_{\pmb{\theta}}F_i(\bar{\pmb{\phi}}(k),\pmb{\theta}_i(k))\right\|^2\right]\nonumber\allowdisplaybreaks\\
    &\qquad\qquad+\frac{1}{N}\sum_{i=1}^N\mathbb{E}\Big\langle \nabla_{\pmb{\theta}}F_i(\bar{\pmb{\phi}}(k),\pmb{\theta}_i(k)),  -\alpha \sum_{s=0}^{\tau-1} g_{\pmb{\theta}}(\pmb{\phi}_i(k),\pmb{\theta}_i(k,s))+\alpha\tau \nabla_{\pmb{\theta}}F_i(\bar{\pmb{\phi}}(k),\pmb{\theta}_i(k))\Big\rangle\nonumber\allowdisplaybreaks\\
    &=\frac{-\alpha\tau}{N}\sum_{i=1}^N \mathbb{E}\left[\left\|\nabla_{\pmb{\theta}}F_i(\bar{\pmb{\phi}}(k),\pmb{\theta}_i(k))\right\|^2\right]+\frac{1}{N}\sum_{i=1}^N\mathbb{E}\Big\langle \sqrt{\alpha\tau}\nabla_{\pmb{\theta}}F_i(\bar{\pmb{\phi}}(k),\pmb{\theta}_i(k)), \nonumber\allowdisplaybreaks\\
    &\qquad\qquad\qquad\qquad\qquad\qquad-\sqrt{\frac{\alpha}{\tau}} \sum_{s=0}^{\tau-1} (g_{\pmb{\theta}}(\pmb{\phi}_i(k),\pmb{\theta}_i(k,s))- \nabla_{\pmb{\theta}}F_i(\bar{\pmb{\phi}}(k),\pmb{\theta}_i(k)))\Big\rangle\nonumber\allowdisplaybreaks\\
    &\overset{(c_1)}{=}\frac{-\alpha\tau}{N}\sum_{i=1}^N \mathbb{E}\left[\left\|\nabla_{\pmb{\theta}}F_i(\bar{\pmb{\phi}}(k),\pmb{\theta}_i(k))\right\|^2\right]+\frac{1}{N}\sum_{i=1}^N\Bigg(\frac{\alpha\tau}{2}\mathbb{E}\left[\left\|\nabla_{\pmb{\theta}}F_i(\bar{\pmb{\phi}}(k),\pmb{\theta}_i(k))\right\|^2\right]\nonumber\allowdisplaybreaks\\
    &-\frac{\alpha}{2\tau}\mathbb{E}\left[\left\| \sum_{s=0}^{\tau-1} g_{\pmb{\theta}}(\pmb{\phi}_i(k),\pmb{\theta}_i(k,s))\right\|^2\right]+\frac{\alpha}{2\tau}\mathbb{E}\left[\left\|\sum_{s=0}^{\tau-1} (g_{\pmb{\theta}}(\pmb{\phi}_i(k),\pmb{\theta}_i(k,s))- \nabla_{\pmb{\theta}}F_i(\bar{\pmb{\phi}}(k),\pmb{\theta}_i(k)))\right\|^2\right]\Bigg)\nonumber\allowdisplaybreaks \\
    &=\frac{-\alpha\tau}{2N}\sum_{i=1}^N \mathbb{E}\left[\left\|\nabla_{\pmb{\theta}}F_i(\bar{\pmb{\phi}}(k),\pmb{\theta}_i(k))\right\|^2\right]-\frac{\alpha}{2N\tau}\sum_{i=1}^N\mathbb{E}\left[\left\| \sum_{s=0}^{\tau-1} g_{\pmb{\theta}}(\pmb{\phi}_i(k),\pmb{\theta}_i(k,s))\right\|^2\right]\nonumber\allowdisplaybreaks\\
    &\qquad\qquad\qquad\qquad+\underset{D_2}{\underbrace{\frac{\alpha}{2N\tau}\sum_{i=1}^N\mathbb{E}
    \left[\left\|\sum_{s=0}^{\tau-1} (g_{\pmb{\theta}}(\pmb{\phi}_i(k),\pmb{\theta}_i(k,s))- \nabla_{\pmb{\theta}}F_i(\bar{\pmb{\phi}}(k),\pmb{\theta}_i(k)))\right\|^2\right]}}\nonumber\allowdisplaybreaks \\
    &\overset{(c_2)}{\leq} \frac{-\alpha\tau}{2} \mathbb{E}\left[\left\|\nabla_{\pmb{\theta}}f(\bar{\pmb{\phi}}(k),\{\pmb{\theta}_i(k))\}_{i=1}^N\right\|^2\right]-\frac{\alpha}{2N\tau}\sum_{i=1}^N\mathbb{E}\left[\left\| \sum_{s=0}^{\tau-1} g_{\pmb{\theta}}(\pmb{\phi}_i(k),\pmb{\theta}_i(k,s))\right\|^2\right]\nonumber\allowdisplaybreaks\\
    &\qquad\qquad\qquad\qquad+\underset{D_2}{\underbrace{\frac{\alpha}{2N\tau}\sum_{i=1}^N\mathbb{E}\left[\left\|\sum_{s=0}^{\tau-1} (g_{\pmb{\theta}}(\pmb{\phi}_i(k),\pmb{\theta}_i(k,s))- \nabla_{\pmb{\theta}}F_i(\bar{\pmb{\phi}}(k),\pmb{\theta}_i(k)))\right\|^2\right]}},
\end{align}
where  $(c_1)$ holds according to the equation $\langle \ba,\bb\rangle = \frac{1}{2}[\|\ba\|^2+\|\bb\|^2-\|\ba-\bb\|^2]$; and   $(c_2)$ follows the Cauchy-Schwartz inequality $\|\sum_{i=1}^N \ba_i\|^2\leq N\sum_{i=1}^N\|\ba_i\|^2$.
To this end, the key to bound $C_3$ is to bound $D_2$, which is given by
\begin{align}\label{eq:D2}
    D_2&=\frac{\alpha}{2N\tau}\sum_{i=1}^N\mathbb{E}\left[\left\|\sum_{s=0}^{\tau-1} g_{\pmb{\theta}}(\pmb{\phi}_i(k),\pmb{\theta}_i(k,s))- \sum_{s=0}^{\tau-1}\nabla_{\pmb{\theta}}F_i(\bar{\pmb{\phi}}(k),\pmb{\theta}_i(k))\right\|^2\right]\nonumber\allowdisplaybreaks\\
    &=\frac{\alpha}{2N\tau}\sum_{i=1}^N\mathbb{E}\Bigg[\Bigg\|\sum_{s=0}^{\tau-1} g_{\pmb{\theta}}(\pmb{\phi}_i(k),\pmb{\theta}_i(k,s))-\nabla_{\pmb{\theta}}F_i(\bar{\pmb{\phi}}(k),\pmb{\theta}_i(k,s))\nonumber\allowdisplaybreaks\\
&\qquad\qquad\qquad\qquad+\nabla_{\pmb{\theta}}F_i(\bar{\pmb{\phi}}(k),\pmb{\theta}_i(k,s))-\sum_{s=0}^{\tau-1}\nabla_{\pmb{\theta}}F_i(\bar{\pmb{\phi}}(k),\pmb{\theta}_i(k))\Bigg\|^2\Bigg]\nonumber\allowdisplaybreaks\\
    &\leq\frac{\alpha}{N\tau}\sum_{i=1}^N\mathbb{E}\left[\left\|\sum_{s=0}^{\tau-1} \nabla_{\pmb{\theta}}F_i(\pmb{\phi}_i(k),\pmb{\theta}_i(k,s))- \nabla_{\pmb{\theta}}F_i(\bar{\pmb{\phi}}(k),\pmb{\theta}_i(k,s))\right\|^2\right]\nonumber\allowdisplaybreaks\\
    &\qquad\qquad\qquad\qquad+\frac{\alpha}{N\tau}\sum_{i=1}^N\mathbb{E}\left[\left\|\sum_{s=0}^{\tau-1}\nabla_{\pmb{\theta}}F_i(\bar{\pmb{\phi}}(k),\pmb{\theta}_i(k,s))-\sum_{s=0}^{\tau-1}\nabla_{\pmb{\theta}}F_i(\bar{\pmb{\phi}}(k),\pmb{\theta}_i(k))\right\|^2\right]\nonumber\allowdisplaybreaks\\
    &\leq \frac{\alpha\tau L^2}{N}\sum_{i=1}^N\mathbb{E}\left[\left\|\pmb{\phi}_i(k)- \bar{\pmb{\phi}}(k)\right\|^2\right]+\frac{\alpha L^2}{N}\sum_{i=1}^N\sum_{s=0}^{\tau-1}\mathbb{E}\left[\left\|\pmb{\theta}_i(k,s)-\pmb{\theta}_i(k)\right\|^2\right].
\end{align}
Next, we bound $\mathbb{E}\left[\left\|\pmb{\theta}_i(k,s)-\pmb{\theta}_i(k)\right\|^2\right]$ as follows.
\begin{align}\label{eq:D2_condition}
    &\mathbb{E}\left[\|\pmb{\theta}_i(k,s)-\pmb{\theta}_{i}(k)\|^2\right]\nonumber\allowdisplaybreaks\\
    &= \mathbb{E}\left[\|\pmb{\theta}_i(k,s-1)-\pmb{\theta}_{i}(k)-\alpha g_{\pmb{\theta}}(\pmb{\phi}_i(k), \pmb{\theta}_i(k,s-1))\|^2\right]\nonumber\allowdisplaybreaks\\
    &=\mathbb{E}\Big[\|\pmb{\theta}_i(k,s-1)-\pmb{\theta}_{i}(k)-\alpha g_{\pmb{\theta}}(\pmb{\phi}_i(k), \pmb{\theta}_i(k,s-1))+\alpha\nabla_{\pmb{\theta}}F_i(\pmb{\phi}_i(k), \pmb{\theta}_i(k,s-1))-\alpha\nabla_{\pmb{\theta}}F_i(\pmb{\phi}_i(k), \pmb{\theta}_i(k,s-1))\nonumber\allowdisplaybreaks\\
    &\qquad\qquad+\alpha\nabla_{\pmb{\theta}}F_i(\pmb{\phi}_i(k), \pmb{\theta}_i(k)-\alpha\nabla_{\pmb{\theta}}F_i(\pmb{\phi}_i(k), \pmb{\theta}_i(k)))+\alpha g_{\pmb{\theta}}(\pmb{\phi}_i(k), \pmb{\theta}_i(k))-\alpha g_{\pmb{\theta}}(\pmb{\phi}_i(k), \pmb{\theta}_i(k))\|^2\Big]\nonumber\allowdisplaybreaks\\
    &\leq \left(1+\frac{1}{2\tau-1}\right)\mathbb{E}\left[\|\pmb{\theta}_i(k,s-1)-\pmb{\theta}_{i}(k)\|^2\right]+\alpha^2\mathbb{E}\left[\|g_{\pmb{\theta}}(\pmb{\phi}_i(k), \pmb{\theta}_i(k,s-1))-\nabla_{\pmb{\theta}}F_i(\pmb{\phi}_i(k), \pmb{\theta}_i(k,s-1))\|^2\right]\nonumber\allowdisplaybreaks\\
    &\qquad\qquad +6\alpha^2\tau \mathbb{E}\left[\|\nabla_{\pmb{\theta}}F_i(\pmb{\phi}_i(k), \pmb{\theta}_i(k,s-1)-\nabla_{\pmb{\theta}}F_i(\pmb{\phi}_i(k), \pmb{\theta}_i(k)))\|^2\right]\nonumber\allowdisplaybreaks\\
    &\qquad\qquad +6\alpha^2\tau \mathbb{E}\left[\|\nabla_{\pmb{\theta}}F_i(\pmb{\phi}_i(k), \pmb{\theta}_i(k)- g_{\pmb{\theta}}(\pmb{\phi}_i(k), \pmb{\theta}_i(k)))\|^2\right]+6\alpha^2\tau\mathbb{E}\left[\|g_{\pmb{\theta}}(\pmb{\phi}_i(k), \pmb{\theta}_i(k))\|^2\right]\nonumber\allowdisplaybreaks\\
    &\leq \left(1+\frac{1}{2\tau-1}\right)\mathbb{E}\left[\|\pmb{\theta}_i(k,s-1)-\pmb{\theta}_{i}(k)\|^2\right]+\alpha^2\sigma^2+6\alpha^2\tau L^2\mathbb{E}\left[\|\pmb{\theta}_i(k,s-1)-\pmb{\theta}_{i}(k)\|^2\right]\nonumber\allowdisplaybreaks\\
    &\qquad\qquad+6\alpha^2\tau \sigma^2+6\alpha^2\tau\mathbb{E}\left[\|g_{\pmb{\theta}}(\pmb{\phi}_i(k), \pmb{\theta}_i(k))\|^2\right]\nonumber\allowdisplaybreaks\\
    &\leq \left(1+\frac{1}{2\tau-1}+6\alpha^2\tau L^2\right)\mathbb{E}\left[\|\pmb{\theta}_i(k,s-1)-\pmb{\theta}_{i}(k)\|^2\right]+(\alpha^2+6\alpha^2\tau)\sigma^2+6\alpha^2\tau\mathbb{E}\|g_{\pmb{\theta}}(\pmb{\phi}_i(k), \pmb{\theta}_i(k))\|^2 \nonumber\allowdisplaybreaks\\
    &\leq (18\tau^2-15\tau-3)\alpha^2\sigma^2+(18\tau^2-18\tau)\alpha^2\mathbb{E}\left[\|g_{\pmb{\theta}}(\pmb{\phi}_i(k), \pmb{\theta}_i(k))\|^2\right],
\end{align}
where the last inequality holds by using Lemma 3 in \cite{reddi2020adaptive}.
Substituting \eqref{eq:D2_condition} into \eqref{eq:D2}, we have
\begin{align}\label{eq:D2_2}
    D_2&\leq \frac{\alpha\tau L^2}{N}\sum_{i=1}^N\mathbb{E}\left[\left\|\pmb{\phi}_i(k)- \bar{\pmb{\phi}}(k)\right\|^2\right]+\alpha^3L^2(18\tau^3-15\tau^2-3\tau)\sigma^2\nonumber\allowdisplaybreaks\\
    &\qquad\qquad+\frac{\alpha^3L^2(18\tau^3-18\tau^2)}{N}\sum_{i=1}^N\mathbb{E}\left[\left\|g_{\pmb{\theta}}(\pmb{\phi}_i(k), \pmb{\theta}_i(k))\right\|^2\right].
\end{align}
Hence, substituting $D_2$ in \eqref{eq:D2_2} back to \eqref{eq:C3}, we have the desired result as follows
\begin{align*}
    C_3{\leq} &\frac{-\alpha\tau}{2} \mathbb{E}\left[\left\|\nabla_{\pmb{\theta}}f(\bar{\pmb{\phi}}(k),\{\pmb{\theta}_i(k))\}_{i=1}^N\right\|^2\right]-\frac{\alpha}{2N\tau}\sum_{i=1}^N\mathbb{E}\left[\left\| \sum_{s=0}^{\tau-1} g_{\pmb{\theta}}(\pmb{\phi}_i(k),\pmb{\theta}_i(k,s))\right\|^2\right]\nonumber\allowdisplaybreaks\\
    &\qquad\qquad+\frac{\alpha\tau L^2}{N}\sum_{i=1}^N\mathbb{E}\left[\left\|\pmb{\phi}_i(k)- \bar{\pmb{\phi}}(k)\right\|^2\right]+\alpha^3L^2(18\tau^3-15\tau^2-3\tau)\sigma^2\nonumber\\
    &\qquad\qquad+\frac{\alpha^3L^2(18\tau^3-18\tau^2)}{N}\sum_{i=1}^N\mathbb{E}\left[\left\|g_{\pmb{\theta}}(\pmb{\phi}_i(k), \pmb{\theta}_i(k))\right\|^2\right].
\end{align*}
\end{proof}

\begin{lemma}\label{lem:C4}
$C_4$ can be bounded as
\begin{align}
     C_4{\leq}\frac{\alpha^2\tau L}{2}\sigma^2+\frac{\alpha^2L}{2N}\sum_{i=1}^N\left\|\mathbb{E}\sum_{s=0}^{\tau-1}g_{\pmb{\theta}}(\pmb{\phi}_i(k),\pmb{\theta}_i(k,s))\right\|^2. 
\end{align}
\end{lemma}
\begin{proof}
The proof directly follows from the definition and the proof for $C_2$ in Lemma \ref{lem:C2}, i.e.,
\begin{align}
   C_4&=\frac{L}{2N}\sum_{i=1}^N\mathbb{E}[\|{\pmb{\theta}}_i(k+1)-{\pmb{\theta}}_i(k)\|^2]\nonumber\displaybreak[0]\\
   &=\frac{\alpha^2L}{2N}\sum_{i=1}^N\mathbb{E}\left[\left\| \sum_{s=0}^{\tau-1} g_{\pmb{\theta}}(\pmb{\phi}_i(k),\pmb{\theta}_i(k,s))\right\|^2\right]\nonumber\displaybreak[1]\\
   &{=}\frac{\alpha^2L}{2N}\sum_{i=1}^N\mathbb{E}\left[\left\|\sum_{s=0}^{\tau-1} g_{\pmb{\theta}}(\pmb{\phi}_i(k), \pmb{\theta}_i(k,s))-\nabla_{\pmb{\theta}}F_i(\pmb{\phi}_i(k),\pmb{\theta}_i(k,s))\right\|^2\right]+\frac{\alpha^2L}{2N}\sum_{i=1}^N\left\|\mathbb{E} \sum_{s=0}^{\tau-1} g_{\pmb{\theta}}(\pmb{\phi}_i(k), \pmb{\theta}_i(k,s))\right\|^2 \nonumber\displaybreak[2]\\
   &{\leq}\frac{\alpha^2\tau L}{2}\sigma^2+\frac{\alpha^2L}{2N}\sum_{i=1}^N\left\|\mathbb{E} \sum_{s=0}^{\tau-1} g_{\pmb{\theta}}(\pmb{\phi}_i(k), \pmb{\theta}_i(k,s))\right\|^2. 
\end{align}
\end{proof}

Now we are ready to prove the main results.  Given the bounds in Lemmas~\ref{lem:C1},~\ref{lem:C2},~\ref{lem:C3} and~\ref{lem:C4} for $C_1, C_2, C_3$ and $C_4$, respectively, we substitute these bound back to \eqref{eq:global_iter2} and obtain
\begin{align}\label{eq:global_iter3}
 \mathbb{E}&[f(\bar{\pmb{\phi}}(k+1),\{\pmb{\theta}_i(k+1)\}_{i=1}^N)]- \mathbb{E}[f(\bar{\pmb{\phi}}(k),\{\pmb{\theta}_i(k)\}_{i=1}^N)]\nonumber\displaybreak[0]\\
 &\quad{\leq}{\frac{\beta L^2}{2N}\sum_{i=1}^N\mathbb{E}\left[\left\|\bar{\pmb{\phi}}(k)-\pmb{\phi}_i(k)\right\|^2\right]}-\frac{\beta}{2}\left\|\nabla_{\pmb{\phi}}f(\bar{\pmb{\phi}}(k),\{\pmb{\theta}_i(k+1)\}_{i=1}^N)\right\|^2\nonumber\allowdisplaybreaks\\
 &\qquad-\frac{\beta}{2N^2}\left\|\mathbb{E}\sum_{i=1}^N g_{\pmb{\phi}}(\pmb{\phi}_i(k),\pmb{\theta}_i(k+1))\right\|^2+\frac{\beta^2L}{2N}\sigma^2+\frac{\beta^2L}{2N^2} \left\|\mathbb{E}\sum_{i=1}^N g_{\pmb{\phi}}(\pmb{\phi}_i(k), \pmb{\theta}_i(k+1))\right\|^2\nonumber \\ \displaybreak[3]
     &\qquad-\frac{\alpha\tau}{2} \left\|\nabla_{\pmb{\theta}}f(\bar{\pmb{\phi}}(k),\{\pmb{\theta}_i(k))\}_{i=1}^N\right\|^2-\frac{\alpha}{2N\tau}\sum_{i=1}^N\left\| \mathbb{E}\sum_{s=0}^{\tau-1} g_{\pmb{\theta}}(\pmb{\phi}_i(k),\pmb{\theta}_i(k,s))\right\|^2\nonumber\allowdisplaybreaks\\
    &\qquad+\frac{\alpha\tau L^2}{N}\sum_{i=1}^N\mathbb{E}\left[\left\|\pmb{\phi}_i(k)- \bar{\pmb{\phi}}(k)\right\|^2\right]+\alpha^3L^2(18\tau^3-15\tau^2-3\tau)\sigma^2\nonumber\allowdisplaybreaks\\
    &\qquad+\frac{\alpha^3L^2(18\tau^3-18\tau^2)}{N}\sum_{i=1}^N\mathbb{E}\left[\left\|g_{\pmb{\theta}}(\pmb{\phi}_i(k), \pmb{\theta}_i(k))\right\|^2\right]+\frac{\alpha^2L}{2N}\sum_{i=1}^N\left\|\mathbb{E}\sum_{s=0}^{\tau-1}g_{\pmb{\theta}}(\pmb{\phi}_i(k),\pmb{\theta}_i(k,s))\right\|^2+\frac{\alpha^2\tau L}{2}\sigma^2\nonumber \allowdisplaybreaks\\
    &={\frac{(2\alpha\tau+\beta) L^2}{2N}\sum_{i=1}^N\mathbb{E}\left[\left\|\bar{\pmb{\phi}}(k)-\pmb{\phi}_i(k)\right\|^2\right]}+\frac{\beta^2L}{2N}\sigma^2+\alpha^3L^2(18\tau^3-15\tau^2-3\tau)\sigma^2+\frac{\alpha^2\tau L}{2}\sigma^2\nonumber\allowdisplaybreaks\\
    &\quad\quad\quad+\frac{\beta^2L-\beta}{2N^2}\left\|\mathbb{E}\sum_{i=1}^N g_{\pmb{\phi}}(\pmb{\phi}_i(k),\pmb{\theta}_i(k+1))\right\|^2-\frac{\beta}{2}\left\|\nabla_{\pmb{\phi}}f(\bar{\pmb{\phi}}(k),\{\pmb{\theta}_i(k+1)\}_{i=1}^N)\right\|^2\nonumber\allowdisplaybreaks\\
    &\quad\quad\quad +\frac{\alpha^2\tau L-\alpha}{2N\tau}\sum_{i=1}^N\left\|\mathbb{E}\sum_{s=0}^{\tau-1} g_{\pmb{\theta}}(\pmb{\phi}_i(k),\pmb{\theta}_i(k,s))\right\|^2-\frac{\alpha\tau}{2}\left\|\nabla_{\pmb{\theta}}f(\bar{\pmb{\phi}}(k),\{\pmb{\theta}_i(k)\}_{i=1}^N)\right\|^2\nonumber\allowdisplaybreaks\\
    &\quad\quad\quad +\frac{18\alpha^3 L^2(\tau^3-\tau^2)}{N}\sum_{i=1}^N\mathbb{E}\left[\left\|g_{\pmb{\theta}}(\pmb{\phi}_i(k), \pmb{\theta}_i(k))\right\|^2\right],
\end{align}
where the first inequality comes from the fact that $\|\mathbb{E}[\bX]\|^2\leq \mathbb{E}[\|\bX\|^2]$.
To characterize the convergence rate, the key then boils down to bound  the consensus error of the global representation $\pmb{\phi}$, i.e., $\mathbb{E}\left[\|\pmb{\phi}_i(k)-\bar{\pmb{\phi}}(k)\|^2\right], \forall i$. We bound it in the following lemma.

\begin{lemma}\label{lemma:consensus_error}
The consensus error $\mathbb{E}\left[\|\pmb{\phi}_i(k)-\bar{\pmb{\phi}}(k)\|^2\right], \forall i, k\geq 1$ is upper bounded by
\begin{align}
    \mathbb{E}\left[\|\pmb{\phi}_i(k)-\bar{\pmb{\phi}}(k)\|^2\right] &\!\leq \!\frac{6\beta^2C^2N^2}{(1-q)^2}\sigma^2\!+\!\frac{18\beta^2C^2N^2}{(1-q)^2}\varsigma^2+\frac{18\beta^2C^2L^2N}{1-q}\sum_{r=0}^{k-1}q^{k-r}\sum_{j=1}^N\mathbb{E}\|\pmb{\phi}_j(r)-\bar{\pmb{\phi}}(r)\|^2\nonumber\displaybreak[0]\\
&+\frac{18\beta^2C^2N^2}{1-q}\sum_{r=0}^{k-1}q^{k-r}\mathbb{E}\|\nabla_{\pmb{\phi}} f(\bar{\pmb{\phi}}(r),\{\pmb{\theta}_j(r+1)\}_{j=1}^N)\|^2.
\end{align}

\end{lemma}
\begin{proof}
Based on \eqref{eq:global_model_update}, \eqref{eq:local-subgradient2} and \eqref{eq:global_model_update2},  we have $\forall k\geq 1$
\begin{align}
  \mathbb{E}&\left[\|\pmb{\phi}_i(k)-\bar{\pmb{\phi}}(k)\|^2 \right]\nonumber\displaybreak[0]\\ 
  &=\mathbb{E}\Bigg\|\frac{1}{N}\sum\limits_{j=1}^{N} \pmb{\phi}_j(0)-\sum\limits_{j=1}^{N}\pmb{\phi}_j(0)\bP^{k}(i,j)-\frac{\beta}{N}\sum\limits_{r=0}^{k-1}\sum\limits_{j=1}^{N}g_{\pmb{\phi}}(\pmb{\phi}_j(r), \pmb{\theta}_j(r+1))
 \displaybreak[0]\nonumber\displaybreak[1]\\ 
&\qquad\qquad\qquad+\beta\sum\limits_{r=0}^{k-1}\sum\limits_{j=1}^{N}g_{\pmb{\phi}}(\pmb{\phi}_j(r),\pmb{\theta}_j(r+1))\bP^{k-r}(i,j) \Bigg\|^2\nonumber\allowdisplaybreaks\\
&=\mathbb{E}\Bigg\|\sum\limits_{j=1}^{N} \pmb{\phi}_j(0)\left(\frac{1}{N}-\bP^k(i,j)\right)-\beta\sum\limits_{r=0}^{k-1}\sum\limits_{j=1}^{N}g_{\pmb{\phi}}(\pmb{\phi}_j(r),\pmb{\theta}_j(r+1))\left(\frac{1}{N}-\bP^{k-r}(i,j)\right) \Bigg\|^2\displaybreak[0]\nonumber\\
&\overset{(e_1)}{\leq}\mathbb{E}\left\|\sum\limits_{j=1}^{N} \pmb{\phi}_j(0)\left(\frac{1}{N}-\bP^k(i,j)\right)\right\|^2+\mathbb{E}\left\|\beta\sum\limits_{r=0}^{k-1}\sum\limits_{j=1}^{N}g_{\pmb{\phi}}(\pmb{\phi}_j(r),\pmb{\theta}_j(r+1))\left(\frac{1}{N}-\bP^{k-r}(i,j)\right)\right\|^2\displaybreak[0] \nonumber\\
&\overset{(e_2)}{\leq}\mathbb{E}\Bigg\|\sum\limits_{j=1}^{N}2\pmb{\phi}_j(0)\frac{1+p^{-N}}{1-p^{N}}(1-p^{N})^{(k-1)/N}\Bigg\|^2\displaybreak[0] \nonumber\\
&\qquad\qquad\qquad+\mathbb{E}\Bigg\|\beta\sum\limits_{r=0}^{k-1}\sum\limits_{j=1}^{N}2g_{\pmb{\phi}}(\pmb{\phi}_j(r),\pmb{\theta}_j(r+1))\frac{1+p^{-N}}{1-p^{N}}(1-p^{N})^{(k-r)/N} \Bigg\|^2\displaybreak[0]\nonumber\\
&\overset{(e_3)}{=}\mathbb{E}\Bigg\|\beta C\sum\limits_{r=0}^{k-1}q^{(k-r)}\sum\limits_{j=1}^{N}g_{\pmb{\phi}}(\pmb{\phi}_j(r),\pmb{\theta}_j(r+1)) \Bigg\|^2\nonumber \displaybreak[0]\\
&=\mathbb{E}\Bigg\|\beta C\sum\limits_{r=0}^{k-1}q^{(k-r)}\sum\limits_{j=1}^{N}(g_{\pmb{\phi}}(\pmb{\phi}_j(r),\pmb{\theta}_j(r+1))-\nabla_{\pmb{\phi}}F_j(\pmb{\phi}_j(r),\pmb{\theta}_j(r+1)))\nonumber\displaybreak[0]\\
&\qquad\qquad\qquad+\beta C\sum\limits_{r=0}^{k-1}q^{(k-r)}\sum\limits_{j=1}^{N}\nabla_{\pmb{\phi}}F_j(\pmb{\phi}_j(r),\pmb{\theta}_j(r+1)) \Bigg\|^2\displaybreak[0]\nonumber\\
&\overset{(e_4)}{\leq}\mathbb{E}\Bigg\|\beta C\sum\limits_{r=0}^{k-1}q^{(k-r)}\sum\limits_{j=1}^{N}(g_{\pmb{\phi}}(\pmb{\phi}_j(r),\pmb{\theta}_j(r+1))-\nabla_{\pmb{\phi}}F_j(\pmb{\phi}_j(r),\pmb{\theta}_j(r+1)))\Bigg\|^2\displaybreak[0]\nonumber\\
&\qquad +\mathbb{E}2\beta^2 C^2\sum\limits_{r=0}^{k-1}\sum\limits_{r^\prime=0}^{k-1}q^{(2k-r-r^\prime)}\Bigg\|\sum\limits_{j=1}^{N}(g_{\pmb{\phi}}(\pmb{\phi}_j(r),\pmb{\theta}_j(r+1))-\nabla_{\pmb{\phi}}F_j(\pmb{\phi}_j(r),\pmb{\theta}_j(r+1)))\Bigg\|\displaybreak[0]\nonumber\\
&\qquad\qquad\cdot\Bigg\|\sum\limits_{j=1}^{N}\nabla_{\pmb{\phi}}F_j(\pmb{\phi}_j(r),\pmb{\theta}_j(r+1))\Bigg\|+\mathbb{E}\Bigg\|\beta C\sum\limits_{r=0}^{k-1}q^{(k-r)}\sum\limits_{j=1}^{N}\nabla_{\pmb{\phi}}F_j(\pmb{\phi}_j(r),\pmb{\theta}_j(r+1))\Bigg\|^2\displaybreak[0]\nonumber\\
&\overset{(e_5)}{\leq} \mathbb{E}\frac{6\beta^2C^2}{1-q}\sum_{r=0}^{k-1} q^{k-r}\Bigg\|\sum\limits_{j=1}^{N}(g_{\pmb{\phi}}(\pmb{\phi}_j(r),\pmb{\theta}_j(r+1))-\nabla_{\pmb{\phi}}F_j(\pmb{\phi}_j(r),\pmb{\theta}_j(r+1)))\Bigg\|^2\nonumber\allowdisplaybreaks\\
&\qquad\qquad\qquad+ \mathbb{E}\frac{6\beta^2C^2}{1-q}\sum_{r=0}^{k-1} q^{k-r}\Bigg\|\sum\limits_{j=1}^{N}\nabla_{\pmb{\phi}}F_j(\pmb{\phi}_j(r),\pmb{\theta}_j(r+1))\Bigg\|^2\displaybreak[0]\nonumber\\
&\overset{(e_6)}{\leq} \frac{6\beta^2C^2N^2}{(1-q)^2}\sigma^2+\frac{18\beta^2C^2N^2}{(1-q)^2}\varsigma^2+\frac{18\beta^2C^2L^2N}{1-q}\sum_{r=0}^{k-1}q^{k-r}\sum_{j=1}^N\mathbb{E}\|\pmb{\phi}_j(r)-\bar{\pmb{\phi}}(r)\|^2\nonumber\displaybreak[0]\\
&\qquad\qquad\qquad+\frac{18\beta^2C^2N^2}{1-q}\sum_{r=0}^{k-1}q^{k-r}\mathbb{E}\|\nabla_{\pmb{\phi}} f(\bar{\pmb{\phi}}(r),\{\pmb{\theta}_j(r+1)\}_{j=1}^N)\|^2,
\end{align}
where $(e_1)$  is due to the inequality $\|\ba-\bb\|^2\leq 2\|\ba\|^2+2\|\bb\|^2$; $(e_2)$ holds according to Lemma \ref{lemma_bound_Phi}. W.l.o.g., we assume that the initial term $\pmb{\phi}_i(0), \forall i$ is small enough and can be neglected;   
$(e_3)$ follows $C:=2\sqrt{2}\cdot\frac{1+p^{-N}}{1-p^{N}}$ and $q:=(1-p^{N})^{1/N}$;
$(e_4)$ is due to $\|\ba+\bb\|^2\leq \|\ba\|^2+\|\bb\|^2+2\ba\bb$; and $(e_5)$ is the standard mathematical manipulation by leveraging the following inequality, i.e., for any $q\in(0,1)$ and non-negative sequence $\{\chi(r)\}_{r=0}^{k-1}$, it holds \citep{assran2019stochastic}
    $\sum_{k=1}^{K-1}\sum_{r=0}^{k-1} q^{k-r}\chi(r)\leq \frac{1}{1-q}\sum_{r=0}^{K-1}\chi(r)$;
and $(e_6)$ holds due to Assumption \ref{assumption-gradient}, the inequality
\begin{align*}
    \sum_{k=1}^{K-1} q^k\sum_{r=0}^{k-1} q^{k-r}\chi(r)\leq \sum_{k=1}^{K-1}\sum_{r=0}^{k-1} q^{2(k-r)}\chi(r)\leq \frac{1}{1-q^2}\sum_{r=0}^{K-1}\chi(r),
\end{align*}
and the fact that
\begin{align*}
    &\frac{1}{N}\sum_{i=1}^N\mathbb{E}\|\nabla_{\pmb{\phi}} F_i(\pmb{\phi}_i(k-1), \pmb{\theta}_i(k))\|^2\allowdisplaybreaks\\
    &\leq \frac{1}{N}\sum_{i=1}^N\mathbb{E}\|\nabla_{\pmb{\phi}} F_i(\pmb{\phi}_i(k-1), \pmb{\theta}_i(k))-\nabla_{\pmb{\phi}} F_i(\bar{\pmb{\phi}}(k-1), \pmb{\theta}_i(k))+\nabla_{\pmb{\phi}} F_i(\bar{\pmb{\phi}}(k-1), \pmb{\theta}_i(k))\allowdisplaybreaks\\
    &\quad\quad-\nabla_{\pmb{\phi}} f(\bar{\pmb{\phi}}(k-1), \{\pmb{\theta}_i(k)\}_{i=1}^N)+\nabla_{\pmb{\phi}} f(\bar{\pmb{\phi}}(k-1), \{\pmb{\theta}_i(k)\}_{i=1}^N)\|^2\allowdisplaybreaks\\
    &\leq \underset{\text{Lipschitz continuous gradient in Assumption \ref{assumption-lipschitz}}}{\underbrace{\frac{3}{N}\sum_{i=1}^N\mathbb{E}\|\nabla_{\pmb{\phi}} F_i(\pmb{\phi}_i(k-1), \pmb{\theta}_i(k))-\nabla_{\pmb{\phi}} F_i(\bar{\pmb{\phi}}(k-1)\|^2}}\allowdisplaybreaks\\
   &\quad\quad+\underset{\text{Bounded global variablity in Assumption \ref{assumption:global-var}}}{\underbrace{\frac{3}{N}\sum_{i=1}^N\mathbb{E}\|\nabla_{\pmb{\phi}} F_i(\bar{\pmb{\phi}}(k-1), \pmb{\theta}_i(k))-\nabla_{\pmb{\phi}} f(\bar{\pmb{\phi}}(k-1), \{\pmb{\theta}_i(k)\}_{i=1}^N)\|^2}}\allowdisplaybreaks\\
   &\quad\quad+\frac{3}{N}\sum_{i=1}^N\mathbb{E}\|\nabla_{\pmb{\phi}} f(\bar{\pmb{\phi}}(k-1), \{\pmb{\theta}_i(k)\}_{i=1}^N)\|^2\allowdisplaybreaks\\
    &\leq \frac{3L^2}{N}\sum_{i=1}^N\mathbb{E}\|\bw_i(k)-\bar{\bw}(k)\|^2+3\varsigma^2+3\mathbb{E}\|\nabla_{\pmb{\phi}} f(\bar{\pmb{\phi}}(k-1), \{\pmb{\theta}_i(k)\}_{i=1}^N)\|^2.
\end{align*}
This completes the proof.
\end{proof}

Rearrange the order of each term in \eqref{eq:global_iter3} and let $\max(L\beta, \alpha\tau L(1+36\tau^2))\leq 1$, we have
\begin{align} \label{eq:gradient_2}
&\frac{\alpha\tau}{2}\left\|\nabla_{\pmb{\theta}}f(\bar{\pmb{\phi}}(k),\{\pmb{\theta}_i(k)\}_{i=1}^N)\right\|^2+\frac{\beta}{2}\left\|\nabla_{\pmb{\phi}}f(\bar{\pmb{\phi}}(k),\{\pmb{\theta}_i(k+1)\}_{i=1}^N)\right\|^2\nonumber\\ \displaybreak[0]
\leq &\mathbb{E}[f(\bar{\pmb{\phi}}(k),\{\pmb{\theta}_i(k)\}_{i=1}^N)]-\mathbb{E} [f(\bar{\pmb{\phi}}(k+1),\{\pmb{\theta}_i(k+1)\}_{i=1}^N)] \nonumber\\ \displaybreak[0]
&\quad+ {\frac{(2\alpha\tau+\beta) L^2}{2N}\sum_{i=1}^N\mathbb{E}\left\|\bar{\pmb{\phi}}(k)-\pmb{\phi}_i(k)\right\|^2}+\frac{\beta^2L}{2N}\sigma^2+\alpha^3L^2(18\tau^3-15\tau^2-3\tau)\sigma^2+\frac{\alpha^2\tau L}{2}\sigma^2\nonumber\\
    &\quad+\underset{\leq 0 }{\underbrace{\frac{\alpha^2\tau L-\alpha}{2N\tau}\sum_{i=1}^N\mathbb{E}\left\|\sum_{s=0}^{\tau-1} g_{\pmb{\theta}}(\pmb{\phi}_i(k),\pmb{\theta}_i(k,s))\right\|^2+\frac{18\alpha^3L^2(\tau^3-\tau^2)}{N}\sum_{i=1}^N\mathbb{E}\left\|g_{\pmb{\theta}}(\pmb{\phi}_i(k), \pmb{\theta}_i(k))\right\|^2}}\nonumber\allowdisplaybreaks\\
    &\quad +\underset{\leq 0 }{\underbrace{\frac{\beta^2L-\beta}{2N^2}\mathbb{E}\left\|\sum_{i=1}^N g_{\pmb{\phi}}(\pmb{\phi}_i(k),\pmb{\theta}_i(k+1))\right\|^2}}\nonumber\allowdisplaybreaks\\
 \leq& \mathbb{E}[f(\bar{\pmb{\phi}}(k),\{\pmb{\theta}_i(k)\}_{i=1}^N)]-\mathbb{E} [f(\bar{\pmb{\phi}}(k+1),\{\pmb{\theta}_i(k+1)\}_{i=1}^N)] \nonumber\\ \displaybreak[0]
&\quad+ {\frac{(2\alpha\tau+\beta) L^2}{2N}\sum_{i=1}^N\mathbb{E}\left\|\bar{\pmb{\phi}}(k)-\pmb{\phi}_i(k)\right\|^2}+\frac{\beta^2L}{2N}\sigma^2+\alpha^3L^2(18\tau^3-15\tau^2-3\tau)\sigma^2+\frac{\alpha^2\tau L}{2}\sigma^2.
\end{align}

According to Lemma \ref{lemma:consensus_error}, we have the following inequality
\begin{align}
  &\sum_{k=0}^{K-1}\sum\limits_{i=1}^{N}\mathbb{E}\|\pmb{\phi}_i(k)-\bar{\pmb{\phi}}(k)\|^2  \nonumber\allowdisplaybreaks\\
  \leq &\sum\limits_{i=1}^{N}\mathbb{E}\|\pmb{\phi}_i(0)-\bar{\pmb{\phi}}(0)\|^2+\sum_{k=1}^{K-1}\Bigg(\frac{18N^2\beta^2C^2L^2}{1-q}\sum_{r=0}^{k-1}q^{k-r}\sum_{i=1}^N\mathbb{E}\|\pmb{\phi}_i(r)-\bar{\pmb{\phi}}(r)\|^2\nonumber\allowdisplaybreaks\\
&\quad+\frac{18\beta^2C^2N^3}{1-q}\sum_{r=0}^{k-1}q^{k-r}\mathbb{E}\|\nabla_{\pmb{\phi}} f(\bar{\pmb{\phi}}(r),\{\pmb{\theta}_i(r+1)\}_{i=1}^N)\|^2+\frac{6\beta^2C^2N^3}{(1-q)^2}\sigma^2+\frac{18\beta^2C^2N^3}{(1-q)^2}\varsigma^2\Bigg)\displaybreak[0]\nonumber\allowdisplaybreaks\\
\leq& \Bigg(\frac{18N^2\beta^2C^2L^2}{(1-q)^2}\sum_{k=0}^{K-1}\sum_{i=1}^N\mathbb{E}\|\pmb{\phi}_i(k)-\bar{\pmb{\phi}}(k)\|^2 \nonumber\allowdisplaybreaks\\
&\quad+\frac{18\beta^2C^2N^3}{(1-q)^2}\sum_{k=0}^{K-1}\mathbb{E}\|\nabla_{\pmb{\phi}} f(\bar{\pmb{\phi}}(k),\{\pmb{\theta}_i(k+1)\}_{i=1}^N)\|^2+\frac{6K\beta^2C^2N^3}{(1-q)^2}\sigma^2+\frac{18K\beta^2C^2N^3}{(1-q)^2}\varsigma^2\Bigg)\displaybreak[0]\nonumber\allowdisplaybreaks\\
\leq&\frac{18\beta^2C^2N^3}{(1-q)^2-18N^2\beta^2C^2L^2}\sum_{k=0}^{K-1}\mathbb{E}\|\nabla_{\pmb{\phi}} f(\bar{\pmb{\phi}}(k),\{\pmb{\theta}_i(k+1)\}_{i=1}^N)\|^2+\frac{6K\beta^2C^2N^3}{(1-q)^2-18N^2\beta^2C^2L^2}\sigma^2\nonumber\allowdisplaybreaks\\
&\quad+\frac{18K\beta^2C^2N^3}{(1-q)^2-18N^2\beta^2C^2L^2}\varsigma^2,
\end{align}
where the second inequality holds due to the initialization such that $\sum_{i=1}^N\mathbb{E}\|\pmb{\phi}_i(0)-\bar{\pmb{\phi}}(k)\|^2=0.$ 
Summing the recursion in \eqref{eq:gradient_2} from round $0$ to round $K-1$  yields
\begin{align}
\sum\limits_{k=0}^{K-1}& \frac{\alpha\tau}{2}\left\|\nabla_{\pmb{\theta}}f(\bar{\pmb{\phi}}(k),\{\pmb{\theta}_i(k)\}_{i=1}^N)\right\|^2+\frac{\beta}{2}\left\|\nabla_{\pmb{\phi}}f(\bar{\pmb{\phi}}(k),\{\pmb{\theta}_i(k+1)\}_{i=1}^N)\right\|^2\nonumber\\
\leq &\mathbb{E}[f(\bar{\pmb{\phi}}(0),\{\pmb{\theta}_i(0)\}_{i=1}^N)]-\mathbb{E} [f(\bar{\pmb{\phi}}(K),\{\pmb{\theta}_i(K)\}_{i=1}^N)]\nonumber\\
&\qquad + {\frac{(2\alpha\tau+\beta) L^2}{2N}\sum_{k=0}^{K-1}\sum_{i=1}^N\mathbb{E}\left\|\bar{\pmb{\phi}}(k)-\pmb{\phi}_i(k)\right\|^2}+\frac{K\beta^2L}{2N}\sigma^2\nonumber\allowdisplaybreaks\\
&\qquad+K\alpha^3L^2(18\tau^3-15\tau^2-3\tau)\sigma^2+\frac{K\alpha^2\tau L}{2}\sigma^2\nonumber\\
\leq &\mathbb{E}[f(\bar{\pmb{\phi}}(0),\{\pmb{\theta}_i(0)\}_{i=1}^N)]-\mathbb{E} [f(\bar{\pmb{\phi}}(K),\{\pmb{\theta}_i(K)\}_{i=1}^N)]\nonumber\allowdisplaybreaks\\
&\qquad+\frac{K\beta^2L}{2N}\sigma^2+K\alpha^3L^2(18\tau^3-15\tau^2-3\tau)\sigma^2+\frac{K\alpha^2\tau L}{2}\sigma^2\nonumber\displaybreak[0]\\
&\qquad+ \frac{9\beta^2C^2N^2L^2(2\alpha\tau+\beta)}{(1-q)^2-18N^2\beta^2C^2L^2}\sum_{k=0}^{K-1}\mathbb{E}\|\nabla_{\pmb{\phi}} f(\bar{\pmb{\phi}}(k),\{\pmb{\theta}_i(k)\}_{i=1}^N)\|^2\nonumber\displaybreak[0]\\
&\qquad+\frac{3K\beta^2C^2N^2L^2(2\alpha\tau+\beta)}{(1-q)^2-18N^2\beta^2C^2L^2}\sigma_L^2+\frac{9K\beta^2C^2N^2L^2(2\alpha\tau+\beta)}{(1-q)^2-18N^2\beta^2C^2L^2}\varsigma^2.
\end{align}

Let $\beta\leq \min\left(1/L,\frac{NL^2}{2L^2+2},\frac{1-q}{3\sqrt{2}CLN}\right)$, we obtain $(1-q)^2-18N^2\beta^2C^2L^2\geq 18 C^2N^3L^2(2\alpha\tau+\beta).$

Hence we have
\begin{align}
&\sum\limits_{k=0}^{K-1} \frac{\alpha\tau}{4}\left\|\nabla_{\pmb{\theta}}f(\bar{\pmb{\phi}}(k),\{\pmb{\theta}_i(k)\}_{i=1}^N)\right\|^2+\frac{\beta}{4}\left\|\nabla_{\pmb{\phi}}f(\bar{\pmb{\phi}}(k),\{\pmb{\theta}_i(k+1)\}_{i=1}^N)\right\|^2+\frac{\beta}{4N}\sum_{k=0}^{K-1}\sum_{i=1}^N\mathbb{E}\|\bar{\pmb{\phi}}(k)-\pmb{\phi}_i(k)\|^2 \nonumber\\
\leq&\sum\limits_{k=0}^{K-1} \frac{\alpha\tau}{2}\left\|\nabla_{\pmb{\theta}}f(\bar{\pmb{\phi}}(k),\{\pmb{\theta}_i(k)\}_{i=1}^N)\right\|^2+\frac{\beta}{4}\left\|\nabla_{\pmb{\phi}}f(\bar{\pmb{\phi}}(k),\{\pmb{\theta}_i(k+1)\}_{i=1}^N)\right\|^2+\frac{\beta}{2N}\sum_{k=0}^{K-1}\sum_{i=1}^N\mathbb{E}\|\bar{\pmb{\phi}}(k)-\pmb{\phi}_i(k)\|^2 \nonumber\\
\leq&\mathbb{E}[f(\bar{\pmb{\phi}}(0),\{\pmb{\theta}_i(0)\}_{i=1}^N)]-\mathbb{E} [f(\bar{\pmb{\phi}}(K),\{\pmb{\theta}_i(K)\}_{i=1}^N)]+\frac{K\beta^2L}{2N}\sigma_L^2+K\alpha^3L^2(18\tau^3-15\tau^2-3\tau)\sigma_L^2\nonumber\displaybreak[0]\\
&\qquad\qquad\qquad+\frac{K\alpha^2\tau L}{2}\sigma_L^2+\frac{K\beta^2}{6N}\left(1+\frac{1}{L^2}\right)\sigma_L^2+\frac{K\beta^2}{2N}\left(1+\frac{1}{L^2}\right)\varsigma^2.
\end{align}
Dividing both sides by $\beta K/4$, we obtain
\begin{align}
\frac{1}{K}\sum_{k=0}^{K-1}\mathbb{E}[M(k) ]
&=\frac{1}{K}\sum\limits_{k=0}^{K-1} \frac{\alpha\tau}{\beta}\mathbb{E}\left\|\nabla_{\pmb{\theta}}f(\bar{\pmb{\phi}}(k),\{\pmb{\theta}_i(k)\}_{i=1}^N)\right\|^2+\mathbb{E}\left\|\nabla_{\pmb{\phi}}f(\bar{\pmb{\phi}}(k),\{\pmb{\theta}_i(k+1)\}_{i=1}^N)\right\|^2\nonumber\allowdisplaybreaks\\
&\qquad+\frac{1}{N}\sum_{i=1}^N\mathbb{E}\|\bar{\pmb{\phi}}(k)-\pmb{\phi}_i(k)\|^2\nonumber\\
\leq& \frac{4\mathbb{E}[f(\bar{\pmb{\phi}}(0),\{\pmb{\theta}_i(0)\}_{i=1}^N)]-4\mathbb{E} [f(\bar{\pmb{\phi}}(K),\{\pmb{\theta}_i(K)\}_{i=1}^N)]}{K\beta}\nonumber\displaybreak[0]\\
&\qquad+\frac{2\beta L}{N}\sigma_L^2+\frac{12\alpha^3L^2}{\beta}(\tau-1)(6\tau^2-\tau)\sigma_L^2\nonumber\\
&\qquad+\frac{2\alpha^2\tau L}{\beta}\sigma_L^2+\frac{2\beta}{3 N}\left(1+\frac{1}{L^2}\right)\sigma_L^2+\frac{2\beta}{ N}\left(1+\frac{1}{L^2}\right)\varsigma^2\nonumber\allowdisplaybreaks\\
\leq &\frac{4f\left(\bar{\pmb{\phi}}(0),\{\pmb{\theta}_i(0)\}_{i=1}^N\right)-4f\left({\pmb{\phi}}^*,\{\pmb{\theta}_i^*\}_{i=1}^N\right)}{K\beta}\nonumber\displaybreak[0]\\
&\qquad+\frac{2\beta L}{N}\sigma_L^2+\frac{12\alpha^3L^2}{\beta}(\tau-1)(6\tau^2-\tau)\sigma_L^2\nonumber\\
&\qquad+\frac{2\alpha^2\tau L}{\beta}\sigma_L^2+\frac{2\beta}{3 N}\left(1+\frac{1}{L^2}\right)\sigma_L^2+\frac{2\beta}{ N}\left(1+\frac{1}{L^2}\right)\varsigma^2.
\end{align}
This completes the proof of Theorem \ref{thm:loss_convergence}.

\section{Proof of Corollary \ref{cor:1}}
The proof of corollary \ref{cor:1} goes as follow. 
When the total number of communication rounds $K$ satisfies 
    $K\geq 
    \max\left(\frac{18C^2L^2N^3}{(1-q)^2},\frac{(2L^2+2)^2}{NL^4},NL^2\right),$
we can set $\alpha=\frac{1}{\tau\sqrt{K}}$ and $\beta=\sqrt{{N}/{K}}$, and substitute them back into (\ref{eq:thm_loss}), which leads to
\begin{align}
&\frac{1}{K}\sum\limits_{k=0}^{K-1}\mathbb{E}[M(k)]
\leq\frac{4f(\bar{\pmb{\phi}}(0),\{\pmb{\theta}_i(0)\}_{i=1}^N)-4 f({\pmb{\phi}^*},\{\pmb{\theta}_i^*\}_{i=1}^N)}{\sqrt{NK}}\nonumber\displaybreak[1]\\
+&\frac{2 L}{\sqrt{NK}}\sigma^2+\frac{72L^2}{K\sqrt{N}}\sigma^2+\frac{2L}{\tau\sqrt{NK}}\sigma^2\!+\!\frac{2}{3 \sqrt{NK}}\!\left(\!1\!+\!\frac{1}{L^2}\!\right)\!\sigma^2\!+\!\frac{2}{\sqrt{NK}}\!\left(\!1\!+\!\frac{1}{L^2}\!\right)\!\varsigma^2.
\end{align}
Therefore, the convergence rate of \DePRL is 
$\mathcal{O}\left(\frac{1}{\sqrt{NK}}+\frac{1}{K\sqrt{N}}+\frac{1}{\tau\sqrt{NK}}\right),$ which completes the proof.

\begin{table}[t]
	\centering
	\scalebox{0.7}{
	\begin{tabular}{ccc}
		\hline
		Parameter & Shape & Layer hyper-parameter \\ \hline
		layer1.conv1.weight & $3\times 3, 64$ & stride:1; padding: 1 \\
		layer1.conv1.bias & 64 & N/A \\
		batchnorm2d & 64 & N/A \\
		layer2.conv2 &$ \begin{bmatrix} 3\times 3, & 64 \\ 3\times 3, & 64 \\ \end{bmatrix} \times 2$ & stride:1; padding: 1 \\
		layer3.conv3 &$ \begin{bmatrix} 3\times 3, & 128 \\ 3\times 3, & 128 \\ \end{bmatrix} \times 2$ & stride:1; padding: 1 \\
		layer4.conv4 &$ \begin{bmatrix} 3\times 3, & 256 \\ 3\times 3, & 256 \\ \end{bmatrix} \times 2$ & stride:1; padding: 1 \\
		layer5.conv5 &$ \begin{bmatrix} 3\times 3, & 512 \\ 3\times 3, & 512 \\ \end{bmatrix} \times 2$ & stride:1; padding: 1 \\
		pooling.avg & N/A & N/A\\
		layer6.fc6.weight & $512 \times 10$ & N/A \\
		layer6.fc6.bias & 10 & N/A \\\hline
	\end{tabular}}
		\caption{Detailed information of the ResNet-18 architecture used in our experiments.  All non-linear activation function in this architecture is ReLU.  The shapes for convolution layers follow $(C_{in},C_{out},c,c)$.}
	\label{tab:resnet}
\end{table}

\begin{table}[h]
	\centering
	\scalebox{0.7}{
	\begin{tabular}{ccc}
		\hline
		Parameter & Shape & Layer hyper-parameter \\ \hline
		layer1.conv1.weight & $3\times 64\times 3\times 3$ & stride:1; padding: 1 \\
		layer1.conv1.bias & 64 & N/A \\
		pooling.max & N/A & kernel size:2; stride: 2 \\
		layer2.conv2.weight & $64\times 128\times 3\times 3$ & stride:1; padding: 1   \\
		layer2.conv2.bias & 128 & N/A \\
		layer3.conv3.weight & $128\times 128\times 3\times 3$ & stride:1; padding: 1   \\
		layer3.conv3.bias & 128 & N/A \\
		pooling.max & N/A & kernel size:2; stride: 2 \\
		layer4.conv4.weight & $128\times 256\times 3\times 3$ & stride:1; padding: 1   \\
		layer4.conv4.bias & 256 & N/A \\
		layer5.conv5.weight & $256\times 256\times 3\times 3$ & stride:1; padding: 1   \\
		layer5.conv5.bias & 256 & N/A \\
		pooling.max & N/A & kernel size:2; stride: 2 \\
		layer6.conv6.weight & $256\times 512\times 3\times 3$ & stride:1; padding: 1   \\
		layer6.conv6.bias & 512 & N/A \\
		layer7.conv7.weight & $512\times 512\times 3\times 3$ & stride:1; padding: 1   \\
		layer7.conv7.bias & 512 & N/A \\
		pooling.max & N/A & kernel size:2; stride: 2 \\
		layer8.conv8.weight & $512\times 512\times 3\times 3$ & stride:1; padding: 1   \\
		layer8.conv8.bias & 512 & N/A \\
		pooling.max & N/A & kernel size:2; stride: 2 \\
		dropout & N/A & p=20\% \\
		layer9.fc9.weight & $4096\times 512$ & N/A \\
		layer9.fc9.bias & 512 & N/A \\
		layer10.fc10.weight & $512 \times 512$ &  N/A \\
		layer10.fc10.bias & 512 & N/A \\
		dropout & N/A & p=20\% \\
		layer11.fc11.weight & $512 \times 100$ & N/A \\
		layer11.fc11.bias & 100 & N/A \\
		\hline
	\end{tabular}}
		\caption{Detailed information of the VGG-11 architecture used in our experiments.  All non-linear activation function in this architecture is ReLU.  The shapes for convolution layers follow $(C_{in},C_{out},c,c)$.}
	\label{tab:vgg11}
\end{table}

\section{Additional Experimental Details and Results}\label{sec:sim-app}

\textbf{Datasets and Models.} We implement \DePRL and considered baselines in PyTorch \cite{paszke2017automatic} on Python 3 with three NVIDIA RTX A6000 GPUs, 48GB with 128GB RAM.  We conduct experiments on the popular datasets CIFAR-10 and CIFAR-100 \cite{krizhevsky2009learning}, Fashion-MNIST \cite{xiao2017fashion}, and HARBox \cite{ouyang2021clusterfl}.  The CIFAR-10 and CIFAR1-00 dataset consists of 60,000 32$\times$32 color images in 10 and 100 classes, respectively, where 50,000 samples are for training and the other 10,000 samples for testing.  The Fashion-MNIST datasets contain handwritten digits with 60,000 samples for training and 10,000 samples for testing, where each sample is an 28$\times$28 grayscale images over 10 classes.  The HARBox dataset is the 9-axis IMU data collected from 121 users' smartphones for human activity recognition in a crowdsourcing manner.  We simulate a heterogeneous partition into $N$ workers by sampling $\boldsymbol p_k\sim\text{Dir}_N(\alpha)$, where $\alpha$ is the parameter of the Dirichlet distribution \citep{wang2020federated,wang2020tackling}.  The level of heterogeneity among local datasets across different workers can be reduced when $\alpha$ increases.  Specifically, for each class of samples, set the class probability in each worker by sampling from a Dirichlet distribution with the same $\alpha$ parameter. For instance, when $\alpha=0.3$, sampling $p_o\sim \text{Dir}(0.3)$ and allocating a $p_{o,i}$ proportion of the training instances of class $o$ to local worker $i$.

\begin{table}[t]
	\centering
	\scalebox{0.7}{
	\begin{tabular}{ccc}
		\hline
		Parameter & Shape & Layer hyper-parameter \\ \hline
		layer1.conv1.weight & $3\times 64\times 3\times 3$ & stride:2; padding: 1 \\
		layer1.conv1.bias & 64 & N/A \\
		pooling.max & N/A & kernel size:2; stride: 2 \\
		layer2.conv2.weight & $64\times 192\times 3\times 3$ & stride:1; padding: 1 \\
		layer2.conv2.bias & 64 & N/A \\
		pooling.max & N/A & kernel size:2; stride: 2 \\
		layer3.conv3.weight & $192\times 384\times 3\times 3$ & stride:1; padding: 1   \\
		layer3.conv3.bias & 128 & N/A \\
		layer4.conv4.weight & $384\times 256\times 3\times 3$ & stride:1; padding: 1   \\
		layer4.conv4.bias & 128 & N/A \\
		layer5.conv5.weight & $256\times 256\times 3\times 3$ & stride:1; padding: 1   \\
		layer5.conv5.bias & 256 & N/A \\
		pooling.max & N/A & kernel size:2; stride: 2 \\
		dropout & N/A & p=20\% \\
		layer6.fc6.weight & $1024 \times 4096$ & N/A \\
		layer6.fc6.bias & 512 & N/A \\
		dropout & N/A & p=20\% \\
		layer7.fc7.weight & $4096 \times 4096$ &  N/A \\
		layer7.fc7.bias & 512 & N/A \\
		layer8.fc8.weight & $4096 \times 10$ & N/A \\
		layer8.fc8.bias & 10 & N/A \\ \hline
	\end{tabular}}
		\caption{Detailed information of the AlexNet architecture used in our experiments.  All non-linear activation function in this architecture is ReLU.  The shapes for convolution layers follow $(C_{in},C_{out},c,c)$.}
	\label{tab:alexnet}
\end{table}

\begin{table}[h]
	\centering
	\scalebox{0.7}{
	\begin{tabular}{ccc}
		\hline
		Parameter & Shape & Layer hyper-parameter \\ \hline
            layer1.fc1.weight & $900 \times 512$ & N/A \\
		layer1.fc1.bias & 512 & N/A \\
            dropout & N/A & p=5\% \\
            layer2.fc2.weight & $512 \times 256$ & N/A \\
		layer2.fc2.bias & 256 & N/A \\
            dropout & N/A & p=5\% \\
            layer3.fc3.weight & $256 \times 128$ & N/A \\
		layer3.fc3.bias & 128 & N/A \\
            dropout & N/A & p=5\% \\
		layer4.fc4.weight & $128 \times 64$ &  N/A \\
		layer4.fc4.bias & 64 & N/A \\
		layer5.fc5.weight & $64 \times 5$ & N/A \\
		layer5.fc5.bias & 5 & N/A \\ \hline
	\end{tabular}}
		\caption{Detailed information of the DNN architecture used in our experiments.  All non-linear activation function in this architecture is ReLU.  The shapes for convolution layers follow $(C_{in},C_{out},c,c)$.}
	\label{tab:dnn}
\end{table}

We summarize the details of ResNet-18 \cite{he2016deep}, VGG-11 \citep{simonyan2015very}, AlexNet \citep{krizhevsky2012imagenet} and DNN architectures used in our experiments for classification tasks in Tables~\ref{tab:resnet},~\ref{tab:vgg11},~\ref{tab:alexnet} and~\ref{tab:dnn}, respectively.

\textbf{Hyperparameters.}  In our experiments, we consider the total number of workers to be 128. 
The local head learning rate $\alpha$ and global representation learning rate  $\beta$ are initialized as $0.005$ and $0.01$, and decayed with $0.96$ after each communication round.  We set the weight decay to be $10^{-5}$.  The batch size is fixed to be 128 for CIFAR100, and 16 for CIFAR10, Fashion-MNIST and HARBox.  The number of local update steps is $\tau=2$.  All results are averaged over four random seeds: 1, 12, 123, 1234. In the following, we conduct an ablation study to investigate the impact of these hyperparameters.

\begin{table}[t]
\centering
\caption{Average test accuracy  with ``Ring'' communication graph and different data heterogeneities.} 
\vspace{-0.1in}
\scalebox{0.73}{
\begin{tabular}{|c|c|ccccccc|}
\hline
\multirow{2}{*}{\begin{tabular}[c]{@{}c@{}}Dataset\\ (Model)\end{tabular}}           & \multirow{2}{*}{$\alpha$} & \multicolumn{7}{c|}{Ring}                                      \\ \cline{3-9} 
&                           & \multicolumn{1}{c|}{FedAvg} & \multicolumn{1}{c|}{FedRep} & \multicolumn{1}{c|}{Ditto} & \multicolumn{1}{c|}{FedRoD} & \multicolumn{1}{c|}{D-PSGD} & \multicolumn{1}{c|}{DisPFL} & DePRL \\ \hline\hline

\multirow{3}{*}{\begin{tabular}[c]{@{}c@{}}CIFAR-100\\ (ResNet-18)\end{tabular}}     
&0.1 &\multicolumn{1}{c|}{29.20\scriptsize{$\pm$0.5}} &\multicolumn{1}{c|}{59.74\scriptsize{$\pm$0.5}} &\multicolumn{1}{c|}{56.97\scriptsize{$\pm$0.7}} &\multicolumn{1}{c|}{58.54\scriptsize{$\pm$0.5}} &\multicolumn{1}{c|}{25.18\scriptsize{$\pm$0.4}} &\multicolumn{1}{c|}{46.09\scriptsize{$\pm$0.2}} &\textbf{60.72}\scriptsize{$\pm$0.2}  \\

&0.3 &\multicolumn{1}{c|}{31.55\scriptsize{$\pm$0.4}} &\multicolumn{1}{c|}{49.43\scriptsize{$\pm$0.3}} &\multicolumn{1}{c|}{42.20\scriptsize{$\pm$0.5}} &\multicolumn{1}{c|}{48.83\scriptsize{$\pm$0.6}} &\multicolumn{1}{c|}{26.51\scriptsize{$\pm$0.4}} &\multicolumn{1}{c|}{37.92\scriptsize{$\pm$0.3}} &\textbf{49.82}\scriptsize{$\pm$0.4}  \\

&0.5 &\multicolumn{1}{c|}{33.28\scriptsize{$\pm$0.4}} &\multicolumn{1}{c|}{45.60\scriptsize{$\pm$0.7}} &\multicolumn{1}{c|}{37.43\scriptsize{$\pm$0.4}} &\multicolumn{1}{c|}{45.27\scriptsize{$\pm$0.4}} &\multicolumn{1}{c|}{26.90\scriptsize{$\pm$0.3}} &\multicolumn{1}{c|}{35.33\scriptsize{$\pm$0.4}} &\textbf{45.89}\scriptsize{$\pm$0.4}  \\ \hline\hline

\multirow{3}{*}{\begin{tabular}[c]{@{}c@{}}CIFAR-10\\ (VGG-11)\end{tabular}}         
&0.1 &\multicolumn{1}{c|}{43.09\scriptsize{$\pm$0.7}} &\multicolumn{1}{c|}{87.86\scriptsize{$\pm$0.2}} &\multicolumn{1}{c|}{86.84\scriptsize{$\pm$0.3}} &\multicolumn{1}{c|}{87.07\scriptsize{$\pm$0.2}} &\multicolumn{1}{c|}{53.91\scriptsize{$\pm$0.2}} &\multicolumn{1}{c|}{86.38\scriptsize{$\pm$0.3}} &\textbf{89.57}\scriptsize{$\pm$0.2}   \\ 

&0.3 &\multicolumn{1}{c|}{56.33\scriptsize{$\pm$0.4}} &\multicolumn{1}{c|}{74.60\scriptsize{$\pm$0.3}} &\multicolumn{1}{c|}{73.70\scriptsize{$\pm$0.4}} &\multicolumn{1}{c|}{75.57\scriptsize{$\pm$0.1}} &\multicolumn{1}{c|}{59.17\scriptsize{$\pm$0.2}} &\multicolumn{1}{c|}{73.48\scriptsize{$\pm$0.3}} &\textbf{76.41}\scriptsize{$\pm$0.3}   \\ 

&0.5 &\multicolumn{1}{c|}{57.91\scriptsize{$\pm$0.6}} &\multicolumn{1}{c|}{70.19\scriptsize{$\pm$0.2}} &\multicolumn{1}{c|}{68.45\scriptsize{$\pm$0.5}} &\multicolumn{1}{c|}{72.21\scriptsize{$\pm$0.3}} &\multicolumn{1}{c|}{60.45\scriptsize{$\pm$0.4}} &\multicolumn{1}{c|}{68.83\scriptsize{$\pm$0.2}} &\textbf{72.51}\scriptsize{$\pm$0.2}   \\ \hline\hline

\multirow{3}{*}{\begin{tabular}[c]{@{}c@{}}Fashion\\ MNIST\\ (AlexNet)\end{tabular}} 
&0.1 &\multicolumn{1}{c|}{84.19\scriptsize{$\pm$0.2}} &\multicolumn{1}{c|}{94.27\scriptsize{$\pm$0.4}} &\multicolumn{1}{c|}{95.23\scriptsize{$\pm$0.2}} &\multicolumn{1}{c|}{95.93\scriptsize{$\pm$0.2}} &\multicolumn{1}{c|}{77.45\scriptsize{$\pm$0.2}} &\multicolumn{1}{c|}{95.74\scriptsize{$\pm$0.2}} &\textbf{96.66}\scriptsize{$\pm$0.2}  \\ 

&0.3 &\multicolumn{1}{c|}{86.58\scriptsize{$\pm$0.3}} &\multicolumn{1}{c|}{89.57\scriptsize{$\pm$0.4}} &\multicolumn{1}{c|}{90.82\scriptsize{$\pm$0.2}} &\multicolumn{1}{c|}{92.29\scriptsize{$\pm$0.5}} &\multicolumn{1}{c|}{81.95\scriptsize{$\pm$0.5}} &\multicolumn{1}{c|}{91.52\scriptsize{$\pm$0.3}} &\textbf{92.81}\scriptsize{$\pm$0.2}   \\ 

&0.5 &\multicolumn{1}{c|}{86.89\scriptsize{$\pm$0.4}} &\multicolumn{1}{c|}{88.53\scriptsize{$\pm$0.4}} &\multicolumn{1}{c|}{89.12\scriptsize{$\pm$0.3}} &\multicolumn{1}{c|}{90.92\scriptsize{$\pm$0.2}} &\multicolumn{1}{c|}{84.63\scriptsize{$\pm$0.2}} &\multicolumn{1}{c|}{89.49\scriptsize{$\pm$0.2}} &\textbf{91.36}\scriptsize{$\pm$0.3}   \\ \hline\hline

\multirow{3}{*}{\begin{tabular}[c]{@{}c@{}}HARBox\\ (DNN)\end{tabular}}              
&0.1 &\multicolumn{1}{c|}{48.34\scriptsize{$\pm$0.2}} &\multicolumn{1}{c|}{90.08\scriptsize{$\pm$0.1}} &\multicolumn{1}{c|}{88.62\scriptsize{$\pm$0.1}} &\multicolumn{1}{c|}{91.23\scriptsize{$\pm$0.2}} &\multicolumn{1}{c|}{54.90\scriptsize{$\pm$0.7}} &\multicolumn{1}{c|}{90.96\scriptsize{$\pm$0.1}} &\textbf{92.07}\scriptsize{$\pm$0.1}    \\ 

&0.3 &\multicolumn{1}{c|}{50.80\scriptsize{$\pm$0.2}} &\multicolumn{1}{c|}{76.93\scriptsize{$\pm$0.2}} &\multicolumn{1}{c|}{75.56\scriptsize{$\pm$0.1}} &\multicolumn{1}{c|}{78.02\scriptsize{$\pm$0.2}} &\multicolumn{1}{c|}{55.41\scriptsize{$\pm$0.7}} &\multicolumn{1}{c|}{80.02\scriptsize{$\pm$0.2}} &\textbf{80.85}\scriptsize{$\pm$0.1}  \\ 

&0.5 &\multicolumn{1}{c|}{52.60\scriptsize{$\pm$0.2}} &\multicolumn{1}{c|}{70.54\scriptsize{$\pm$0.2}} &\multicolumn{1}{c|}{67.63\scriptsize{$\pm$0.1}} &\multicolumn{1}{c|}{72.59\scriptsize{$\pm$0.1}} &\multicolumn{1}{c|}{56.66\scriptsize{$\pm$0.7}} &\multicolumn{1}{c|}{74.47\scriptsize{$\pm$0.1}} &\textbf{75.84}\scriptsize{$\pm$0.1}   \\ \hline
\end{tabular}}
\label{tbl:final-accuracy-ring}
\end{table}

\begin{table}[t]
\centering
\caption{Generalization performance with ``Ring'' communication graph.} 
\vspace{-0.1in}
\scalebox{0.73}{
\begin{tabular}{|c|c|ccccccc|}
\hline
\multirow{2}{*}{\begin{tabular}[c]{@{}c@{}}Dataset\\ (Model)\end{tabular}}           & \multirow{2}{*}{$\alpha$} & \multicolumn{7}{c|}{Ring}                                    \\ \cline{3-9} 
&                           & \multicolumn{1}{c|}{FedAvg} & \multicolumn{1}{c|}{FedRep} & \multicolumn{1}{c|}{Ditto} & \multicolumn{1}{c|}{FedRoD} & \multicolumn{1}{c|}{D-PSGD} & \multicolumn{1}{c|}{DisPFL} & DePRL \\ \hline\hline

\multirow{3}{*}{\begin{tabular}[c]{@{}c@{}}CIFAR-100\\ (ResNet-18)\end{tabular}}     
&0.1 &\multicolumn{1}{c|}{49.64\scriptsize{$\pm$0.2}} &\multicolumn{1}{c|}{52.71\scriptsize{$\pm$0.1}} &\multicolumn{1}{c|}{42.22\scriptsize{$\pm$0.2}} &\multicolumn{1}{c|}{50.72\scriptsize{$\pm$0.1}} &\multicolumn{1}{c|}{38.64\scriptsize{$\pm$0.1}} &\multicolumn{1}{c|}{34.67\scriptsize{$\pm$0.2}} &\textbf{53.58}\scriptsize{$\pm$0.3}   \\

&0.3 &\multicolumn{1}{c|}{41.62\scriptsize{$\pm$0.2}} &\multicolumn{1}{c|}{44.06\scriptsize{$\pm$0.2}} &\multicolumn{1}{c|}{31.63\scriptsize{$\pm$0.2}} &\multicolumn{1}{c|}{42.43\scriptsize{$\pm$0.1}} &\multicolumn{1}{c|}{27.75\scriptsize{$\pm$0.2}} &\multicolumn{1}{c|}{24.89\scriptsize{$\pm$0.1}} &\textbf{44.62}\scriptsize{$\pm$0.2}   \\

&0.5 &\multicolumn{1}{c|}{39.19\scriptsize{$\pm$0.2}} &\multicolumn{1}{c|}{41.26\scriptsize{$\pm$0.1}} &\multicolumn{1}{c|}{30.29\scriptsize{$\pm$0.2}} &\multicolumn{1}{c|}{40.22\scriptsize{$\pm$0.2}} &\multicolumn{1}{c|}{26.15\scriptsize{$\pm$0.2}} &\multicolumn{1}{c|}{22.93\scriptsize{$\pm$0.1}} &\textbf{42.27}\scriptsize{$\pm$0.2}    \\ \hline\hline

\multirow{3}{*}{\begin{tabular}[c]{@{}c@{}}CIFAR-10\\ (VGG-11)\end{tabular}}         
&0.1 &\multicolumn{1}{c|}{72.30\scriptsize{$\pm$0.2}} &\multicolumn{1}{c|}{74.67\scriptsize{$\pm$0.2}} &\multicolumn{1}{c|}{72.66\scriptsize{$\pm$0.3}} &\multicolumn{1}{c|}{73.49\scriptsize{$\pm$0.2}} &\multicolumn{1}{c|}{73.81\scriptsize{$\pm$0.2}} &\multicolumn{1}{c|}{70.85\scriptsize{$\pm$0.3}} &\textbf{75.93}\scriptsize{$\pm$0.2}   \\

&0.3 &\multicolumn{1}{c|}{55.47\scriptsize{$\pm$0.3}} &\multicolumn{1}{c|}{57.94\scriptsize{$\pm$0.4}} &\multicolumn{1}{c|}{55.45\scriptsize{$\pm$0.3}} &\multicolumn{1}{c|}{56.55\scriptsize{$\pm$0.2}} &\multicolumn{1}{c|}{57.34\scriptsize{$\pm$0.3}} &\multicolumn{1}{c|}{51.89\scriptsize{$\pm$0.4}} &\textbf{59.01}\scriptsize{$\pm$0.4}   \\

&0.5 &\multicolumn{1}{c|}{47.84\scriptsize{$\pm$0.2}} &\multicolumn{1}{c|}{51.56\scriptsize{$\pm$0.3}} &\multicolumn{1}{c|}{48.64\scriptsize{$\pm$0.2}} &\multicolumn{1}{c|}{49.07\scriptsize{$\pm$0.3}} &\multicolumn{1}{c|}{50.92\scriptsize{$\pm$0.2}} &\multicolumn{1}{c|}{47.36\scriptsize{$\pm$0.2}} &\textbf{52.72}\scriptsize{$\pm$0.2}   \\ \hline\hline

\multirow{3}{*}{\begin{tabular}[c]{@{}c@{}}Fashion\\ MNIST\\ (AlexNet)\end{tabular}} 
&0.1 &\multicolumn{1}{c|}{83.78\scriptsize{$\pm$0.3}} &\multicolumn{1}{c|}{86.12\scriptsize{$\pm$0.3}} &\multicolumn{1}{c|}{84.24\scriptsize{$\pm$0.3}} &\multicolumn{1}{c|}{84.56\scriptsize{$\pm$0.3}} &\multicolumn{1}{c|}{84.76\scriptsize{$\pm$0.3}} &\multicolumn{1}{c|}{83.72\scriptsize{$\pm$0.3}} &\textbf{87.34}\scriptsize{$\pm$0.2}   \\

&0.3 &\multicolumn{1}{c|}{75.75\scriptsize{$\pm$0.2}} &\multicolumn{1}{c|}{77.28\scriptsize{$\pm$0.5}} &\multicolumn{1}{c|}{75.76\scriptsize{$\pm$0.2}} &\multicolumn{1}{c|}{76.66\scriptsize{$\pm$0.2}} &\multicolumn{1}{c|}{74.84\scriptsize{$\pm$0.3}} &\multicolumn{1}{c|}{73.07\scriptsize{$\pm$0.2}} &\textbf{78.29}\scriptsize{$\pm$0.3}   \\

&0.5 &\multicolumn{1}{c|}{69.26\scriptsize{$\pm$0.2}} &\multicolumn{1}{c|}{69.96\scriptsize{$\pm$0.3}} &\multicolumn{1}{c|}{68.01\scriptsize{$\pm$0.2}} &\multicolumn{1}{c|}{69.87\scriptsize{$\pm$0.2}} &\multicolumn{1}{c|}{67.54\scriptsize{$\pm$0.3}} &\multicolumn{1}{c|}{65.64\scriptsize{$\pm$0.4}} &\textbf{71.14}\scriptsize{$\pm$0.3}    \\ \hline\hline

\multirow{3}{*}{\begin{tabular}[c]{@{}c@{}}HARBox\\ (DNN)\end{tabular}}              
&0.1 &\multicolumn{1}{c|}{48.59\scriptsize{$\pm$0.7}} &\multicolumn{1}{c|}{50.44\scriptsize{$\pm$0.6}} &\multicolumn{1}{c|}{37.83\scriptsize{$\pm$0.4}} &\multicolumn{1}{c|}{49.10\scriptsize{$\pm$0.7}} &\multicolumn{1}{c|}{51.07\scriptsize{$\pm$0.7}} &\multicolumn{1}{c|}{50.58\scriptsize{$\pm$0.6}} &\textbf{55.97}\scriptsize{$\pm$0.3}   \\

&0.3 &\multicolumn{1}{c|}{46.04\scriptsize{$\pm$0.5}} &\multicolumn{1}{c|}{48.44\scriptsize{$\pm$0.3}} &\multicolumn{1}{c|}{26.06\scriptsize{$\pm$0.3}} &\multicolumn{1}{c|}{46.52\scriptsize{$\pm$0.3}} &\multicolumn{1}{c|}{49.50\scriptsize{$\pm$0.3}} &\multicolumn{1}{c|}{48.97\scriptsize{$\pm$0.3}} &\textbf{55.86}\scriptsize{$\pm$0.3}   \\

&0.5 &\multicolumn{1}{c|}{44.47\scriptsize{$\pm$0.3}} &\multicolumn{1}{c|}{46.04\scriptsize{$\pm$0.3}} &\multicolumn{1}{c|}{26.04\scriptsize{$\pm$0.2}} &\multicolumn{1}{c|}{44.97\scriptsize{$\pm$0.4}} &\multicolumn{1}{c|}{48.24\scriptsize{$\pm$0.3}} &\multicolumn{1}{c|}{48.32\scriptsize{$\pm$0.3}} &\textbf{52.42}\scriptsize{$\pm$0.3}     \\ \hline
\end{tabular}}
\label{tbl:generalization-ring}
\end{table}

\textbf{Test accuracy and generalization comparison with PS based methods.} Since communications occur between the central server and workers in PS based setting, for a fair comparison, we only compare decentralized baselines with PS based baselines under ``Ring''.  As shown in Table~\ref{tbl:final-accuracy-ring}, we observe that \DePRL outperforms all considered baselines over all four datasets and three non-IID partitions. First, the state-of-the-art FedAvg performs worse in non-IID settings due to the fact that they target on learning a single model without encouraging personalization.  Second, though Ditto and FedRoD are incorporated with personalization, \DePRL always outperforms them. In particular, \DePRL achieves a remarkable performance improvement on non-IID partitioned CIFAR-100.  Compared to CIFAR-10, the data heterogeneity across workers are further increased due to the larger number of classes, and hence calls for personalization of local models.  This observation makes our representation learning augmented personalized model in \DePRL even pronounced compared to learning a single full-dimensional model in these baseline methods.  In addition, \DePRL also outperforms FedRep to a large extent, which also leverages the representation learning theory to learn a common representation under the PS setting.  Likewise, the same observations can be made when we compare the generalization performance as shown in Table~\ref{tbl:generalization-ring}. 

{We further elaborate the key technical idea behind \DePRL that guarantees enhanced performance over DisPFL.  Though DisPFL focuses on addressing personalization, it tackles data heterogeneity in a totally different way. The reason for the benefits of \DePRL over DisPFL is two-fold. First, in DisPFL, all workers still learn a single global parameter $\mathbf{w}$
 with each worker 
 learns an individual mask $\mathbf{m}_i$ 
, where the mask $\mathbf{m}_i$ 
 is a sparse vector with 
$0s$ and 
$1s$. The mask vector updates involve binary (i.e., 0-1) optimization, which requires some additional approximation techniques, possibly leading to some performance loss. Second, in DisPFL, to reduce the communication overhead, each worker 
 transmits the weights vector 
 $\mathbf{w}_i$  after sparsification with mask $\mathbf{m}_i$  
 to its neighbor. As this known, such sparsification techniques often improves the communication efficiency at the cost of some information loss, which may degrade the model accuracy (though in some cases may not be significant). While for our \DePRL, we leverage the representation learning theory by the insights that heterogeneous clients often share a common global representation . As a result, the global representation maps from 
$d$-dimensional data points to a lower space of size 
$z$, with $z\ll d$ 
. Since only global representation is shared among clients, our representation learning model naturally reduces the communication among clients without the extra sparsification used in DisPFL, and hence without sacrificing the model accuracy. To this end, we believe that the advantange of \DePRL comes from the powerful representation learning framework that allows heterogeneous workers to learn a shared common global representation but also a unique local head, in which the size of global representation is way larger than that of the local heads.}

\begin{table}[H]
\centering
\caption{The final test accuracy, converenge time and time-to-accuracy using non-IID partitioned CIFAR-100 and HARBox.}
\vspace{-0.1in}
\scalebox{0.7}{
\begin{tabular}{|c|c|cccc|cccc|cccc|}
\hline
\multirow{2}{*}{\textbf{\begin{tabular}[c]{@{}c@{}}Dataset\\ (Model)\end{tabular}}} & \multirow{2}{*}{\textbf{\begin{tabular}[c]{@{}c@{}}\# of\\ Clients\end{tabular}}} & \multicolumn{4}{c|}{\textbf{Final Accuracy (\%)}}                                                                                   & \multicolumn{4}{c|}{\textbf{Convergence Time (s)}}     & \multicolumn{4}{c|}{\textbf{Time to Accuracy (s)}}     \\ \cline{3-14} 
&                                                                                   & \multicolumn{1}{c|}{\textbf{FedRep}} & \multicolumn{1}{c|}{\textbf{D-PSGD}} & \multicolumn{1}{c|}{\textbf{DisPFL}} & \textbf{DePRL} & \multicolumn{1}{c|}{\textbf{FedRep}} & \multicolumn{1}{c|}{\textbf{D-PSGD}} & \multicolumn{1}{c|}{\textbf{DisPFL}} & \textbf{DePRL} & \multicolumn{1}{c|}{\textbf{FedRep}} & \multicolumn{1}{c|}{\textbf{D-PSGD}} & \multicolumn{1}{c|}{\textbf{DisPFL}} & \textbf{DePRL} \\ \hline\hline

\multirow{4}{*}{\begin{tabular}[c]{@{}c@{}}CIFAR-100\\ (ResNet-18)\end{tabular}}    
& 4                                                                                 & \multicolumn{1}{c|}{74.10}           & \multicolumn{1}{c|}{55.76}           & \multicolumn{1}{c|}{68.96}           & \textbf{77.98} & \multicolumn{1}{c|}{690.3}           & \multicolumn{1}{c|}{726.5}           & \multicolumn{1}{c|}{731.4}           & \textbf{602.7} & \multicolumn{1}{c|}{381.5}           & \multicolumn{1}{c|}{427.6}           & \multicolumn{1}{c|}{356.5}           & \textbf{332.8} \\
& 8                                                                                 & \multicolumn{1}{c|}{68.56}           & \multicolumn{1}{c|}{49.87}           & \multicolumn{1}{c|}{62.87}           & \textbf{71.58} & \multicolumn{1}{c|}{358.6}           & \multicolumn{1}{c|}{380.0}           & \multicolumn{1}{c|}{387.9}           & \textbf{331.2} & \multicolumn{1}{c|}{194.1}           & \multicolumn{1}{c|}{221.3}           & \multicolumn{1}{c|}{185.9}           & \textbf{167.5} \\
& 16                                                                                & \multicolumn{1}{c|}{59.61}           & \multicolumn{1}{c|}{42.40}           & \multicolumn{1}{c|}{56.63}           & \textbf{61.35} & \multicolumn{1}{c|}{230.5}           & \multicolumn{1}{c|}{243.6}           & \multicolumn{1}{c|}{234.2}           & \textbf{205.3} & \multicolumn{1}{c|}{105.4}           & \multicolumn{1}{c|}{121.0}           & \multicolumn{1}{c|}{107.5}           & \textbf{87.6}  \\ 
& 32                                                                                & \multicolumn{1}{c|}{53.37}           & \multicolumn{1}{c|}{36.71}           & \multicolumn{1}{c|}{47.95}           & \textbf{54.94} & \multicolumn{1}{c|}{160.9}           & \multicolumn{1}{c|}{167.9}           & \multicolumn{1}{c|}{163.3}           & \textbf{149.8} & \multicolumn{1}{c|}{59.7}            & \multicolumn{1}{c|}{66.1}            & \multicolumn{1}{c|}{55.8}            & \textbf{47.5}  \\ \hline\hline

\multirow{4}{*}{\begin{tabular}[c]{@{}c@{}}HARBox\\ (DNN)\end{tabular}}             
& 4                                                                                 & \multicolumn{1}{c|}{88.59}           & \multicolumn{1}{c|}{73.67}           & \multicolumn{1}{c|}{95.13}           & \textbf{95.98} & \multicolumn{1}{c|}{369.1}           & \multicolumn{1}{c|}{391.5}           & \multicolumn{1}{c|}{376.8}           & \textbf{319.3} & \multicolumn{1}{c|}{72.6}            & \multicolumn{1}{c|}{82.3}            & \multicolumn{1}{c|}{66.8}            & \textbf{55.2}  \\ 
& 8                                                                                 & \multicolumn{1}{c|}{84.98}           & \multicolumn{1}{c|}{69.66}           & \multicolumn{1}{c|}{92.45}           & \textbf{93.75} & \multicolumn{1}{c|}{222.5}           & \multicolumn{1}{c|}{245.6}           & \multicolumn{1}{c|}{230.9}           & \textbf{195.6} & \multicolumn{1}{c|}{38.2}            & \multicolumn{1}{c|}{48.9}            & \multicolumn{1}{c|}{34.3}            & \textbf{30.9}  \\
& 16                                                                                & \multicolumn{1}{c|}{80.70}           & \multicolumn{1}{c|}{65.59}           & \multicolumn{1}{c|}{88.75}           & \textbf{90.57} & \multicolumn{1}{c|}{128.2}           & \multicolumn{1}{c|}{141.2}           & \multicolumn{1}{c|}{137.1}           & \textbf{118.9} & \multicolumn{1}{c|}{20.5}            & \multicolumn{1}{c|}{24.2}            & \multicolumn{1}{c|}{18.1}            & \textbf{16.8}  \\
& 32                                                                                & \multicolumn{1}{c|}{78.79}           & \multicolumn{1}{c|}{61.30}           & \multicolumn{1}{c|}{83.62}           & \textbf{85.68} & \multicolumn{1}{c|}{89.9}            & \multicolumn{1}{c|}{95.3}            & \multicolumn{1}{c|}{91.6}            & \textbf{79.3}  & \multicolumn{1}{c|}{11.6}            & \multicolumn{1}{c|}{12.7}            & \multicolumn{1}{c|}{9.1}             & \textbf{8.3}   \\ \hline

\end{tabular}
}
\label{tbl:linearspeedup-app}
\vspace{-0.1in}
\end{table}

\textbf{Speedup.} Complementary to Figure~\ref{fig:linearspeedup}, we present the corresponding final test accuracy, convergence time and time-to-accuracy in Table~\ref{tbl:linearspeedup-app}. In particular, the time-to-accuracy is measured when the targeted accuracy is 35\% in CIFAR-100 and 60\% in HARBox. Again, we observe that the speedup (convergence
time) is almost linearly increasing (decreasing) as the number of workers increases, which validates the linear speedup property of \DePRL. The training loss of \DePRL is presented in Figure~\ref{fig:trainingloss-app}. Specifically, we compute the training loss as follows. After completing local training on individual workers, each worker evaluates the training loss using its local model (which consists of the global representation and the local head). Subsequently, each worker sends its calculated loss values to a designated worker (e.g., one dummy worker) that is responsible for determining convergence, and aggregates the received loss values to compute the global training loss. From Figure~\ref{fig:trainingloss-app}, we observe that \DePRL scales well when the number of workers is growing in the system. 

\begin{figure}
    \centering
 \includegraphics[width=0.7\columnwidth]{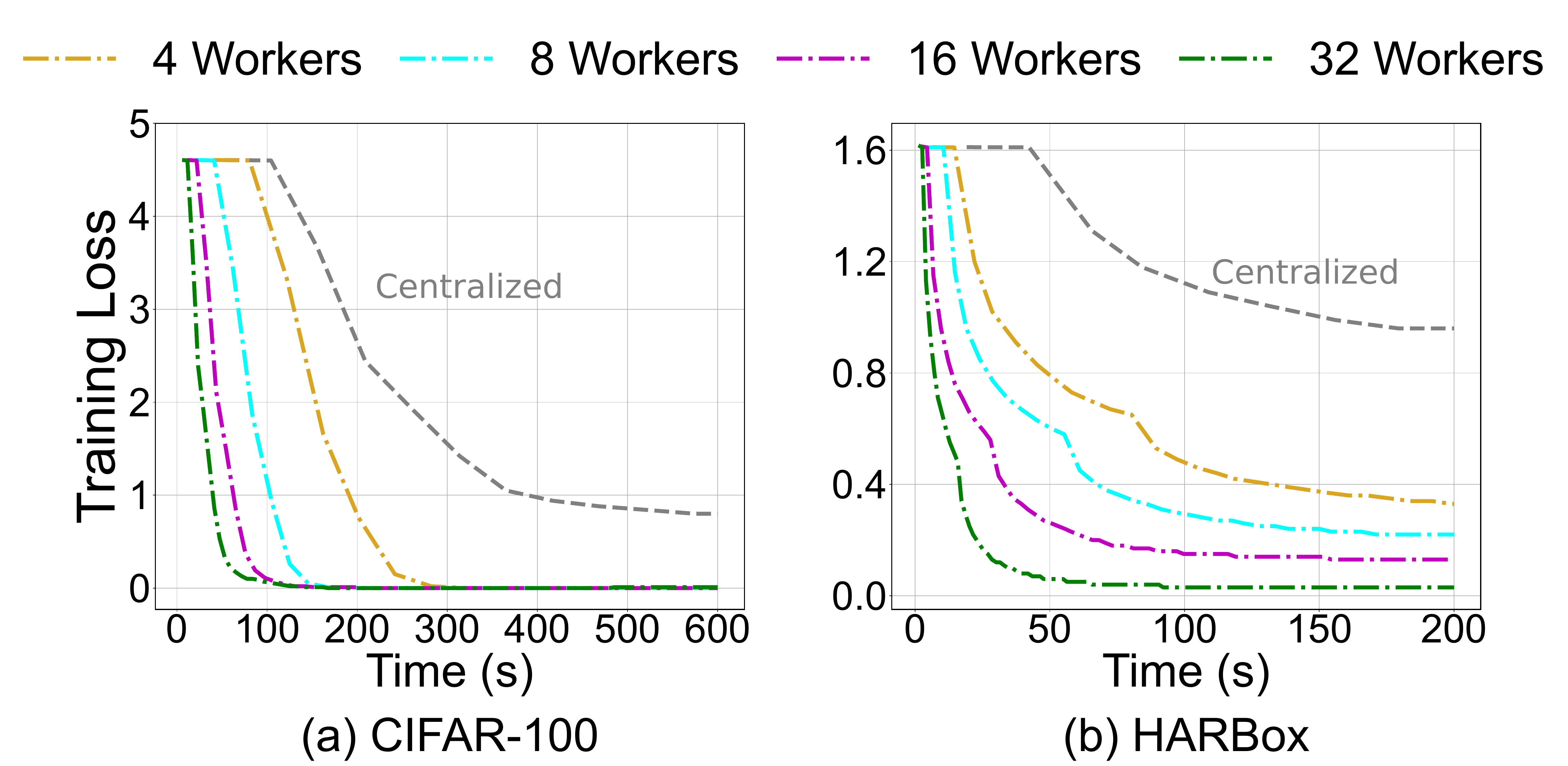}
 \vspace{-0.1in}
\caption{Training loss of \DePRL with different number of workers.}
 \label{fig:trainingloss-app}
  \vspace{-0.1in}
\end{figure}

\begin{figure}[t]
    \centering
       \includegraphics[width=0.7\textwidth]{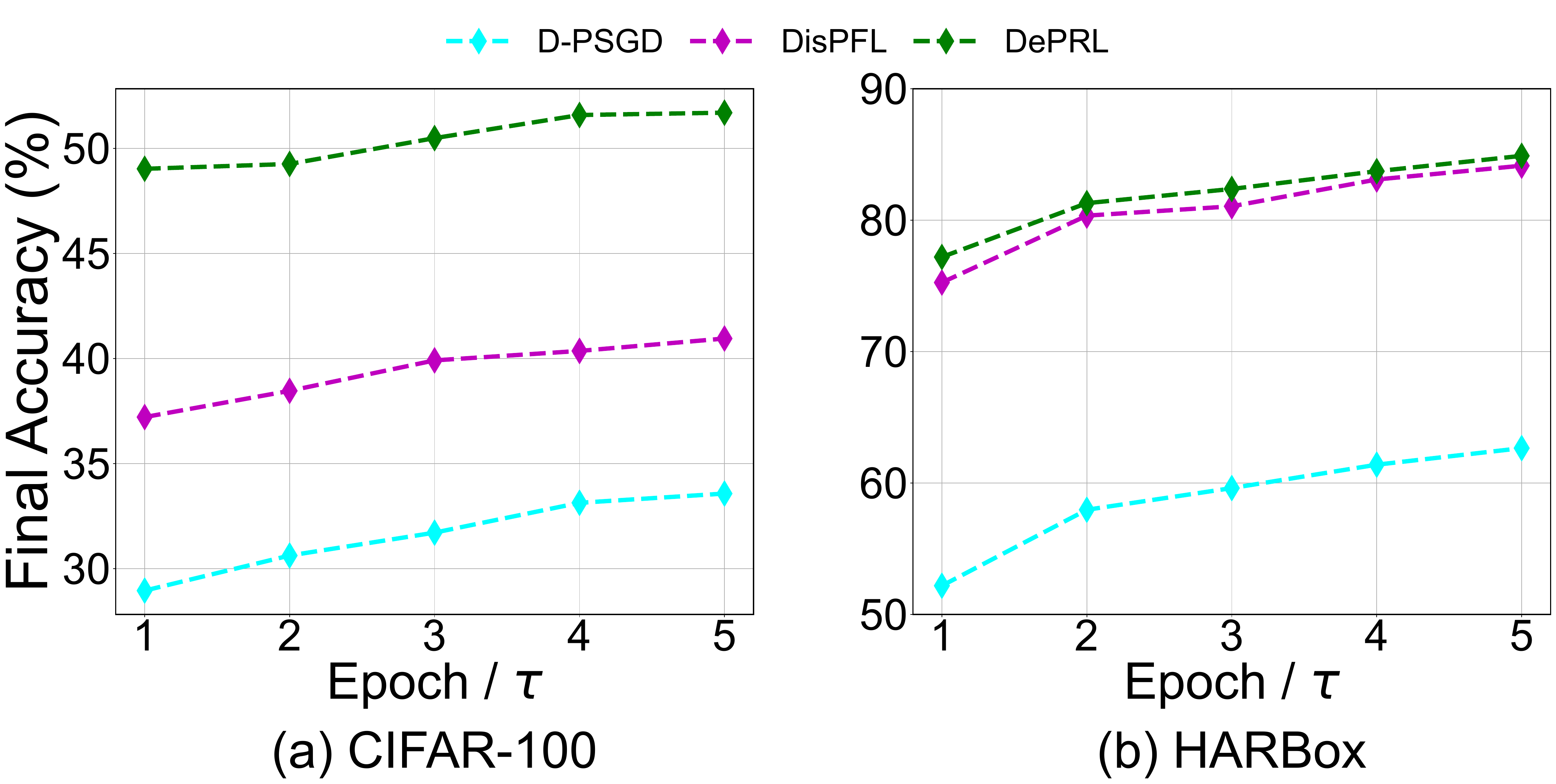}
 \caption{Test accuracy of different baselines using (Left) ResNet-18 on non-IID partitioned CIFAR-100 and (Right) DNN on non-IID partitioned HARBox across different update steps $\tau$ for local head.}
	\vspace{-0.1in}
	\label{fig:localsteps}
\end{figure}

\textbf{Local head update steps.} Since \DePRL allows to update the local head with $\tau$ steps, we set it to be 2 in our experiments.  The impact of $\tau$ on the final test accuracy is presented in Figure~\ref{fig:localsteps}.  On one hand, we observe that more local updates can benefit the performance of all algorithms, as observed in the representation learning theory augmented PS-based method \cite{collins2021exploiting}.  On the other hand, \DePRL consistently outperforms all baselines across all considered local head update steps.

\begin{figure}[t]
    \centering
       \includegraphics[width=0.75\textwidth]{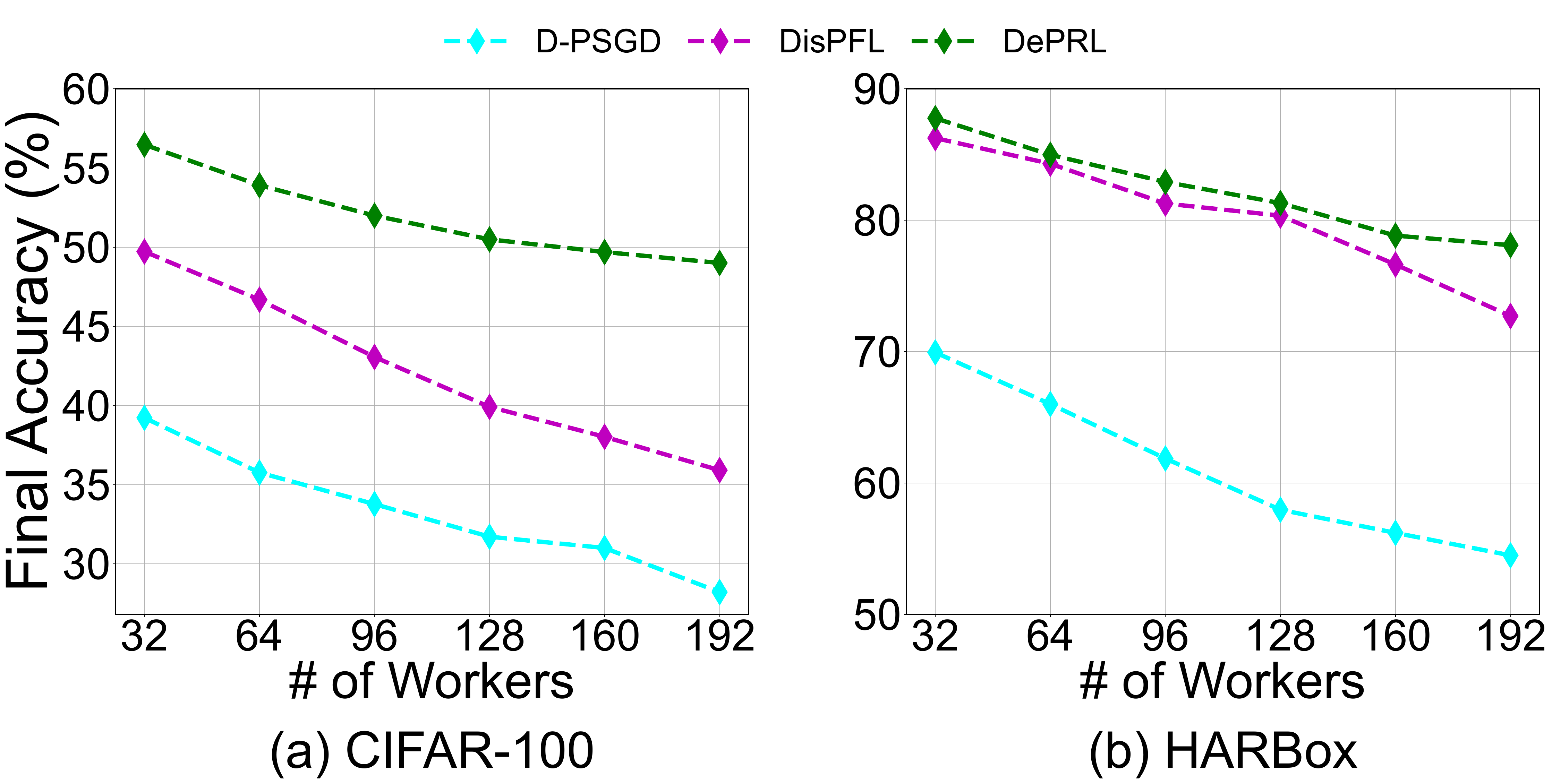}
 \caption{Test accuracy of different baselines using (Left) ResNet-18 on non-IID partitioned CIFAR-100 and (Right) DNN on non-IID partitioned HARBox across different numbers of total clients.}
	\vspace{-0.1in}
	\label{fig:clients}
\end{figure}

\textbf{Number of total workers.} Our experiments are reported with a total of 128 workers.  From Figure~\ref{fig:clients}, we can see that \DePRL always outperforms all baselines regardless of the number of total clients.

\begin{figure}[t]
    \centering
       \includegraphics[width=0.75\textwidth]{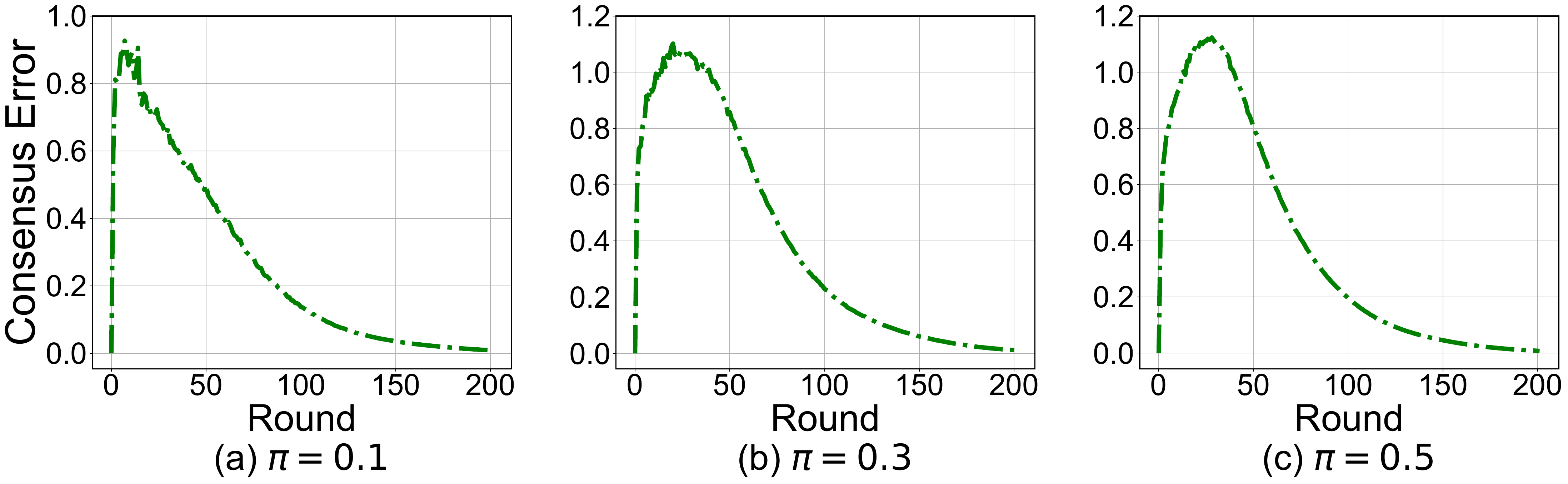}
 \caption{Consensus errors of \DePRL using VGG-11 on non-IID partitioned CIFAR-10 with different heterogeneities.}
	\vspace{-0.1in}
	\label{fig:CIFAR10-consensus}
\end{figure}

\begin{figure}[t]
    \centering
       \includegraphics[width=0.75\textwidth]{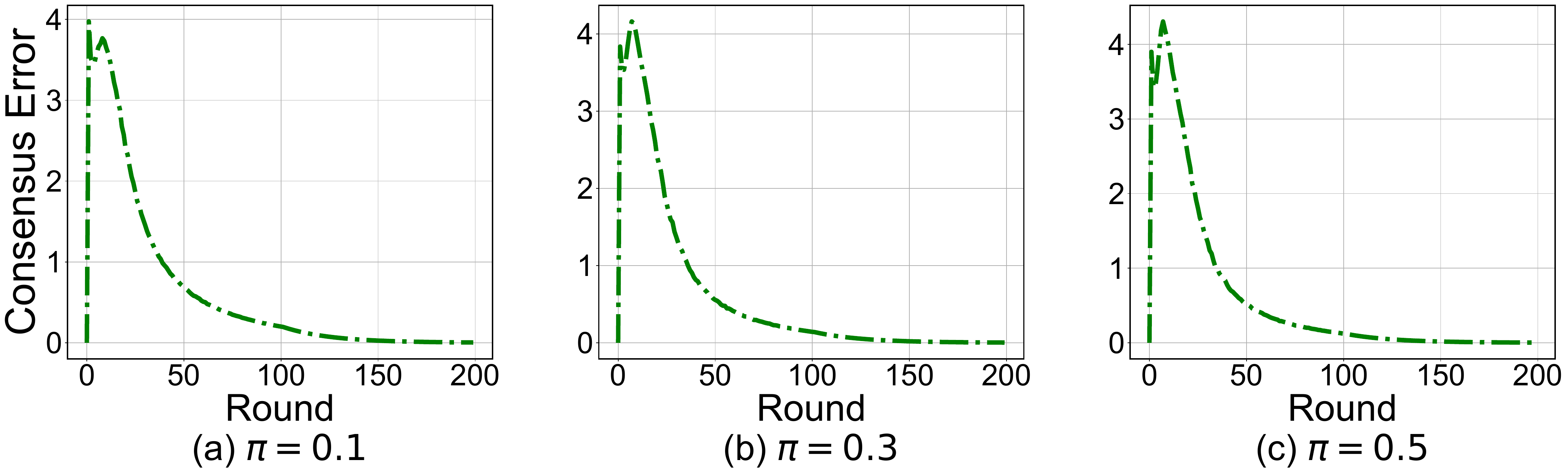}
 \caption{Consensus errors of \DePRL using ResNet-18 on non-IID partitioned CIFAR-100 with different heterogeneities.}
	\vspace{-0.1in}
	\label{fig:CIFAR100-consensus}
\end{figure}

\textbf{Consensus errors.} Since our \DePRL is a decentralized learning algorithm with provably convergence guarantee, we now report the consensus error of \DePRL as a function of training rounds.  As shown in Figures~\ref{fig:CIFAR10-consensus} and~\ref{fig:CIFAR100-consensus}, the consensus errors first increase and then decrease.  This is due to the fact that all workers initialized with a random local head, and it takes some rounds for workers to learn the global representation together through the communication, which is part of the nature of our \DePRL (see Step 9 in Algorithm~\ref{alg:general}).  After these initial rounds, the consensus decreases, which in alignment with our theoretical analysis (see Theorem~\ref{thm:loss_convergence} and proof sketch in the main paper).

\end{document}